\documentclass[12pt]{article}
\usepackage{amssymb,amsmath,bbm}
\usepackage{graphicx}
\usepackage{float}
\usepackage{subcaption}
\usepackage{booktabs}
\usepackage{natbib}
\usepackage{appendix}
\bibpunct{(}{)}{;}{a}{,}{,}
\usepackage[colorlinks=true, citecolor=blue]{hyperref}
\hypersetup{
    colorlinks=true,
    linkcolor=blue,
    filecolor=black,
    urlcolor=blue,
    }
\let\linkedref\ref
\renewcommand{\ref}[1]{\linkedref*{#1}}
\def\sqr#1#2{{\vcenter{\hrule height.#2pt \hbox{\vrule width.#2pt
 height#2pt \kern#1pt \vrule width.#2pt} \hrule height.#2pt}}}

\newcommand{\htheta}{{\hat\theta}}
\newcommand{\hepsilon}{{\hat{\epsilon}}}

\newcommand{\tod}{\stackrel{d}{\longrightarrow}}

\newcommand{\pto}{\stackrel{p}{\longrightarrow}}
\newcommand{\toiid}{\stackrel{iid}{\sim}}

\newcommand{\Var}{\operatorname{Var}}
\newcommand{\Cov}{\operatorname{Cov}}

\newif\ifshowappendix
\showappendixtrue 

\newcommand{\mref}[1]{\ref{#1}}
\newcommand{\meqref}[1]{\eqref{#1}}
\providecommand{\ORCID}[1]{}

\newcounter{condcounter}
\renewcommand{\thecondcounter}{\arabic{condcounter}}

\newtheorem{defi}{\sc Definition}
\newtheorem{theo}{\sc Theorem}
\newtheorem{lemm}{\sc Lemma}

\newtheorem{coro}{\sc Corollary}
\newtheorem{rema}{\sc Remark}


\def\be{\begin{equation}}
\def\ee{\end{equation}}
\def\nn{\nonumber}
\def\bea{\begin{eqnarray}}
\def\eea{\end{eqnarray}}

\def\noopsort#1{}
\begin{document}

\begin{center}
 {\bf \large Tolerance-Based Fairness Auditing: Violation Certification and Sensitivity Screening}
\end{center}

\bigskip

\begin{center}
\renewcommand{\thefootnote}{*}
Jie Tang, \quad
Chuanlong Xie,\quad
Lixing Zhu\footnote{Correspondence to Lixing Zhu (lzhu@bnu.edu.cn). All coauthors are listed in seniority order. Drs. Tang and Xie are co-first authors with equal contributions. The research of Lixing Zhu was supported by the grants (NSFC 12471276)  from the National Natural Scientific Foundation of China.} %NSFC 12131006
\bigskip

{\small (School of Statistics, Beijing Normal University, Zhuhai, Guangdong 519087, China)}

\renewcommand{\thefootnote}{\arabic{footnote}}
\end{center}
\bigskip

 \noindent
{\bf Abstract }  
As artificial intelligence is increasingly deployed, algorithmic unfairness has raised growing concerns and intensified demands for transparent fairness auditing.
In practice, the tolerable degree of algorithmic unfairness depends on the specific legal, ethical, or application context.
Given a prespecified tolerance threshold, an important statistical question is how to determine whether a group disparity exceeds the allowable tolerance across different auditing objectives.
To address this problem, we develop a unified tolerance-based  fairness auditing framework for two complementary auditing objectives: violation certification, which prioritizes control of false violation declarations, and sensitivity screening, which prioritizes reducing missed violations.
For the first objective, we develop a constrained empirical likelihood test for formal settings that uses least-favorable-point calibration and can be combined with false flagging rate control for simultaneous subgroup auditing.
For the second objective, we develop split empirical likelihood and adjusted split empirical likelihood tests using an adaptive boundary-proxy principle for early-warning settings.
Numerical experiments show the distinct error-control--sensitivity trade-offs of these procedures. 
A COMPAS analysis illustrates the framework in predictive fairness auditing.


%\noindent AMS 2010 subject classification: Primary 62G05\ \ \  secondary
%62E20
\noindent {\it ACM class}: I.2.6 

\noindent {\it AMS class}: 62G10 (Primary), 62G20, 68T05 (Secondary)

\noindent {\it Short title:} Tolerance-Based Fairness Auditing

\newpage

\section{Introduction}

As artificial intelligence increasingly informs high-stakes decisions in employment, credit, health care, and criminal justice, disparities in predictive performance or decision outcomes across groups defined by sensitive attributes have become an important public concern \citep{XueEtAl2020,PoriaEtAl2020,GargEtAl2023,FraserEtAl2022}. Such disparities may reinforce existing social stereotypes and lead to unequal access to opportunities and resources \citep{DresselFarid2018,SalimiEtAl2020,JiangEtAl2023}. Consequently, regulators, developers, and affected communities increasingly demand transparent and independent third-party fairness audits \citep{BrundageEtAl2020,YanZhang2022,LaineEtAl2024,LacmanovicSkare2025}. Such audits must do more than report observed group disparities: they must quantify sampling uncertainty and determine whether the observed evidence is sufficient to support consequential decisions.

Many existing inferential approaches to fairness take exact parity or zero disparity as a benchmark, for example by testing $H_0:\epsilon_G=0$, where $\epsilon_G$ denotes a signed group disparity defined under the fairness criterion of interest \citep{TramerEtAl2017,DiCiccioEtAl2020,MehrabiEtAl2022,vonZahnEtAl2023,TangEtAl2026}. In practice, however, an audit often asks not whether a disparity is exactly zero, but whether it is large enough to be unacceptable in the relevant legal, ethical, or application context \citep{LuoEtAl2024}. We therefore consider a prespecified tolerance threshold $\epsilon_0$: $0<\epsilon_G\leq\epsilon_0$ represents a nonzero but acceptable disparity, whereas $\epsilon_G>\epsilon_0$ constitutes a tolerance violation. This distinction separates statistically detectable disparities from disparities that are practically unacceptable and recasts fairness auditing as a statistical decision problem concerning whether $\epsilon_G$ lies outside the tolerance region. Many group-fairness criteria, including demographic parity and equal opportunity, can be audited through an analogous tolerance formulation once an appropriate disparity measure has been specified \citep{CaldersVerwer2010,DworkEtAl2012,HardtEtAl2016,PessachShmueli2022,CherianCandes2024,CatonHaas2024}.

Tolerance-based auditing gives rise to two complementary objectives with different statistical requirements. The first is formal {\it violation certification}: falsely declaring a disparity within the tolerance region to be a violation may incur substantial regulatory, organizational, or decision costs, so a violation should be certified only when the evidence is sufficiently strong. The second is {\it sensitivity screening}: in early risk discovery, missing a substantively meaningful disparity may be more costly than issuing an additional warning, so greater priority is placed on sensitivity near the tolerance boundary. Violation certification prioritizes protection against false declarations, whereas sensitivity screening prioritizes the reduction of missed violations. These objectives therefore induce different error costs and calibration requirements, making a single calibration principle difficult to optimize for both. This tension motivates our central question: how should statistical evidence be calibrated to the objective of a tolerance-based fairness audit?

To address this question, we consider auditing settings in which the adverse direction is prespecified, restrict the disparity parameter to $\epsilon_G\in\mathbb{R}^+$, and develop objective-adaptive inference for the constrained tolerance hypothesis
\bea \label{test-1}
H_0: 0 \le \epsilon_G \le \epsilon_0 \quad \text{versus} \quad H_1:\epsilon_G >\epsilon_0,
\eea
where $\epsilon_0>0$ denotes a prespecified allowable disparity.
Empirical likelihood (EL) as a nonparametric approach requires no explicit variance estimation, enjoys a Wilks-type property \citep{Owen1988}, and avoids repeated bootstrap calibration \citep{CherianCandes2024}. 
For violation certification, we propose constrained empirical likelihood (CEL) using least-favorable-point (LFP) calibration to control the risk of false violation declarations throughout the tolerance region. 
For sensitivity screening, we propose split empirical likelihood (SEL) and adjusted split empirical likelihood (ASEL) using an adaptive boundary-proxy principle that permits increased rejection probability at the tolerance boundary in exchange for greater sensitivity to emerging violations.
Together, these procedures form an objective-adaptive inferential framework that aligns statistical calibration with the distinct error priorities of violation certification and sensitivity screening.
Simulations further illustrate the distinct finite-sample trade-offs between error control and detection under skewed and heavy-tailed distributions.

This article makes three contributions. 
First, we formulate fairness auditing as a directionally constrained tolerance-decision problem, assessing whether a group disparity in a prespecified adverse direction exceeds an allowable threshold. This formulation accommodates nonzero but tolerable disparities and provides a common testing framework for violation certification and sensitivity screening.
Second, our main methodological contribution is the development of SEL and ASEL to address the limited sensitivity of conservative violation-certification procedures in early-warning settings. Through an adaptive boundary-proxy construction, SEL targets emerging tolerance violations, while ASEL further enhances detection power. 
We characterize their rejection behavior at interior null points, at the tolerance boundary, and under local and fixed violations, explicitly quantifying the trade-off between detection sensitivity and false alerts.
Third, CEL provides a quantitative basis for formal violation certification through LFP calibration. Its one-sided lower confidence bound serves as a numerical certificate, quantifying how far the disparity can be certified to exceed the tolerance threshold at a prescribed confidence level. For disjoint subpopulations, we establish asymptotic false flagging rate control for CEL combined with the Benjamini--Hochberg (BH) procedure under suitable regularity conditions.

The article is organized as follows. 
Section~\ref{sec-preliminaries} presents the problem setup and preliminary results. Section~\ref{sec-methodology} develops the objective-specific calibration principles and characterizes the operating behavior required for violation certification and sensitivity screening.
Section~\ref{sec-sim} reports simulation results. Section~\ref{sec-emp} presents the real-data analysis. Section~\ref{sec-conclusion} concludes the article. 
\ifshowappendix
An extension to lower-tail tolerance auditing (Section~\ref{app-extended-test}) and technical proofs (Section~\ref{app-proofs}) are provided in the Supplementary Material.
\else
An extension to lower-tail tolerance auditing and technical proofs are provided in the Supplementary Material.
\fi

{\bf Notation.}
Let $\mathbb{P}(G) := \mathbb{P}((X, Y ) \in G)$ denote the  probability of $(X, Y )$ belonging to subpopulation $G$.
To simplify our notation, we replace $ (X, Y) \in G$ with $G$ whenever the meaning is clear. 
We also replace $M(h (X), Y )$ by the abbreviation $M$.
For a function $f : X \times Y \to \mathbb{R}^k$, $\mathbb{P} [f]$ is shorthand for $\mathbb{E}_{\mathbb{P}} [f (X, Y )]$, and $\mathbb{P}_n[f ]$ is shorthand  for $n^{-1} \sum_{i=1}^n f (X_i, Y_i)$. 
We also denote  $\mathbb{P}_n(G) = \frac{1}{n}  \sum_{i = 1}^n \mathbbm{1} \{(x_i, y_i) \in G\}$ and $ n_G = n  \cdot \mathbb{P}_n(G)$.

%\section{Problem setup}\label{sec-setup}
\section{Preliminaries} \label{sec-preliminaries}

We assume that only a holdout dataset $D = \{X_i, Y_i\}^n_{i=1} \toiid P$ is accessible. 
This is sometimes referred to as audit trails \citep{BrundageEtAl2020}. 
We say that a prediction rule $h$ exhibits a performance disparity on a subpopulation $G$ if the conditional expectation of a metric $M(h (X), Y)$ given $(X, Y)\in G$ differs substantially from a target $\theta_P \in \mathbb{R}$. 
The group-wise performance disparity is given by \cite{CherianCandes2024} as follows. 

\begin{defi}\label{defi1} 
A group-wise performance disparity is defined as 
\be 
\epsilon_G=\mathbb{E}_P[M(h (X), Y ) | (X, Y) \in G] - \theta_P. \nn
\ee 
%where $G \in \mathcal{G}$,  $\mathcal{G} = \{ G_1, G_2, \dots, G_m\}$ and $m  < \infty$.
%$m =|\mathcal{G}| < \infty$ where $|\mathcal{G}|$ is defined as the cardinality of $\mathcal{G}$.
\end{defi} 

While subgroup membership is often determined by a subset of covariates (such as sensitive attributes) that may not be explicitly used by the prediction rule, we use the same covariate vector $X$ for both the model input and the group definitions to maintain notational simplicity. The definition serves as a flexible, unified framework capable of subsuming nearly any standard group fairness definition through the specific choice of $\epsilon_G$. To illustrate this versatility, consider Statistical Parity (also referred to as Demographic Parity) \citep{DworkEtAl2012, MehrabiEtAl2021, GargEtAl2020, VermaRubin2018a}. 
This criterion demands that the probability of a positive prediction remains invariant across different groups: $\mathbb{P}(h(X)=1\mid X \in G)=\mathbb{P}(h(X)=1).$ 
The disparity $\epsilon_G$ is
\(
\epsilon_G = \mathbb{P}(h(X)=1\mid X \in G)- \mathbb{P}(h(X)=1).
\)
By mapping this to the definition, we simply assign: 
\(M(h (X), Y ) = \mathbbm{1} \{h (X) = 1\}\) and \(\theta_P = \mathbb{P}(h(X)=1)\).

We formulate upper-tail tolerance auditing through the following constrained hypothesis, where the goal is to assess whether $h$ that the disparity exceeds the prespecified tolerance.
Given $h$, let $M_i=M(h (X_i), Y_i)$, and then $M_1, \dots, M_n$ are independent and identically distributed random variables. 
For the tolerance audit in~\eqref{test-1}, we prespecify the adverse direction and restrict the disparity parameter to $[0,\infty)$.
Given a tolerance threshold $\epsilon_0>0$, the null region $[0,\epsilon_0]$ includes both exact parity and nonzero disparities regarded as practically acceptable.

%We give the construction of the EL ratio statistic.
%By the definition and $\mathbb{E}_P  [ M \mid G ] = \frac{1}{\mathbb{P}(G)} \mathbb{E}_P [ M \cdot \mathbbm{1}_G ]$,  the moment equation is
%\bea  \label{moment_constraint}
%\mathbb{E}_P \left\{ \frac{1}{\mathbb{P}(G)} [ M - \theta_P - \epsilon_G ] \cdot \mathbbm{1}_G  \right\} = 0. \nn
%\eea
%We then denote $p_i=\mathbb{P}(x_i, y_i)$ and $g_i(\epsilon_G;\theta_P) = \left[M_i - \theta_P - \epsilon_G\right]\cdot \mathbbm{1}_{G}$.
To quantify evidence about the disparity, we construct the EL ratio statistic using the estimating functions
\[
g_i(\epsilon;\theta_P) =(M_i-\theta_P-\epsilon)\mathbbm{1}\{(X_i,Y_i)\in G\}, \qquad i=1,\ldots,n,
\]
where $\epsilon$ denotes a candidate disparity value.
At the true disparity $\epsilon_G^*$, these functions satisfy
$\mathbb{E}_P[g_i(\epsilon_G^*;\theta_P)]=0$.
Thus, the fairness constraint is
\(
\sum_{i=1}^n p_i g_i(\epsilon_G;\theta_P) = 0,
\)
and then the EL function at $\epsilon_G$ is defined to be
\bea 
L(\epsilon_G;\theta_P) =\sup_{p_i,1 \le i \le n}  \left\{ \prod^n_{i=1} p_i \ \Big| \ p_i\geq 0, \,\, \sum_{i=1}^n p_i = 1, \,\, \sum_{i=1}^n p_i  g_i(\epsilon_G;\theta_P) = 0 \right\}. \nn
\eea
Following \citet{Owen1990}, the maximizer is unique for a given $\epsilon_G$, provided that zero lies in the interior of the convex hull of $\{g_i(\epsilon_G;\theta_P)\}_{i=1}^n$. In that case,
\bea 
\widehat{p}_i = \widehat{p}_i(\epsilon_G) = \frac{1}{n \left(1 + \lambda g_i(\epsilon_G;\theta_P)\right)}, \quad 1 \le i \le n,  \nn
\eea
where $\lambda \in \mathbb{R}$ solves
\bea \label{lambda}
{ 1 \over n} \sum_{i=1}^n{ g_i(\epsilon_G;\theta_P) \over 1+\lambda g_i(\epsilon_G;\theta_P) }=0.  \nn
\eea
Then, the empirical log-likelihood ratio is defined as
\bea \label{eq-logL}
\ell(\epsilon_G;\theta_P) :=-2 \log { L(\epsilon_G;\theta_P) \over n^{-n}} = 2\sum_{i=1}^n\log\left(1+\lambda g_i(\epsilon_G;\theta_P)\right). \nn
\eea
To obtain the limiting distribution of $\ell(\epsilon_G;\theta_P)$,  we need the following assumptions.

\begin{itemize}
\setcounter{condcounter}{0}

\refstepcounter{condcounter}\item[(\thecondcounter)]\label{C1} ${ (X_i, Y_i), 1 \le i \le n}$ are independent and identically distributed random  variables.

\refstepcounter{condcounter}\item[(\thecondcounter)]\label{C2}  $\mathbb{P}(G) $ and $ \Var( M \mid G )$ are bounded away from 0 for all $G \in \mathcal{G}$. Moreover, for some $\delta>0$, $\sup_{G\in\mathcal G}\mathbb{E}\left[|M-\theta_P-\epsilon_G^*|^{2+\delta}\mid G\right]< \infty$.

%\refstepcounter{condcounter}\item[(\thecondcounter)]\label{C3} The estimator of the target 
%is asymptotically linear, i.e.,
%\be \label{htheta}
%\sqrt{n}(\htheta - \theta_P) = {1 \over \sqrt{n}}\sum_{i=1}^n\psi(X_i, Y_i) + o_P(1),\nn
%\ee
%where $\psi$ is an influence function with mean zero and finite variance. 

\refstepcounter{condcounter}\item[(\thecondcounter)]\label{C3} The estimator of the target is asymptotically negligible uniformly over the audited groups:
$$
\sup_{G\in\mathcal G}
\sqrt n\,\mathbb P_n(G)
\left|\hat\theta-\theta_P\right|
=o_P(1).
$$

\end{itemize}

\begin{rema}
Condition~\eqref{C2} ensures the well-posedness of the covariance $\sigma_G^2$, where $ \sigma_G^2 = { \mathbb{E} \left[(M - \theta_P - \epsilon_{G})^2 \mathbbm{1}_{G} \right] } = { \mathbb{P}(G) } \Var(M \mid G)$. 
It prevents numerical instabilities caused by vanishing denominators or irregular behavior in the numerator terms.
\end{rema}

\begin{rema}\label{rema-candes2}
%Condition~\eqref{C3} is satisfied by any estimator expressible as a differentiable function of averages, i.e., $\hat{\theta}:= f(\sum_i \upsilon(x_i, y_i)/n)$ for some differentiable function $f$ and known features $\upsilon: \mathcal{X} \times \mathcal{Y} \to \mathbb{R}^k$. 
%In this paper, we typically compare group-wise performance with the population average $\theta_P=\mathbb E_P[M(h(X),Y)]$. The sample average $\hat\theta=n^{-1}\sum_{i=1}^n M(h(x_i),y_i)$ therefore satisfies Condition~\eqref{C3}.
Condition~\eqref{C3} requires the plug-in estimation error of the reference target to be asymptotically negligible at the empirical-likelihood scale. Since \(\mathbb P_n(G)\leq 1\), a simple sufficient condition for~\eqref{C3} is
\(
\sqrt n\left|\hat\theta-\theta_P\right|=o_P(1).
\)
Thus, Condition~\eqref{C3} requires the estimation error of the reference target to be of smaller order than the usual \(n^{-1/2}\) sampling fluctuation.
\end{rema}

For a known reference target $\theta_P$, under conditions~(\ref{C1})--(\ref{C2}), we have \(\ell(\epsilon_G^*;\theta_P)\tod \chi^2_1\) as \(n \to \infty\),  where $\chi^2_1$ is a chi-squared distribution with one degree of freedom.
For an estimated reference target $\hat{\theta}$, under conditions~(\ref{C1})--(\ref{C3}), we have \(\ell(\epsilon_G^*;\htheta)\tod \chi^2_1\) as \(n \to \infty\).
This shows that under the stated conditions, replacing $\theta_P$ with $\hat{\theta}$ preserves the limiting distribution, so we omit $\theta_P$ and $\htheta$ from the notation in the sequel for notational simplicity.  
\ifshowappendix
Formal statements of these Wilks-type limits are provided in Section~\ref{app-Related} of the Supplementary Material.
\else
Formal statements of these Wilks-type limits are provided in the Supplementary Material.
\fi

These limiting results provide a basis for testing the point null hypothesis $H_0:\epsilon_G=\epsilon_0$ versus $H_1:\epsilon_G\neq\epsilon_0$.
The tolerance hypothesis~\eqref{test-1}, however, concerns an interval of acceptable disparities.
Section~\ref{sec-methodology} develops calibration principles for this composite hypothesis tailored to violation certification and sensitivity screening.

\section{Tolerance-based fairness auditing} \label{sec-methodology}

Tolerance-based fairness auditing requires different operating criteria depending on whether the objective is violation certification or sensitivity screening. Violation certification prioritizes protection against false violation declarations throughout the tolerance region, whereas sensitivity screening places greater emphasis on detecting emerging violations near the tolerance boundary. These distinct priorities motivate two calibration principles: least-favorable-point calibration for violation certification and adaptive boundary-proxy calibration for sensitivity screening.

Section~\ref{sec-cel} formalizes the violation-certification criterion and develops the corresponding CEL procedure. Section~\ref{sec-sel} introduces the sensitivity-screening criterion and develops SEL and its power-enhanced variant, ASEL.
\ifshowappendix
An extension of the proposed framework to other tolerance-constrained testing problems is given in Section~\ref{app-extended-test} of the Supplementary Material.
\else
An extension of the proposed framework to other tolerance-constrained testing problems is deferred to the Supplementary Material.
\fi

\subsection{Violation certification under a tolerance constraint} \label{sec-cel}

Formal violation certification requires protection against falsely declaring an acceptable disparity to exceed the prespecified tolerance.
For a generic violation-certification rule, let $\delta_n(G;\epsilon_0)=1$ denote a declaration that group $G$ violates the tolerance requirement.
Because every $\epsilon_G\in[0,\epsilon_0]$ represents an acceptable disparity under~\eqref{test-1}, the asymptotic violation-certification requirement is
\[
\sup_{\epsilon_G\in[0,\epsilon_0]}
\limsup_{n\to\infty}
\mathbb{P}_{\epsilon_G}\{\delta_n(G;\epsilon_0)=1\}
\leq \alpha,
\]
where the subscript on the probability indicates the true disparity value and other nuisance features are suppressed.
Thus, calibration must protect against the null configuration with the largest asymptotic rejection probability rather than against an arbitrarily selected point in the tolerance region.

For the one-sided auditing problem~\eqref{test-1}, the tolerance boundary $\epsilon_G=\epsilon_0$ is the least favorable null configuration, as formalized below.
We therefore calibrate the violation-certification rule using the limiting distribution at this boundary.
This LFP principle prioritizes control of false violation declarations throughout the tolerance region, while allowing the procedure to be more conservative at fixed interior null points.

We instantiate this calibration principle through constrained empirical likelihood.
The resulting statistic compares the best empirical-likelihood fit under the tolerance constraint with the best fit over the nonnegative parameter space:
\bea \label{eq-logL-ineq}
T_c:=-2 \log { \sup_{0 \le \epsilon\le\epsilon_0}L(\epsilon) \over \sup_{\epsilon \ge 0} L(\epsilon)}. \nn
\eea
The numerator represents the best fit among disparities regarded as acceptable, whereas the denominator represents the best fit over all disparities in the prespecified adverse direction.
Accordingly, a large value of $T_c$ indicates that the data are incompatible with the entire tolerance region.
\ifshowappendix
The concavity argument and a detailed derivation of the computable form of $T_c$ are provided in Supplementary Material~\ref{app-cel-derivation}.
\else
The concavity argument and a detailed derivation of the computable form of $T_c$ are provided in the Supplementary Material.
\fi

To determine whether CEL is suitable for formal violation certification, we first characterize its rejection behavior at the tolerance boundary, where the risk of a false violation declaration is greatest.

\begin{theo}\label{theo-1}
Assume that conditions~\eqref{C1}--\eqref{C2} hold. 
Then, under the null hypothesis $H_0: 0 \le \epsilon_G \le \epsilon_0$, it follows that for the least favorable case $\epsilon_G = \epsilon_0$,
\[
T_c\tod \frac{1}{2}\chi^2_{0} + \frac{1}{2}\chi^2_1 \quad \mathrm{as} \quad n \to \infty,
\] 
where $\chi^2_{0}$ is a degenerate random variable with mass 1 at the point zero.  
\end{theo}

Theorem~\ref{theo-1} provides the boundary reference distribution required to calibrate violation-certification decisions.
Boundary calibration alone, however, does not establish the violation-certification guarantee; we must also determine how the rejection probability behaves throughout the tolerance region.

\begin{theo}\label{theo-2}
Assume that conditions~\eqref{C1}--\eqref{C2} hold. 
Then, under the null hypothesis $H_0: 0 \le \epsilon_G \le \epsilon_0$, we have
\bea
\lim_{n \to \infty} \mathbb{P}\left\{ T_c> \chi_{1,2\alpha}^2 \right\} =
\begin{cases}
\alpha, & \mathrm{if}\quad \epsilon_G^* = \epsilon_0; \\
0, & \mathrm{if}\quad 0 \le \epsilon_G^* < \epsilon_0, \nn
\end{cases}
\eea 
where $\chi_{1,2\alpha}^2$ is the upper $2\alpha$ quantile of the $\chi^2_1$ distribution.
\end{theo}

For the CEL violation-certification rule $\delta_n^{\mathrm{CEL}}(G;\epsilon_0)=\mathbbm{1}\{T_c>\chi_{1,2\alpha}^2\}$, Theorem~\ref{theo-2} establishes the operating characteristics required for violation certification.
The probability of a false violation declaration reaches the nominal level $\alpha$ at the least favorable boundary and converges to zero at every fixed interior null point.
Thus, CEL satisfies the asymptotic violation-certification requirement while becoming increasingly protective against false violation declarations in the interior of the tolerance region.

According to \citet{LehmannRomano2005}, Theorem~\ref{theo-2} demonstrates that $\epsilon_0$ is the least favorable point under the null hypothesis.
Hence the limiting distribution of $T_c$ at the boundary point $\epsilon_0$ is employed to determine the critical value for the rejection region. 
According to Theorems~\ref{theo-1}--\ref{theo-2}, the critical value $\chi_{1,2\alpha}^2$ is determined by
\(
\mathbb{P}\!\left(\tfrac{1}{2}\chi_0^2+\tfrac{1}{2}\chi^2_1>\chi_{1,2\alpha}^2\right)=\alpha. \nn
\)
Then the CEL-based confidence region for $\epsilon_G$ with asymptotic coverage $1-\alpha$ can be constructed as
\(
\{ \epsilon: T_c\leq \chi_{1,2\alpha}^2 \}.  
\)
Accordingly, we obtain the \(1-\alpha\) one-sided lower confidence bound 
\[
L_G=\inf\left\{ \epsilon: T_c \leq \chi_{1,2\alpha}^2\right\}.
\]
\(L_G\) serves as a data-driven violation-certification boundary answers whether the observed data provide sufficient evidence that the true disparity exceeds a prespecified allowable range.
If \(L_G>\epsilon_0\), the data provide sufficient evidence that the disparity exceeds the tolerance threshold, and the audited subpopulation is therefore flagged as exceeding the prescribed tolerance.
If \(0<L_G\leq\epsilon_0\), the data support the presence of a positive disparity but do not provide sufficient evidence that it exceeds the specified tolerance.
If \(L_G=0\), the data do not certify a positive disparity, and assessing a disparity in the opposite direction requires a separate violation-certification procedure.

In practice, the $p$-values are computed conservatively under the boundary $\epsilon_G = \epsilon_0$, which represents the least favorable case for rejection. Therefore, for any observed $T_c> 0$, the $p$-value is given by
\(
p = \mathbb{P} \left\{ \frac{1}{2} \chi^2_0 + \frac{1}{2} \chi^2_1 > T_c\right\} = \frac{1}{2} \big( 1 - F_1( T_c) \big), \label{p24} \nn
\)
where \( F_1(\cdot) \) is the cumulative distribution function (CDF) of $\chi_1^2$. 
If $ p \geq \alpha$, we do not reject the null hypothesis; otherwise, we reject it. 
In the case of $T_c= 0$, we have $ p = 1 $ , leading to non-rejection of the null hypothesis.

Corollary~\ref{coro-2} completes the operating characterization of violation certification by quantifying its sensitivity to local and fixed tolerance violations.
\begin{coro}\label{coro-2}
For the test $H_0:  0 \le \epsilon_G \le \epsilon_0$  versus $H_1:  \epsilon_G > \epsilon_0$,  
if the true performance disparity is \( \epsilon_G^* = \epsilon_0 + \tau {n_G}^{-1/2} \sigma_G \), where \( \tau > 0\), \(n_G=n\mathbb{P}_n(G)\) is the group-$G$ sample size, and \( \sigma_G^2 = { \mathbb{P}(G) }\Var(M\mid G)  \), then we have 
\[
\lim_{n_G \to \infty} \mathbb{P}\left( T_c> \chi_{1,2\alpha}^2 \mid \epsilon_G^* \right) = \Phi\left(  \tau\sqrt{\mathbb{P}(G)} - \sqrt{\chi_{1,2\alpha}^2} \right),
\] 
where $\Phi$ is the standard normal distribution function. 
Therefore, for any fixed \( \epsilon_G^* > \epsilon_0 \), 
\(
\lim_{n_G \to \infty} \mathbb{P}\left\{ T_c> \chi_{1,2\alpha}^2 \mid H_1 \right\} = 1.\)
\end{coro}

Together, Theorems~\ref{theo-1}--\ref{theo-2} and Corollary~\ref{coro-2} describe the violation-certification operating profile across the tolerance boundary, fixed interior null points, and violations.

When several subpopulations are audited simultaneously, CEL can be combined with the BH procedure to control the false flagging rate (FFR) across violation-certification decisions. The following theorem shows that the proposed auditing procedure controls false flagging.

 \begin{theo}[False flagging rate control] \label{theoFFR}
Assume that for all $G \in \mathcal{G}$, $\mathbb{P}(G)$ and $\Var(M|G)$ are bounded away from zero, and the parameter $\theta_P$ is known a priori. Furthermore, assume that the groups $\{G\}_{G \in \mathcal{G}}$ are disjoint. Then, if the BH(q) procedure is applied to the collection $\{G\}_{G \in \mathcal{G}}$, the FFR is asymptotically controlled at level $\alpha$.
\end{theo}

For disjoint groups $\{G\}_{G \in \mathcal{G}}$, $p$-values are independent, ensuring validity.

\subsection{Sensitivity-oriented screening} \label{sec-sel}

Violation certification and sensitivity screening serve different auditing objectives.
Violation certification supports a formal declaration that a disparity exceeds the prescribed tolerance and therefore prioritizes protection against false violation declarations.
Screening instead provides an early-warning signal for potentially problematic groups, for which missed violations may be more costly than additional alerts.
A screening alert should therefore not be interpreted as a formal violation certificate.

This distinction motivates an operating criterion different from least-favorable violation certification.
For screening, we seek a rule that retains a nondegenerate nominal asymptotic rejection probability at every fixed interior point of the tolerance region and meaningful sensitivity to local violations near the tolerance boundary.
At the boundary, the rule permits an explicitly characterized increase in rejection probability in exchange for greater detection sensitivity.
Consequently, the screening procedure does not claim uniform level-$\alpha$ control over the composite null.

We implement this screening objective through SEL, based on data-adaptive evidence aggregation.
The construction combines centered baseline evidence, boundary-oriented evidence, and a data-dependent gate that determines when the boundary component should be activated.
This design makes SEL an objective-specific auditing procedure rather than a power modification of the violation-certification test.

\subsubsection{Adaptive boundary-proxy inference}

The adaptive boundary-proxy principle combines centered baseline evidence with an additional component that captures departures from the tolerance boundary. A data-dependent gate activates this additional component only when the estimated disparity provides sufficient evidence of a tolerance violation. In the split-sample implementation below, paired differences supply the centered baseline evidence, while paired sums supply both the boundary-oriented evidence and the gate. Thus, adaptation occurs through the selective inclusion of boundary-oriented evidence, rather than through a change in the prespecified tolerance threshold.

To motivate the construction, we decompose the deviation from the boundary in the CEL statistic~\eqref{eq-logL-ineq}.
Let $\bar{g}(\epsilon_G)=n^{-1}\sum_{i=1}^n g_i(\epsilon_G)$. By the definition of $\hepsilon_G$, $\bar{g}(\epsilon_G)=\frac{n_G}{n}(\hepsilon_G-\epsilon_G)$.
It follows from the proof of Corollary~\ref{coro-2} that
\bea
\ell(\epsilon_0) = \left(\sqrt{n}\bar{g}(\epsilon_0) \over \sigma_G\right)^2 + o_p(1)
= \left(\frac{n_G}{\sqrt{n}\sigma_G}(\hepsilon_G-\epsilon_0)\right)^2 + o_p(1), \nn
\eea
where $\sigma_G^2 = { \mathbb{E} \left[(M - \theta_P - \epsilon_G^*)^2 \mathbbm{1}_{G} \right] } = { \mathbb{P}(G) }\Var(M\ |\ G)$.
Since
\( %\label{eq-mu-decomp}
\hepsilon_G-\epsilon_0=(\hepsilon_G-\epsilon_G^*)+(\epsilon_G^*-\epsilon_0),\nn
\)
we obtain
\bea
{\sqrt{n}\bar{g}(\epsilon_0) \over \sigma_G}
= \frac{n_G}{\sqrt{n}\sigma_G}(\hepsilon_G-\epsilon_G^*)+\frac{n_G}{\sqrt{n}\sigma_G}(\epsilon_G^*-\epsilon_0). \nn
\eea
The term $\hepsilon_G-\epsilon_G^*$ captures estimation error and is asymptotically centered, since ${\sqrt{n}\bar{g}(\epsilon_G^*) \over \sigma_G}=\frac{n_G(\hepsilon_G-\epsilon_G^*)}{\sqrt{n}\sigma_G}\tod Z$, whereas the boundary shift $\epsilon_G^*-\epsilon_0$ contributes a nonvanishing translation under local alternatives.
Motivated by this decomposition, the adaptive aggregation is designed to satisfy the following three properties:
\begin{enumerate}
\item[(i)] Under $H_0$, we construct a statistic that serves as a centered proxy for the estimation error $\hepsilon_G-\epsilon_G^*$, without requiring knowledge of the exact value of $\epsilon_G^*$, and that retains an asymptotic $\chi^2$ limit.
\item[(ii)] Under $H_1$, when the true shift $\epsilon_G^*-\epsilon_0$ is positive, we employ a gated statistic as a nonvanishing proxy for the boundary distance, thereby ensuring nontrivial power.
\item[(iii)] Under $H_0$, we introduce a threshold so that the boundary-distance proxy vanishes in probability when $\epsilon_G^* < \epsilon_0$.
\end{enumerate}
Together, these properties implement the screening criterion by retaining baseline evidence throughout the interior, activating boundary evidence under violations, and admitting an explicitly characterized increase in rejection probability at the tolerance boundary.

To implement these adaptive features, 
we partition the group-\(G\) subsample into two disjoint halves and pair observations across the two halves.
We use a primary estimator $\hepsilon_1$, defined below, to construct the test statistic and an auxiliary estimator $\hepsilon_2$ to supply the boundary-proxy component.
Without loss of generality, assume that $n_G=2m$ and that the index set $\{1,\ldots,n_G\}$ is partitioned into disjoint subsets $\{a_1,\ldots,a_m\}$ and $\{b_1,\ldots,b_m\}$.
For $i=1,\ldots,m$, define the paired contrasts
\[
\left\{
\begin{aligned}
\omega_{1i} &= g_{a_i}(0;\theta_P)-g_{b_i}(0;\theta_P) = M_{a_i}-M_{b_i}, \\
\omega_{2i} &= g_{a_i}(0;\theta_P)+g_{b_i}(0;\theta_P) = M_{a_i}+M_{b_i}-2\theta_P,
\end{aligned}
\right.
\]
In this construction, the difference contrast $\omega_{1i}$ serves as a proxy for estimation error, whereas the sum contrast $\omega_{2i}$ serves as a proxy for boundary shift.
Their sample means are
\(
\hepsilon_1=\frac{1}{m}\sum_{i=1}^m \omega_{1i},
\)
and
\(
\hepsilon_2=\frac{1}{m}\sum_{i=1}^m \omega_{2i}.
\)
Let $\epsilon_j^*=\mathbb{E}(\omega_{ji})$ and $\sigma_j^2=\Var(\omega_{ji})$ for $j=1,2$.
Under i.i.d.\ sampling with $\mathbb{E}(M_i-\theta_P\mid G)=\epsilon_G^*$ and $\Var(M_i-\theta_P\mid G)=\sigma_G^2/\mathbb{P}(G)$, we have
\(
\epsilon_1^*=0,\ \epsilon_2^*=2\epsilon_G^*,\ \sigma_1^2=\sigma_2^2=\frac{2\sigma_G^2}{\mathbb{P}(G)},
\)
and
\(
\Var(\hepsilon_1)=\Var(\hepsilon_2)=\frac{2\sigma_G^2}{m\,\mathbb{P}(G)}.
\)

For each $j\in\{1,2\}$, the empirical likelihood ratio for a generic mean value $\epsilon$ based on $\{\omega_{ji}\}_{i=1}^m$ is
\[
\begin{aligned}
R_j(\epsilon) =\sup_{p_i,1\le i\le m} \biggl\{
\prod_{i=1}^m (m p_i) 
\ \Big|\ 
p_i\ge 0, \ 
\sum_{i=1}^m p_i=1,\ 
\sum_{i=1}^m p_i (\omega_{ji}-\epsilon)=0 \ 
\biggr\},
\end{aligned}
\]
and let $\ell_j(\epsilon)=-2\log R_j(\epsilon)$ denote the corresponding empirical log-likelihood ratio statistic.
Although the transformed variables $\omega_{1i}$ and $\omega_{2i}$ are constructed from the same paired observations and hence are generally not independent in finite samples, they are orthogonal in the sense that their covariance vanishes. Together with the joint central limit theorem, this orthogonality implies that the limiting Gaussian components of $\hepsilon_1$ and $\hepsilon_2$ are independent. The quadratic expansion of the empirical likelihood ratios then transfers this asymptotic independence to the statistics $\ell_1(\epsilon)$ and $\ell_2(\epsilon)$. 
This independence greatly simplifies the subsequent asymptotic analysis.

To assess whether the adaptive boundary-proxy construction delivers the intended screening behavior, we first establish the limiting laws of its baseline and boundary-oriented evidence components.

\begin{theo}\label{theo-3}
Suppose that conditions~\eqref{C1}--\eqref{C2} hold and that, for some $\delta>0$, $\mathbb{E}|\omega_{ji}-\epsilon_j^*|^{2+\delta}<\infty$ for $j=1,2$.
Let $\epsilon_G^*$ denote the true performance disparity.
Then,
\[
\ell_j(\epsilon_j^*)\tod \chi^2_1 \quad \text{as} \quad m\to\infty, \quad j=1,2,
\]
where $\epsilon_1^*=0$, $\epsilon_2^*=2\epsilon_G^*$, and $\chi^2_1$ is a chi-squared distribution with one degree of freedom.
\end{theo}

Theorem~\ref{theo-3} supplies the component limits needed to analyze data-adaptive evidence aggregation.

Let
\(T_1=\ell_1(0),\, T_2=\ell_2(2\epsilon_0),\) and \(T_3=\sqrt{m}\,(\hepsilon_2 - 2\epsilon_0).\)
Here, $T_1$ supplies centered baseline screening evidence, $T_2$ measures evidence relative to the tolerance boundary, and $T_3$ acts as a data-dependent gate that determines whether the boundary-oriented component should be activated.
We define the SEL statistic by aggregating these components adaptively:
\[
T_s=T_1+T_2\,\mathbbm{1}(T_3>c_{\alpha_2}),
\]
where the gate threshold is chosen as $c_{\alpha_2}=\Phi^{-1}(1-\alpha_2)\,\sigma_2$, so that the gate is activated when $\hepsilon_2$ significantly exceeds $2\epsilon_0$ at significance level $\alpha_2$.

The indicator $\mathbbm{1}(T_3>c_{\alpha_2})$ implements the adaptive aggregation by activating $T_2$ only when the paired-sum estimator $\hepsilon_2$ provides sufficient evidence that the disparity exceeds the tolerance threshold.
Intuitively, in the interior of $H_0$, the gate is asymptotically inactive: since $\mathbbm{1}(T_3>c_{\alpha_2})\pto 0$ and $T_2\ge 0$, we have
\(
T_2 \mathbbm{1}(T_3>c_{\alpha_2})=o_p(1),
\)
and
\(
T_s = T_1 + o_p(1) \tod \chi^2_1.
\)
Thus, in the interior of $H_0$, $T_s$ adapts to $T_1$ and converges in distribution to a $\chi^2_1$ random variable.
We reject $H_0$ whenever $T_s>\chi_{1,\alpha}^2$.
The boundary behavior, where the gate remains active with nonvanishing probability, is characterized in Theorem~\ref{theo-4}.
Under the alternative $H_1$, because $T_1=O_p(1)$, $T_2=O_p(n)$, and $T_3$ diverges at rate $\sqrt{n}$, we have
\(
\mathbbm{1}(T_3>c_{\alpha_2})\pto 1,
\)
thus
\(
T_s \pto \infty,
\)
and $T_s$ diverges to infinity at the $O(n)$ rate.
Therefore, despite the use of sample splitting, the SEL test remains consistent under fixed alternatives.

We next determine whether the adaptive aggregation satisfies the operating criterion for sensitivity-oriented screening across the interior and boundary of the tolerance region.

%\noindent
\begin{theo}\label{theo-4}
Suppose that conditions~\eqref{C1}--\eqref{C2} hold. 
Let $\epsilon_G^*$ denote the true performance disparity and $c_{\alpha_2}=\Phi^{-1}(1-\alpha_2)\sigma_2$. 
Then, under the null hypothesis $H_0: 0 \le \epsilon_G \le \epsilon_0$, we have 
\bea
\lim_{n_G \to \infty} \mathbb{P}\left\{ T_s> \chi_{1,\alpha_1}^2 \right\} =
\begin{cases}
\alpha_1, & \mathrm{if}\ 0 \le \epsilon_G^* < \epsilon_0\\
\alpha_1 \Phi(z_{\alpha_2}) + \displaystyle \int_{z_{\alpha_2}}^\infty \bar F_{\chi_1^2}(\chi_{1,\alpha_1}^2-z^2)\,\phi(z)\,dz, & \mathrm{if} \   \epsilon_G^* =\epsilon_0
\end{cases} \nn
\eea 
where $\bar F_{\chi_1^2}(x)=\mathbb{P}(\chi_1^2>x)$. Moreover, under $ \epsilon_G^* = \epsilon_0$, we have
\bea
\lim_{n_G \to \infty} \mathbb{P}\left\{ T_s>\chi_{1,\alpha_1}^2 \right\}
=
\begin{cases}
\alpha_1+ (1-\alpha_1)\alpha_2, & \alpha_1\ge 2\alpha_2,\\
\alpha_1\left(\frac32-\alpha_2\right) +  \displaystyle\int_{z_{\alpha_2}}^{z_{\frac{\alpha_1}{2}}} \bar F_{\chi_1^2}(z_{\frac{\alpha_1}{2}}^2-z^2)\phi(z)\,dz,
& \alpha_1<2\alpha_2.
\end{cases} \nn
\eea
\end{theo}

%Theorem~\ref{theo-4} shows that the SEL test deliberately permits boundary size inflation while maintaining asymptotic rejection probability \(\alpha_1\) at every fixed interior null point.
Theorem~\ref{theo-4} characterizes the operating behavior induced by the adaptive boundary-proxy principle in its split empirical-likelihood implementation.
The proxy is asymptotically inactive at every fixed interior null point, remains nondegenerate at the tolerance boundary, and becomes active under violations.
Relative to least-favorable violation certification, this operating behavior yields less conservative screening in the interior and greater sensitivity near the tolerance boundary.
%Theorem~\ref{theo-4} shows that, at the boundary point $\epsilon_G^*=\epsilon_0$, the asymptotic size of the SEL test may exceed $\alpha_1$, leading to a mild size inflation. 
To calibrate the test to a prescribed nominal level $\alpha$, one can choose $\alpha_1$ as a function of $\alpha_2$. In the case $\alpha_1\ge 2\alpha_2$, the explicit calibration is
\(
\alpha_1=\frac{\alpha-\alpha_2}{1-\alpha_2}.
\)
When $\alpha_1<2\alpha_2$, no closed-form expression is available, and $\alpha_1$ can instead be obtained by numerically solving the corresponding limiting size equation in Theorem~\ref{theo-4}.
In fairness testing, the boundary $\epsilon_0$ represents the tolerance threshold. 
%A slightly elevated rejection probability at this boundary is therefore acceptable, as it corresponds to increased sensitivity near the limit of tolerable unfairness. 
%Accordingly, we set $\alpha_1=\alpha$, which maintains the Type~I error rate at level $\alpha$ throughout the interior region $0\le \epsilon_G^*<\epsilon_0$. 
%For sensitivity-oriented screening, we deliberately allow boundary size inflation in exchange for improved detection near the tolerance threshold.
For sensitivity-oriented screening, we admit this explicitly characterized increase in rejection probability at the tolerance boundary in exchange for improved detection near the threshold.
%Accordingly, we set \(\alpha_1=\alpha\), thereby prioritizing sensitivity near the tolerance boundary at the cost of explicitly characterized boundary size inflation.
Accordingly, we set \(\alpha_1=\alpha\), thereby retaining the nominal rejection probability at every fixed interior null point while prioritizing sensitivity near the tolerance boundary.
Because the tolerance boundary belongs to the composite null, this increase is formally an increase in the boundary Type~I rejection probability. Under the screening objective, however, it is interpreted as an explicit operating trade-off rather than as a violation-certification guarantee.
%Relative to LFP-calibrated CEL, this operating rule is less conservative at fixed interior null points and more sensitive near the tolerance boundary.
%These differences are consequences of the screening design rather than the primary motivation for SEL.
%By Theorem~\ref{theo-4}, the SEL-based liberal-screening inversion region for $\epsilon_G$ can be constructed as
By Theorem~\ref{theo-4}, the SEL screening inversion region for $\epsilon_G$ can be constructed as
\bea
\{ \epsilon: T_s \leq \chi_{1,\alpha}^2 \}. \nn
\eea

Interior and boundary behavior alone do not quantify the risk of missed violations. Corollary~\ref{coro-3} therefore characterizes screening sensitivity under local and fixed violations.

\begin{coro}\label{coro-3}
For the test $H_0:  0 \le \epsilon_G \le \epsilon_0$  versus $H_1:  \epsilon_G > \epsilon_0$,  
if the true performance disparity is \( \epsilon_G^* = \epsilon_0 + \tau n_G^{-1/2} \sigma_G \), where \( \tau > 0\), \(n_G=2m\) is the split group-$G$ sample size, and \( \sigma_G^2 = { \mathbb{P}(G) }\Var(M\mid G)  \), then we have 
\bea
\lim_{n_G \to \infty} \mathbb{P}\left\{ T_s>\chi_{1,\alpha_1}^2 \mid \epsilon_G^* \right\}
=  \alpha_1 \Phi(z_{\alpha_2}-\tau\sqrt{\mathbb{P}(G)}) + \int_{z_{\alpha_2}}^\infty \bar F_{\chi_1^2}(\chi_{1,\alpha_1}^2-v^2)\,\phi(v-\tau\sqrt{\mathbb{P}(G)})\,dv. \nn
\eea
More explicitly, this limiting rejection probability admits the following two forms. If $\alpha_1\ge 2\alpha_2$, then
\bea
\lim_{n_G \to \infty} \mathbb{P}\left\{ T_s>\chi_{1,\alpha_1}^2 \mid \epsilon_G^*  \right\}
=1-(1-\alpha_1)\Phi(z_{\alpha_2}-\tau\sqrt{\mathbb{P}(G)}). \nn
\eea
If $\alpha_1<2\alpha_2$, then
\bea
\begin{aligned}
\lim_{n_G \to \infty} \mathbb{P}\left\{ T_s>\chi_{1,\alpha_1}^2  \mid \epsilon_G^*  \right\}
={}&
\alpha_1\Phi(z_{\alpha_2}-\tau\sqrt{\mathbb{P}(G)})
+\int_{z_{\alpha_2}}^{z_{\frac{\alpha_1}{2}}}
\bar F_{\chi_1^2}(\chi_{1,\alpha_1}^2-v^2)\phi(v-\tau\sqrt{\mathbb{P}(G)})\,dv \\
&+1-\Phi(z_{\frac{\alpha_1}{2}}-\tau\sqrt{\mathbb{P}(G)}).
\end{aligned}\nn
\eea
Moreover, for any fixed \( \epsilon_G^* > \epsilon_0 \), 
\(
\lim_{n_G \to \infty} \mathbb{P}\left\{ T_s>\chi_{1,\alpha_1}^2 \mid \epsilon_G^* \right\} = 1.
\)
\end{coro}

Together, Theorems~\ref{theo-3}--\ref{theo-4} and Corollary~\ref{coro-3} establish the operating profile required for the basic adaptive boundary-proxy screening procedure.

\subsubsection{Power-enhanced screening}
%Let \( \widetilde{T}_2=\ell_2(0)\). We define the adjusted split empirical likelihood statistic (ASEL) as follows:
ASEL provides a power-enhanced implementation of the adaptive boundary-proxy principle by replacing the boundary component $T_2$ with $\widetilde{T}_2=\ell_2(0)$, which is more strongly separated from its reference value under violations.
The resulting statistic is
\[
T_a=T_1+\widetilde{T}_2\,\mathbbm{1}(T_3>c_{\alpha_2}),
\]
where $c_{\alpha_2}=\Phi^{-1}(1-\alpha_2)\,\sigma_2$. Intuitively, in the interior of $H_0$, $T_a$ adapts to $T_1$ and converges in distribution to a $\chi^2_1$ random variable.
Under fixed alternatives $H_1$ with $\epsilon_G^*>\epsilon_0$, we have $T_1=O_p(1)$, whereas $T_3$ diverges to $+\infty$. Hence,
\(
\mathbbm{1}(T_3>c_{\alpha_2})\pto 1.
\)
In addition, the empirical likelihood ratio $\widetilde T_2=\ell_2(0)$ is centered farther away from the true mean of $\omega_{2i}$ than $T_2=\ell_2(2\epsilon_0)$. More precisely,
\(
\widetilde T_2-T_2\pto +\infty .
\)
Thus, $\widetilde T_2$ provides a stronger signal against the null under such fixed alternatives. Consequently,
\(
T_a-T_s=(\widetilde T_2-T_2)\mathbbm{1}(T_3>c_{\alpha_2})
\pto +\infty,
\)
and therefore
\(
\mathbbm{1}(T_a>T_s)\pto 1.
\)
This shows that, under fixed alternatives, $T_a$ is asymptotically larger than $T_s$ with probability tending to one. Hence, when the same rejection threshold is used, $T_a$ has asymptotically no smaller rejection probability than $T_s$, and typically yields higher power.
When $\epsilon_0=0$, we have $T_s=T_a$. We therefore focus on the case $\epsilon_0>0$ in what follows.
A power-enhanced implementation is useful for sensitivity-oriented auditing only if it preserves the screening operating criterion. The following result therefore characterizes ASEL in the interior, at the tolerance boundary, and under local violations.
\begin{theo}\label{theo-5}
Suppose that $\epsilon_0>0$ and conditions~\eqref{C1}--\eqref{C2} hold. 
Let $\epsilon_G^*$ denote the true performance disparity and $c_{\alpha_2}=\Phi^{-1}(1-\alpha_2)\sigma_2$. 
Then, under the null hypothesis $H_0: 0 \le \epsilon_G \le \epsilon_0$, we have 
\bea
\lim_{n_G \to \infty} \mathbb{P}\left\{ T_a> \chi_{1,\alpha_1}^2 \right\} =
\begin{cases}
\alpha_1, & \mathrm{if}\ 0 \le \epsilon_G^* < \epsilon_0\\
\alpha_1+(1-\alpha_1)\alpha_2, & \mathrm{if} \   \epsilon_G^* =\epsilon_0
\end{cases} \nn
\eea 
and under the alternative hypothesis $H_1:  \epsilon_G^* = \epsilon_0 + \tau n_G^{-1/2} \sigma_G $, where \( \tau > 0\), \(n_G=2m\) is the split group-$G$ sample size, and \( \sigma_G^2 = { \mathbb{P}(G) }\Var(M\mid G)  \), then we have 
\bea
\lim_{n_G \to \infty} \mathbb{P}\left\{ T_a>\chi_{1,\alpha_1}^2 \mid \epsilon_G^* \right\}
= 
1-(1-\alpha_1)\Phi(z_{\alpha_2}-\tau\sqrt{\mathbb{P}(G)}). \nn
\eea
\end{theo}
Theorem~\ref{theo-5} shows that ASEL preserves the interior and boundary operating structure of the adaptive boundary-proxy principle while strengthening detection under violations.
By Theorem~\ref{theo-5}, the ASEL screening inversion region for $\epsilon_G$ can be constructed as
\bea
\{ \epsilon: T_a \leq \chi_{1,\alpha}^2 \}. \nn
\eea

The theoretical rejection ratios of the EL, CEL, SEL, and ASEL tests are displayed in Figure~\ref{fig-theory-reject-ratio}.
We also consider three modified procedures, denoted by CEL2, SEL2, and ASEL2.
These variants correspond, respectively, to $T_c$ with $\epsilon_0=0$, $T_s$ with the adjusted critical value $\chi^2_{1,\alpha_1}$, where $\alpha_1=(\alpha-\alpha_2)/(1-\alpha_2)$, and $T_a$ with the adjusted critical value $\chi^2_{1,\alpha_1}$, where $\alpha_1=0.027$ and $\alpha_2=0.049$.
As shown in the figure, the CEL2 procedure, which corresponds to the exact-zero fairness hypothesis, yields a relatively high rejection ratio within the tolerance region.
Although CEL and SEL2 attain the maximal size at the boundary point, SEL2 alleviates the excessive conservativeness exhibited by CEL within the null region.
ASEL achieves the highest power under the alternative.
In practice, the true disparity is unknown, and one may not wish to reject the null hypothesis as aggressively as CEL2 when the disparity deviates only slightly from zero.
In this sense, CEL is the most conservative procedure, whereas SEL and ASEL are less conservative alternatives.

\begin{figure}[H]
\centering
\includegraphics[width=0.6\linewidth]{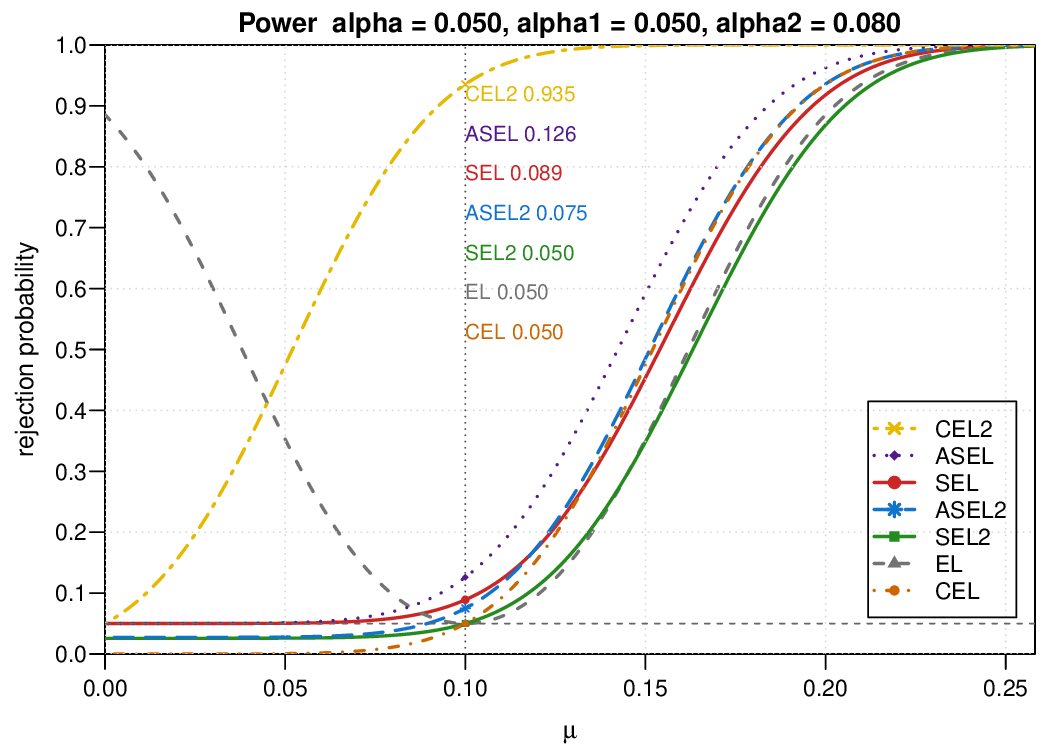}
\caption{The theoretical limiting rejection probabilities of the EL, CEL, SEL, and ASEL statistics, together with the CEL2, SEL2, and ASEL2 variants. The threshold is set to $0.1$, with $\alpha=\alpha_1=0.05$ and $\alpha_2=0.08$.}
\label{fig-theory-reject-ratio}
\end{figure}

\section{Simulations}\label{sec-sim}

% This section evaluates the finite-sample behavior of CEL, SEL, and ASEL in single-group and simultaneous subgroup auditing. We first examine rejection frequencies and power under normal, heavy-tailed, and skewed distributions, then assess false flagging rates and power when the procedures are combined with BH adjustment. These experiments illustrate the trade-offs between error control and detection sensitivity across the two auditing objectives.
This section evaluates the finite-sample behavior of CEL, SEL, and ASEL in single-group and simultaneous subgroup auditing. Section~\ref{sec-Single} examines rejection frequencies and power under normal, heavy-tailed, and skewed distributions. Section~\ref{sec-multiplicity} assesses false flagging rates and power when the procedures are combined with BH adjustment. These experiments illustrate the trade-offs between error control and detection sensitivity across the two auditing objectives. The code related to this article is available at \url{https://github.com/Tang-Jay/TFA}.

\subsection{Single group auditing}\label{sec-Single}

We conduct simulation studies to evaluate the finite-sample power of the one-sided test \eqref{test-1}, 
comparing the CEL, SEL, and ASEL tests with the EL statistic $\ell(\epsilon_0)$ and the one-sided Student $t$ test.
Standard EL is included as an exact-equality benchmark rather than as a tolerance-certification procedure.
Without loss of generality, we set $\theta_P=0$ in the simulations.
We set $\epsilon_0=0.25$, $\alpha=0.05$, and the data-dependent threshold $c_{\alpha_2}=\Phi^{-1}(1-\alpha_2)\,{\hat{\sigma}}_{G_2}$, where $\alpha_2=0.08$ and $\hat{\sigma}_2$ is the sample standard deviation of $\{\omega_{2i}\}_{i=1}^m$.
SEL rejects when $T_s>\chi_{1,\alpha}^2$ and ASEL rejects when $T_a>\chi_{1,\alpha}^2$; EL rejects when $\ell(\epsilon_0)>\chi_{1,\alpha}^2$; CEL rejects when $T_c>\chi_{1,2\alpha}^2$; and the $t$ test rejects when the one-sided $p$-value is below $\alpha$. 
% All simulation results are based on 2,000 replications. The corresponding results are presented in Figures~\ref{fig-Power-N}--\ref{fig-Power-E}.
The single-group simulation results are based on 2,000 replications and are presented in Figures~\ref{fig-Power-N}--\ref{fig-Power-E}.
In each panel, the curves correspond to CEL, ASEL, SEL,  EL, and the one-sided $t$ test. 

We consider three data-generating processes, presented in the order normal, heavy-tailed, and skewed.
Under the symmetric normal benchmark, independent observations are generated as $M_i\sim N(\epsilon,1)$ with $\epsilon\in\{0,0.01,\ldots,0.5\}$.
For heavy tails, we take $M_i=\epsilon+Z_i$ with $Z_i\sim t(3)$ on the same grid of $\epsilon_G$.
For skewness, we consider 
a right-skewed exponential model $M_i\sim \mathrm{Exp}(\frac{1}{\epsilon})$. Since these families require $\epsilon>0$, we use $\epsilon\in\{0.05,0.06,\ldots,0.45\}$. 
In both designs, $\mathbb{E}(M_i)=\epsilon$.

% Figure 1: normal
\begin{figure}[H]
\centering
\begin{subfigure}[b]{0.25\textwidth}
\centering
\includegraphics[width=\linewidth]{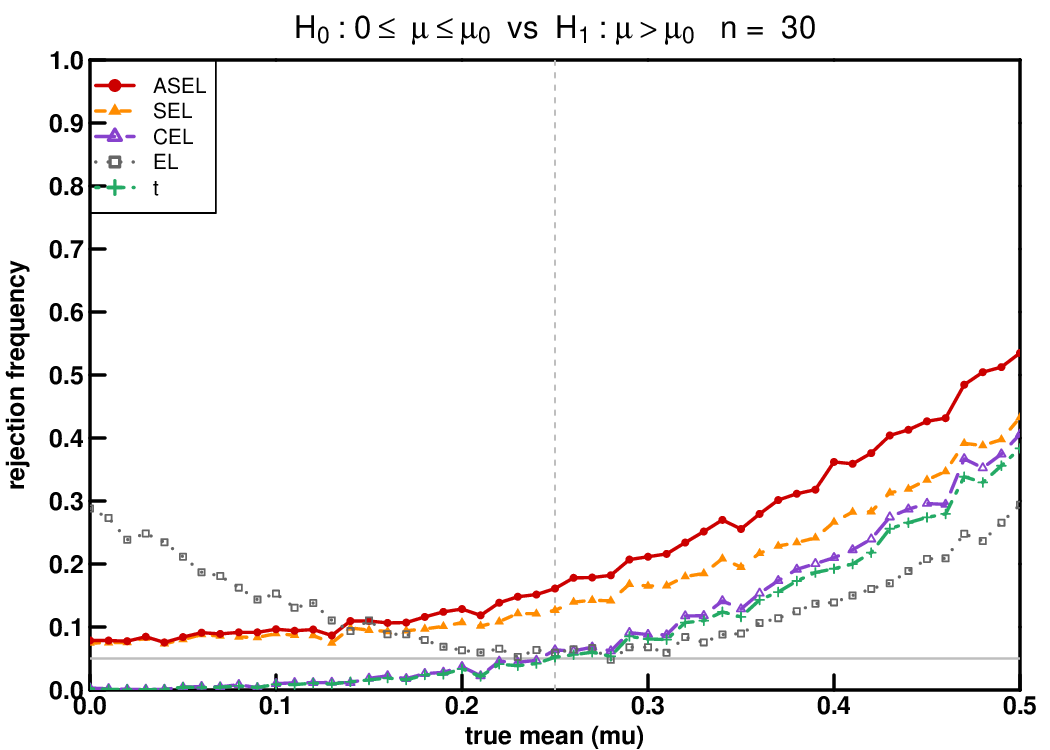}
\subcaption{}
\label{fig-Power-N-a}
\end{subfigure}\hfill
\begin{subfigure}[b]{0.25\textwidth}
\centering
\includegraphics[width=\linewidth]{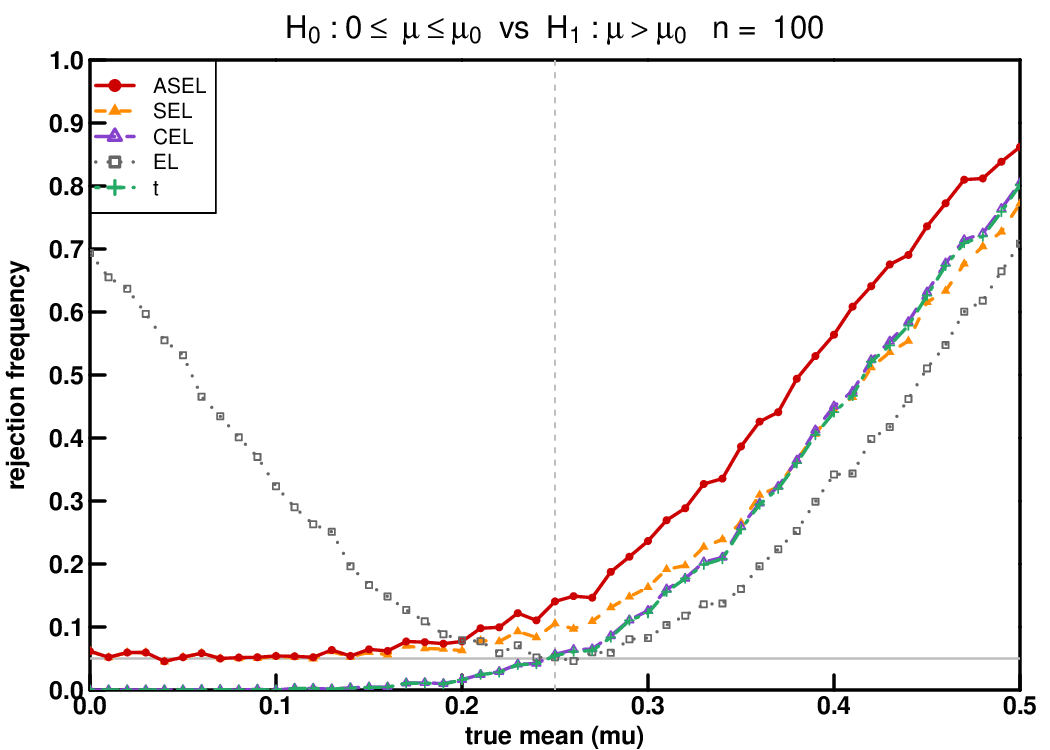}
\subcaption{}
\label{fig-Power-N-b}
\end{subfigure}\hfill
\begin{subfigure}[b]{0.25\textwidth}
\centering
\includegraphics[width=\linewidth]{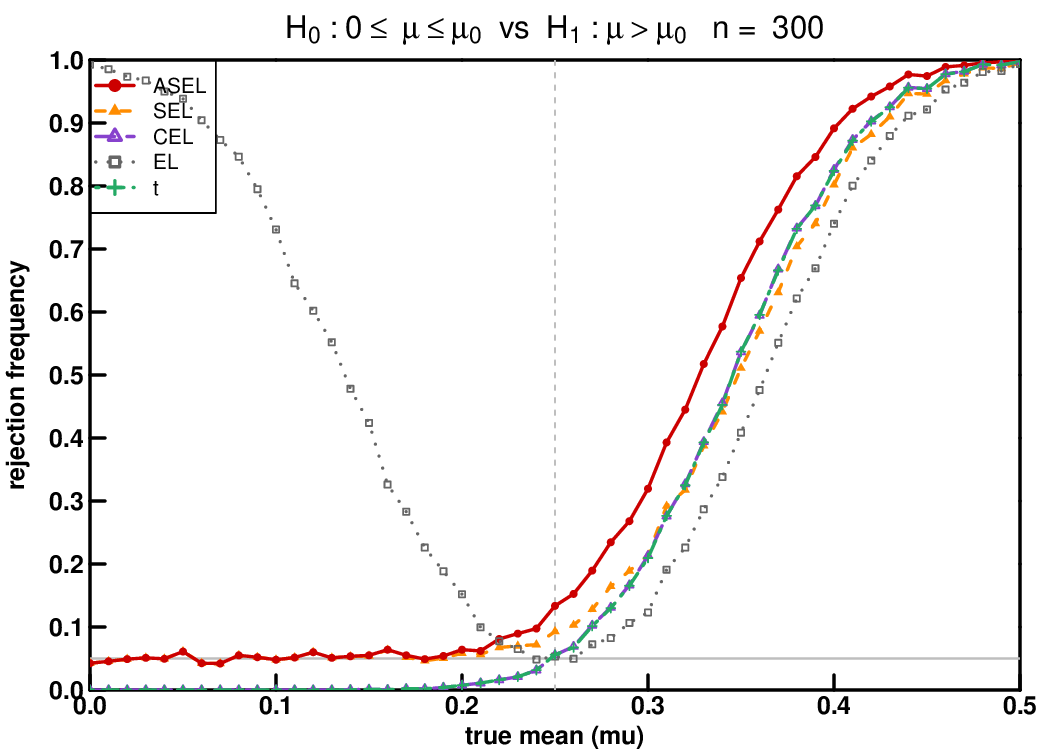}
\subcaption{}
\label{fig-Power-N-c}
\end{subfigure}\hfill
\begin{subfigure}[b]{0.25\textwidth}
\centering
\includegraphics[width=\linewidth]{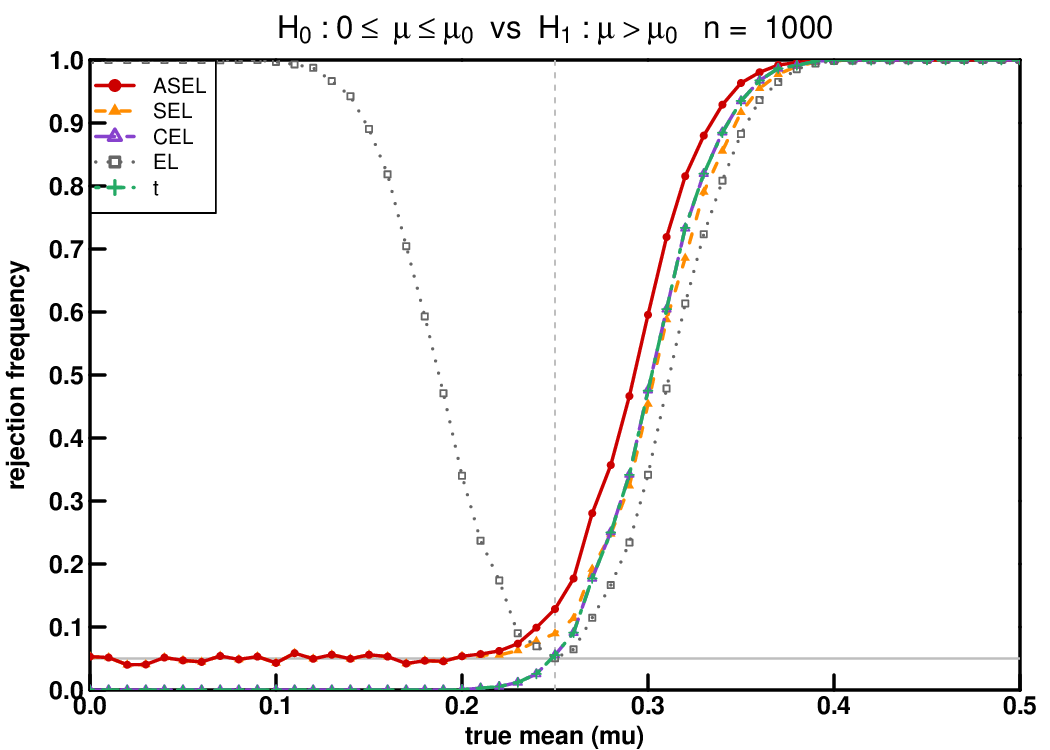}
\subcaption{}
\label{fig-Power-N-d}
\end{subfigure}
\caption{Rejection frequency plots under the symmetric normal model: (a)--(d) correspond to $n=30,100,300,1000$ with $N(\epsilon,1)$ data.
The horizontal gray line is the nominal level $\alpha=0.05$ and the vertical dashed line marks $\epsilon=0.25$.}
\label{fig-Power-N}
\end{figure}

Under the tolerance null $0\leq\epsilon_G^*\leq\epsilon_0$, rejection frequencies should be interpreted in light of the procedures' distinct operating criteria.
For CEL, the relevant criterion is protection against false violation declarations across the tolerance region.
For SEL and ASEL, the screening criterion allows nondegenerate rejection probabilities at fixed interior null points and an explicitly characterized increase at the tolerance boundary,
without targeting uniform level-$\alpha$ control over the composite null.
Subfigures~\eqref{fig-Power-N-a}--\eqref{fig-Power-N-d} illustrate these differences.
As $\epsilon_G^*$ moves farther into the interior of $H_0$, EL rejection frequencies approach one, whereas those of CEL and the one-sided $t$ test approach zero.
For SEL and ASEL, rejection frequencies at fixed interior null points approach the nominal level as the sample size increases.
At the tolerance boundary, their rejection frequencies exceed the nominal level, consistent with the asymptotic limits in Theorems~\ref{theo-4}--\ref{theo-5}.
These boundary exceedances reflect the screening trade-off between detection sensitivity and false alerts.

% Figure 2: heavy-tailed t(3)
\begin{figure}[!htbp]
\centering
\begin{subfigure}[b]{0.25\textwidth}
\centering
\includegraphics[width=\linewidth]{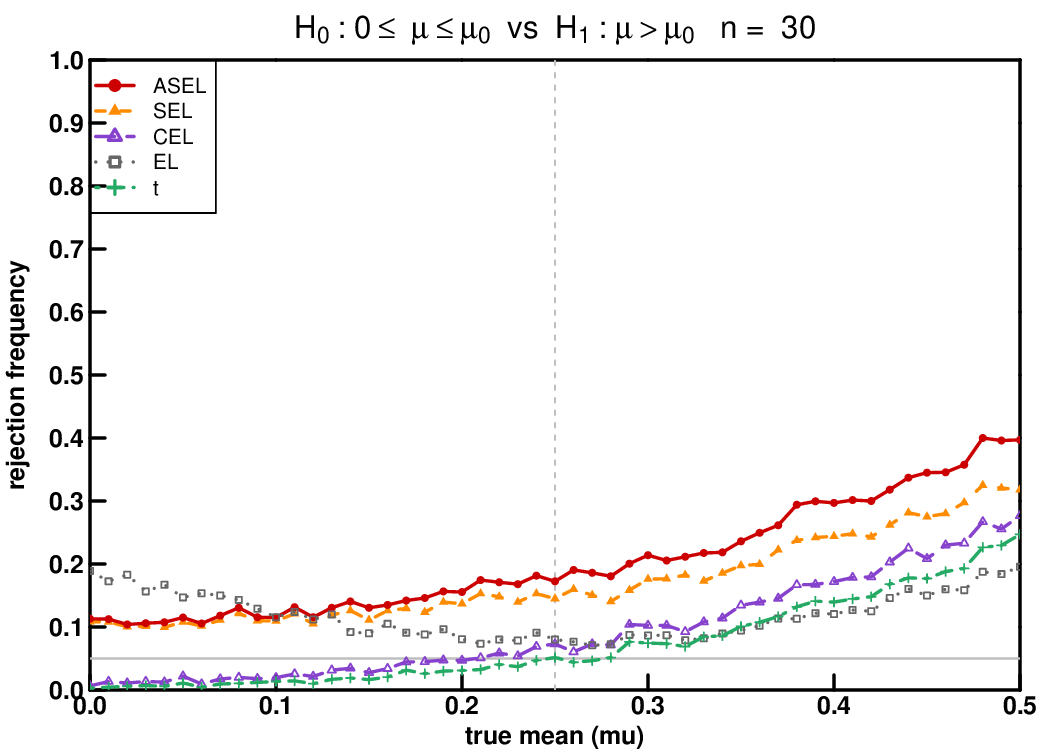}
\subcaption{}
\label{fig-Power-T-a}
\end{subfigure}\hfill
\begin{subfigure}[b]{0.25\textwidth}
\centering
\includegraphics[width=\linewidth]{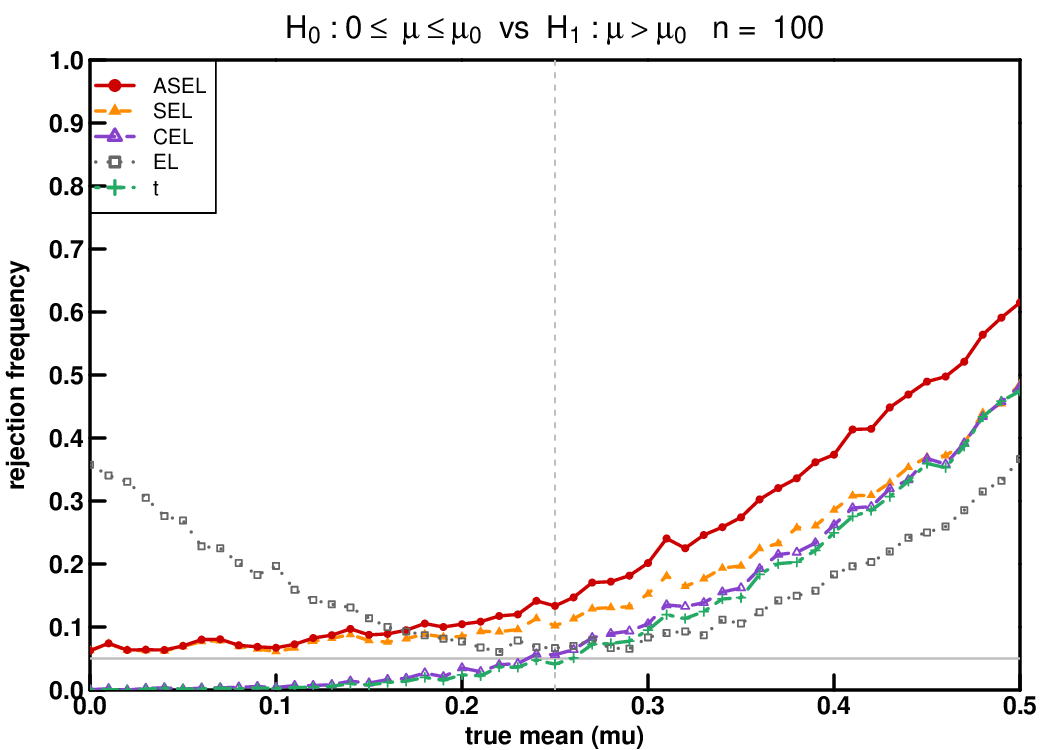}
\subcaption{}
\label{fig-Power-T-b}
\end{subfigure}\hfill
\begin{subfigure}[b]{0.25\textwidth}
\centering
\includegraphics[width=\linewidth]{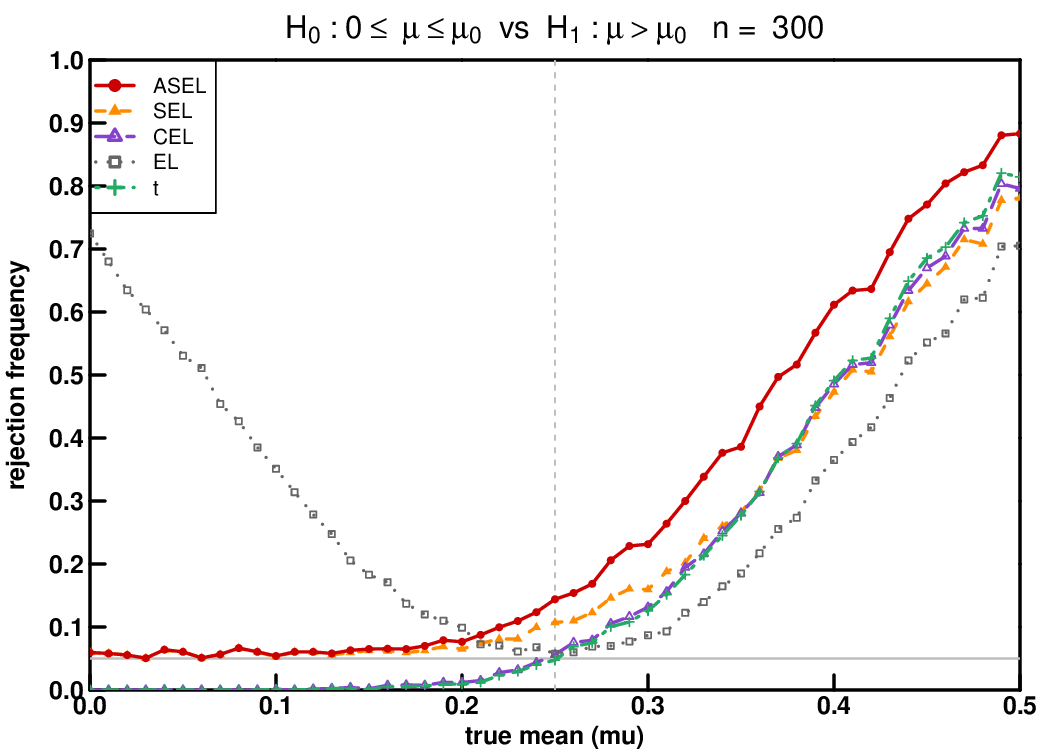}
\subcaption{}
\label{fig-Power-T-c}
\end{subfigure}\hfill
\begin{subfigure}[b]{0.25\textwidth}
\centering
\includegraphics[width=\linewidth]{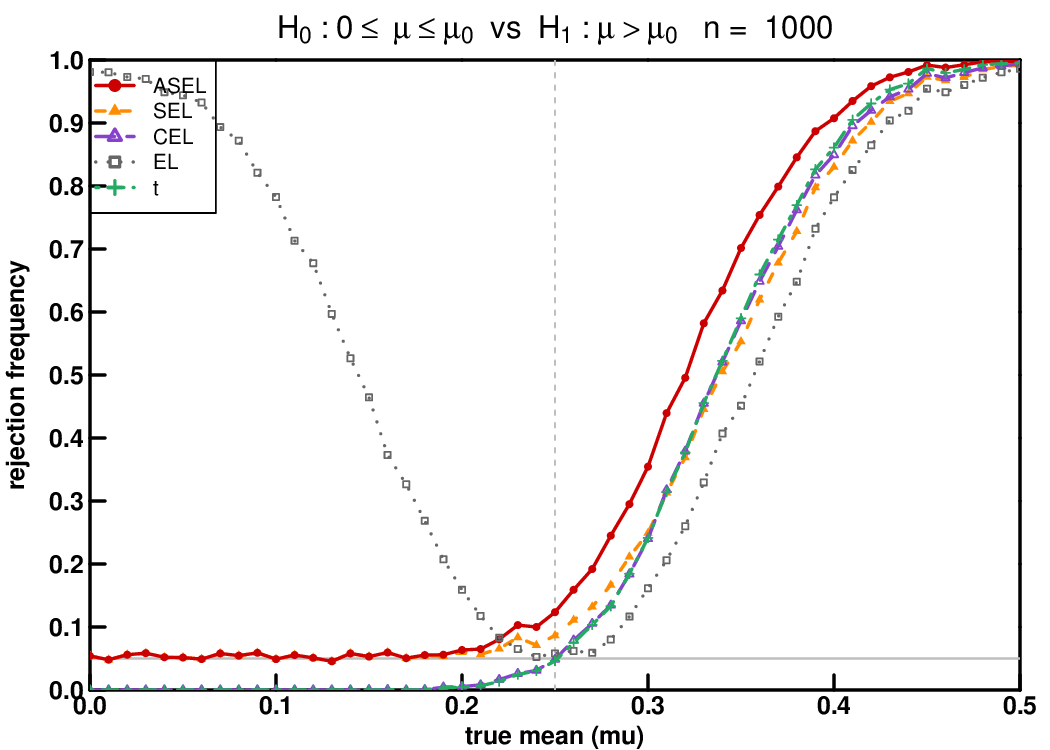}
\subcaption{}
\label{fig-Power-T-d}
\end{subfigure}
\caption{Rejection frequency plots under the heavy-tailed Student~$t$ model: (a)--(d) correspond to $n=30,100,300,1000$ with $t(3)+\epsilon$ data.
The horizontal gray line is the nominal level $\alpha=0.05$ and the vertical dashed line marks $\epsilon=0.25$.}
\label{fig-Power-T}
\end{figure}

Figure~\ref{fig-Power-T} reports rejection frequency curves under the heavy-tailed design.
Because the location-shifted $t(3)$ distribution remains symmetric about $\epsilon_G$, the qualitative patterns resemble those under normality.
Subfigures~\eqref{fig-Power-T-a}--\eqref{fig-Power-T-d} show that EL has the lowest power overall.
Subfigures~\eqref{fig-Power-T-a} and~\eqref{fig-Power-T-b} indicate that CEL is more powerful than the one-sided $t$ test for small sample sizes.
Subfigures~\eqref{fig-Power-T-c} and~\eqref{fig-Power-T-d} show that, for larger sample sizes, SEL and ASEL attain higher power under the alternative while maintaining rejection rates close to the nominal level at fixed interior null points.

Figure~\ref{fig-Power-E} reports rejection frequency curves under the skewed design.
When $\epsilon_G^*>\epsilon_0$, Subfigures~\eqref{fig-Power-E-a} and~\eqref{fig-Power-E-b} show that, for small and moderate sample sizes, the one-sided $t$ test is sensitive to skewness and exhibits reduced power, whereas CEL is more powerful than the one-sided $t$ test and is more robust to skewness.
Subfigures~\eqref{fig-Power-E-a}--\eqref{fig-Power-E-d} show that SEL improves power near the boundary $\epsilon_0$, and ASEL attains the highest power among the five procedures.
When $0.05\le\epsilon_G^*\le\epsilon_0$, as $\epsilon_G^*$ recedes further from the boundary $\epsilon_0$ into the interior of $H_0$, CEL and the one-sided $t$ test become increasingly conservative with rejection rates approaching zero, whereas SEL and ASEL maintain rejection rates close to the nominal level as $n$ increases.

% Figure 4: right-skewed exponential
\begin{figure}[!htbp]
\centering
\begin{subfigure}[b]{0.25\textwidth}
\centering
\includegraphics[width=\linewidth]{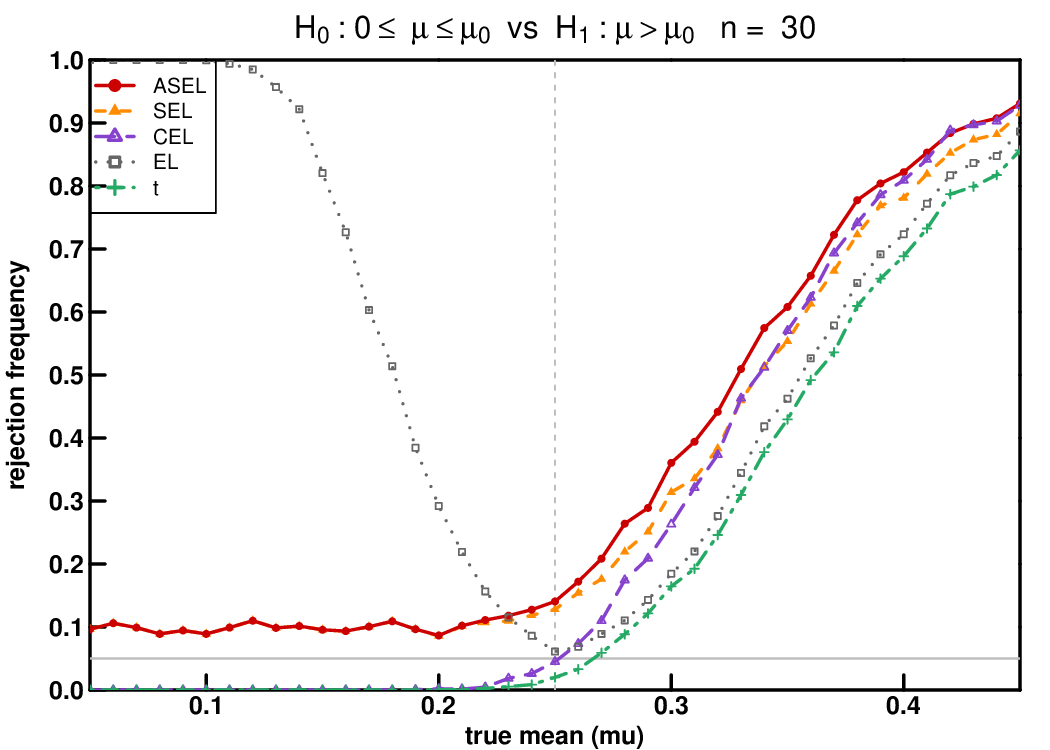}
\subcaption{}
\label{fig-Power-E-a}
\end{subfigure}\hfill
\begin{subfigure}[b]{0.25\textwidth}
\centering
\includegraphics[width=\linewidth]{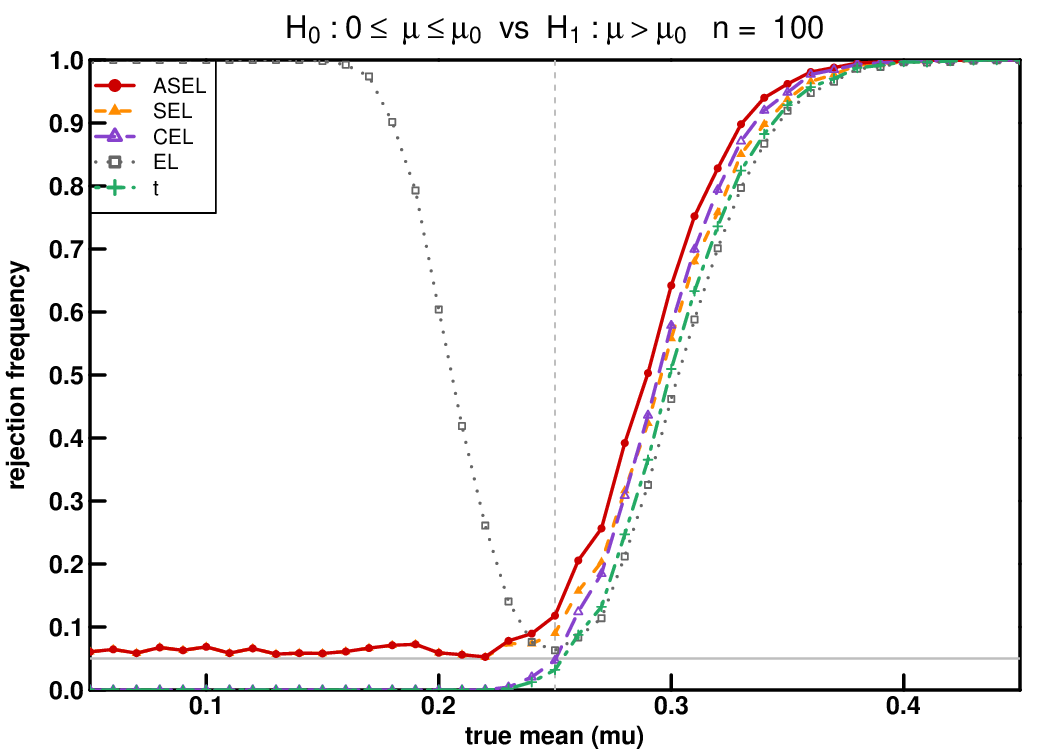}
\subcaption{}
\label{fig-Power-E-b}
\end{subfigure}\hfill
\begin{subfigure}[b]{0.25\textwidth}
\centering
\includegraphics[width=\linewidth]{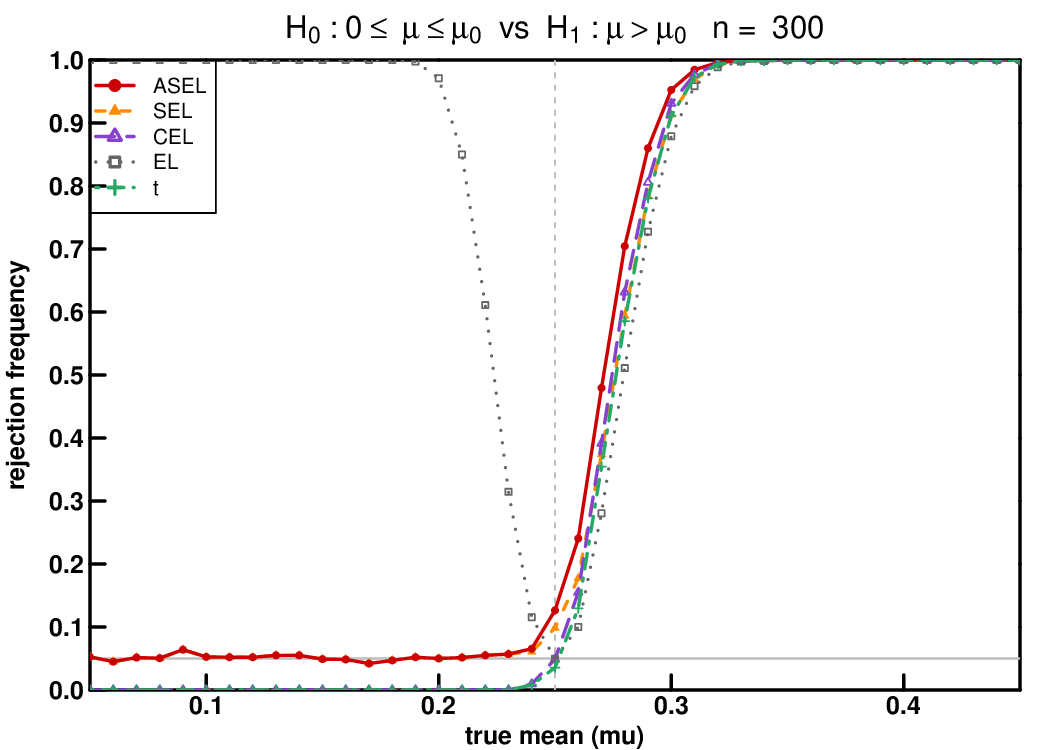}
\subcaption{}
\label{fig-Power-E-c}
\end{subfigure}\hfill
\begin{subfigure}[b]{0.25\textwidth}
\centering
\includegraphics[width=\linewidth]{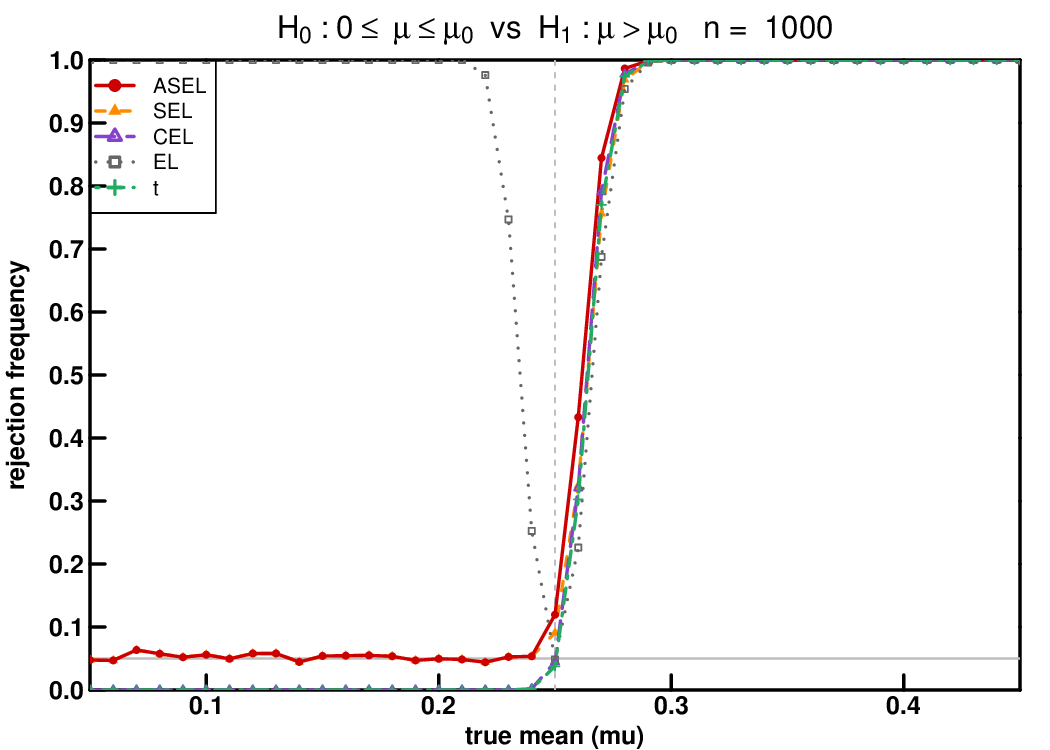}
\subcaption{}
\label{fig-Power-E-d}
\end{subfigure}
\caption{Rejection frequency plots under the right-skewed exponential design $\mathrm{Exp}(\frac{1}{\epsilon})$: (a)--(d) correspond to $n=30,100,300,1000$.
The horizontal gray line is the nominal level $\alpha=0.05$ and the vertical dashed line marks $\epsilon=0.25$.}
\label{fig-Power-E}
\end{figure}

% Across these designs, CEL improves on the one-sided $t$ test in small samples and exhibits greater robustness under skewed distributions.
Across these designs, CEL controls false violation declarations at the prescribed level while achieving higher power than the one-sided $t$ test in small samples.
SEL maintains rejection probabilities close to the nominal level at fixed interior null points as $n$ increases, while ASEL delivers the strongest power gains among the proposed procedures.
These results illustrate the distinct operating characteristics of violation certification and sensitivity-oriented screening.
CEL provides conservative protection against false violation declarations, whereas SEL and ASEL maintain nondegenerate interior rejection probabilities and exhibit greater detection
sensitivity in the simulation designs considered, with increased rejection probabilities at the tolerance boundary.
These operating characteristics make SEL and ASEL better aligned with sensitivity-oriented screening settings in which missed violations are of primary concern.
CEL is better aligned with violation-certification settings in which false declarations require stronger protection.

\subsection{Simultaneous subgroup auditing}\label{sec-multiplicity}

% Experimental purpose
We next compare false discoveries and detection power when CEL, SEL, ASEL, and the one-sided $t$ test are each combined with BH for simultaneous auditing.
The design uses disjoint groups and a known reference value, as in Theorem~\ref{theoFFR} for CEL.
Including SEL and ASEL assesses how their screening behavior translates into the joint discovery process under the same multiplicity adjustment.

% Data-generating mechanisms
In each replication, we draw $n=1000$ independent uniform random variables $U_i\sim\mathrm{Unif}(0,1)$.
We partition $[0,1)$ into $K=26$ equal-length and disjoint intervals and assign observation $i$ to group $G_j$ with $j=1,\ldots,K$.
Each configuration is evaluated over $B=100$ replications, with all methods applied to the same observations within each replication.
We fix the reference value at $\theta_P=0$ and the tolerance at $\epsilon_0=0.25$.
Conditional on group membership, observations are generated independently from either $N(\mu_j,1)$ or $\mathrm{Exp}(1/\mu_j)$, where $\mu_j=\epsilon_j$ is the mean in group $G_j$.
The conditional variances are $1$ and $\mu_j^2$, respectively, so comparisons across these designs reflect differences in scale as well as skewness.
For each group, we test $H_{0j}:0\leq\epsilon_j\leq\epsilon_0$ against $H_{1j}:\epsilon_j>\epsilon_0$ and set
\[
\epsilon_j=
\begin{cases}
\epsilon \, (\le \epsilon_0), & \pi_1=0,\\
\epsilon_0, & \pi_1>0,\ j\leq K(1-\pi_1),\\
\epsilon_0+\delta, & \pi_1>0,\ j>K(1-\pi_1).
\end{cases}
\]
Here, $\pi_1$ is the proportion of violating groups, and $\delta>0$ is the excess above the tolerance for violating groups.
Under the global null ($\pi_1=0$), we vary $\epsilon$ over $\{0,0.05,\ldots,0.25\}$ for the normal design and $\{0.05,0.10,\ldots,0.25\}$ for the exponential design.
Under the alternative configurations, we consider $\pi_1\in\{0.5,1\}$ with $\delta\in\{0.05,0.10,\ldots,0.60\}$.

% Evaluation criteria and Monte Carlo precision
We set the BH target level to $q=0.05$ and the SEL/ASEL gate parameter to $\alpha_2=0.08$.
The BH inputs are the one-sided boundary-calibrated $p$-values for CEL, the one-sided $p$-values for the $t$ test, and the $\chi_1^2$ upper-tail scores for SEL and ASEL without additional boundary recalibration.
We apply BH to the unadjusted SEL and ASEL screening scores as an exploratory benchmark, without claiming FFR control.
Let $R_b$ and $V_b$ denote the total and false discoveries in replication $b$, and let $m_1=K\pi_1$ denote the number of false null hypotheses.
We report
\[
\widehat{\mathrm{FFR}}=\frac{1}{B}\sum_{b=1}^{B}\frac{V_b}{\max(R_b,1)},
\qquad
\widehat{\mathrm{Power}}=\frac{1}{B}\sum_{b=1}^{B}\frac{R_b-V_b}{m_1},\quad m_1>0.
\]
A lower FFR indicates a smaller proportion of false flags among flagged groups on average, whereas higher power indicates a smaller proportion of missed violations among truly violating groups on average.

% Case one
Under the global null $\pi_1=0$, every rejection is a false flagging, so $V_b=R_b$ and the FFR equals the probability of falsely flagging at least one group in an audit and Power is undefined.
Table~\ref{tab:fdr-global-null} shows that the empirical FFR of CEL-BH approaches the nominal level as the true disparity approaches the tolerance boundary under both designs. 
In contrast, the empirical FFR of $t$-BH remains below the nominal level under the normal design and, notably, is zero under the exponential design, highlighting its pronounced conservativeness in the latter setting.
At the group sample sizes considered, both SEL-BH and ASEL-BH exhibit substantial FFR inflation. 

\begin{table}[!htbp]
\centering
\caption{Empirical FFR under the global null $\pi_1=0$. Bold entries correspond to the tolerance boundary.}
\label{tab:fdr-global-null}
\footnotesize
\setlength{\tabcolsep}{3pt}
\begin{tabular}{@{}crrrr@{\hspace{1.5em}}crrrr@{}}
\toprule
\multicolumn{5}{c}{$N(\epsilon,1)$} &
\multicolumn{5}{c}{$\mathrm{Exp}(1/\epsilon)$} \\
\cmidrule(r){1-5}\cmidrule(l){6-10}
$\epsilon$ & CEL-BH & SEL-BH & ASEL-BH & $t$-BH &
$\epsilon$ & CEL-BH & SEL-BH & ASEL-BH & $t$-BH \\
\midrule
0.00 & 0.00 & 0.23 & 0.30 & 0.00 & 0.00 & - & - & -  & -\\
0.05 & 0.00 & 0.18 & 0.26 & 0.00 & 0.05 & 0.00 & 0.32 & 0.32 & 0.00 \\
0.10 & 0.00 & 0.23 & 0.38 & 0.00 & 0.10 & 0.00 & 0.25 & 0.25 & 0.00 \\
0.15 & 0.01 & 0.28 & 0.56 & 0.01 & 0.15 & 0.00 & 0.26 & 0.26 & 0.00 \\
0.20 & 0.04 & 0.18 & 0.70 & 0.03 & 0.20 & 0.00 & 0.40 & 0.42 & 0.00 \\
\textbf{0.25} & \textbf{0.05} & \textbf{0.27} & \textbf{0.90} & \textbf{0.03} &
\textbf{0.25} & \textbf{0.05} & \textbf{0.37} & \textbf{0.85} & \textbf{0.00} \\
\bottomrule
\end{tabular}
\end{table}

\begin{table}[!htbp]
\centering
\caption{Empirical power for selected configurations with $\pi_1=1$.}
\label{tab:power-all-alternative}
\footnotesize
\setlength{\tabcolsep}{3pt}
\begin{tabular}{@{}crrrr@{\hspace{1.5em}}crrrr@{}}
\toprule
\multicolumn{5}{c}{$N(\epsilon,1)$} &
\multicolumn{5}{c}{$\mathrm{Exp}(1/\epsilon)$} \\
\cmidrule(r){1-5}\cmidrule(l){6-10}
$\delta$ & CEL-BH & SEL-BH & ASEL-BH & $t$-BH &
$\delta$ & CEL-BH & SEL-BH & ASEL-BH & $t$-BH \\
\midrule
0.05 & 0.0092 & 0.0319 & 0.1562 & 0.0054 &
0.05 & 0.1042 & 0.1408 & 0.3692 & 0.0058 \\
0.20 & 0.1165 & 0.1785 & 0.4638 & 0.0908 &
0.10 & 0.5988 & 0.5635 & 0.6931 & 0.2827 \\
0.25 & 0.2854 & 0.3238 & 0.5858 & 0.2485 &
0.15 & 0.8735 & 0.8415 & 0.8896 & 0.7496 \\
0.35 & 0.6235 & 0.6123 & 0.7915 & 0.5912 &
0.20 & 0.9619 & 0.9446 & 0.9596 & 0.9127 \\
0.50 & 0.9112 & 0.8881 & 0.9496 & 0.8973 &
0.25 & 0.9915 & 0.9858 & 0.9896 & 0.9796 \\
0.60 & 0.9750 & 0.9604 & 0.9838 & 0.9719 &
0.60 & 1.0000 & 1.0000 & 1.0000 & 1.0000 \\
\bottomrule
\end{tabular}
\end{table}

% Case two
With $\pi_1=0.5$, Figure~\ref{fig-FFR} shows that the empirical FFR of SEL-BH and ASEL-BH generally decreases as $\delta$ increases, whereas that of CEL-BH and $t$-BH remains at or below the nominal level in most configurations.
Under the exponential design, the power advantage of CEL-BH over $t$-BH is more pronounced. For example, at $\delta=0.10$, both procedures have empirical FFR below $0.05$, with values of $0.0305$ for CEL-BH and $0.0081$ for $t$-BH. However, CEL-BH achieves power of $0.4692$, compared with $0.1323$ for $t$-BH, a gain of $33.69$ percentage points. This result indicates that the conservativeness of $t$-BH comes at a substantial cost in detection power.
The power advantage of ASEL-BH is more apparent at weaker signals. For example, under the normal design with $\delta=0.25$, ASEL-BH achieves power of $0.5769$, compared with $0.1854$ for CEL-BH, a gain of $39.15$ percentage points. Its empirical FFR, however, is $0.1553$, compared with $0.0440$ for CEL-BH, an increase of $11.13$ percentage points. Thus, the substantial power gain of ASEL-BH is accompanied by greater false flagging risk, illustrating the trade-off between screening sensitivity and error control.

\begin{figure}[H]
\centering
\includegraphics[width=0.98\linewidth]{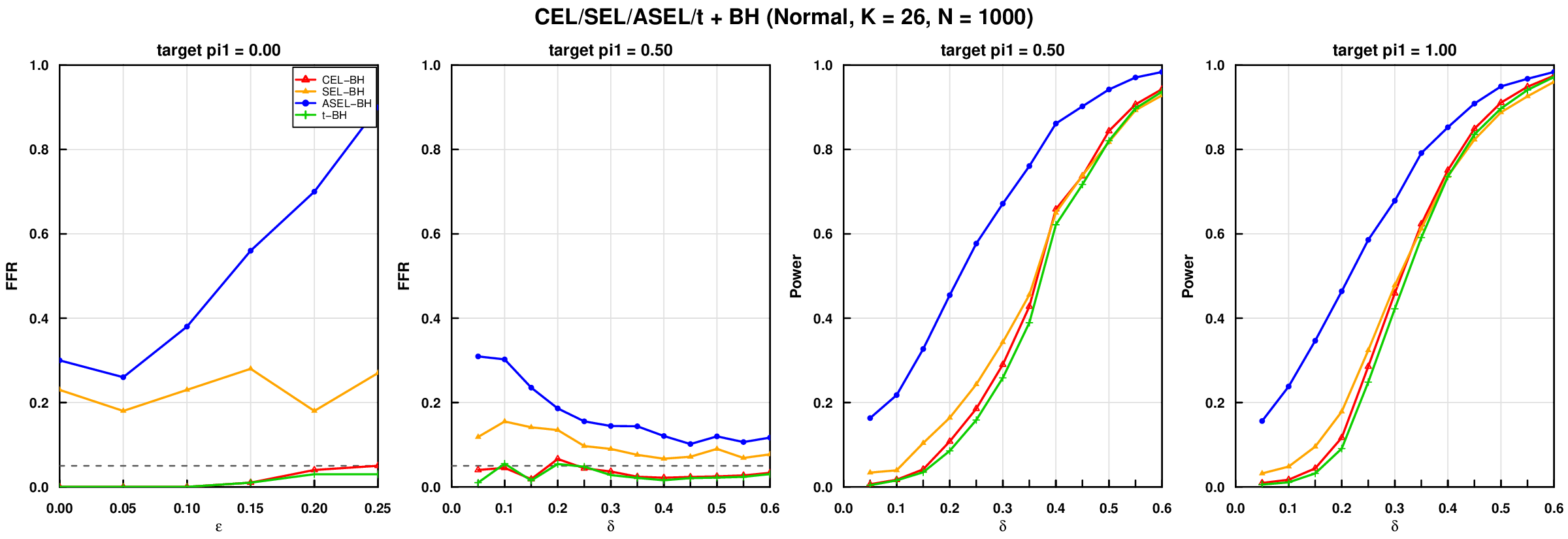}\\
\includegraphics[width=0.98\linewidth]{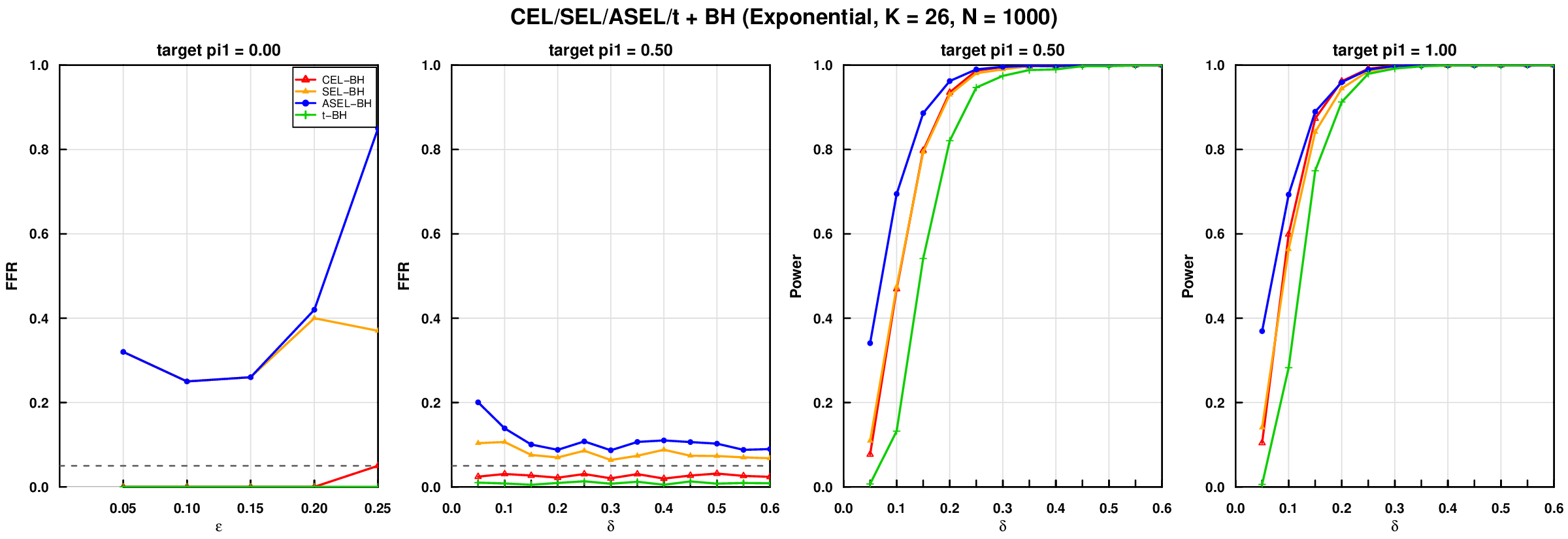}
\caption{Empirical FFR and power under the normal (top) and exponential (bottom) designs. The dashed lines mark the nominal FFR level $0.05$.}
\label{fig-FFR}
\end{figure}

% Case three
When $\pi_1=1$, the FFR is zero by definition because $V_b=0$, so these panels assess detection power alone.
First, Figure~\ref{fig-FFR} shows that  ASEL-BH achieves the highest empirical power across all evaluated effect sizes under the normal design, with its advantage most pronounced at weak to moderate signals. As shown in Table~\ref{tab:power-all-alternative}, its largest power gain over CEL-BH under this design occurs at $\delta=0.20$, reaching $34.73$ percentage points. Under the exponential design, ASEL-BH also exhibits a substantial advantage at weak signals, with a power gain of $26.50$ percentage points over CEL-BH at $\delta=0.05$.
Second, CEL-BH achieves empirical power at least as high as that of $t$-BH across all evaluated configurations. The advantage is particularly pronounced under the exponential design: at $\delta=0.10$, CEL-BH attains power of $0.5988$, compared with $0.2827$ for $t$-BH, a gain of $31.62$ percentage points.
Finally, the power differences among the methods narrow as the signal strengthens, with all four procedures approaching complete detection.

In summary, these finite-sample results illustrate the trade-off between false flags and missed violations. CEL-BH provides stronger protection against false discoveries than SEL-BH and ASEL-BH while achieving higher power than the conservative $t$-BH benchmark, particularly under the exponential design. ASEL-BH offers substantial detection gains at weak to moderate signals, but directly applying BH to the unadjusted SEL and ASEL screening scores yields inflated FFR when true null hypotheses are present.

\section{Empirical application}\label{sec-emp}

In this section, we present an empirical evaluation of the proposed tolerance-based fairness auditing framework using a real-world dataset.  
Our analysis is based on the well-known COMPAS recidivism prediction instrument (RPI), which assigns each defendant a risk score ranging from 1 to 10 that represents the estimated likelihood of reoffending. 
We use the publicly available COMPAS dataset compiled and released by ProPublica, available at \url{https://raw.githubusercontent.com/propublica/compas-analysis/master/compas-scores-two-years.csv}. The dataset has been widely used to study algorithmic bias in criminal justice.
Prior studies have documented systematic disparities, showing that African-American defendants are more likely to receive higher risk scores than their Caucasian counterparts \citep{AngwinEtAl2016, CherianCandes2024}. 
In response to this finding, the creators of COMPAS, Northpointe Inc., argued that the algorithm is fair when evaluated by the \textit{predictive parity} criterion \citep{DieterichEtAl2016,FloresEtAl2016}. 
While they provide statistical evidence for the absence of significant racial bias by this measure, 
our approach builds on this same benchmark to construct explicit bounds on the true performance disparity. 

The predictive parity criterion is satisfied for a single group, $G$, when the positive predictive value (PPV) of $f$ for $G$ matches the PPV for the complement of $G$, i.e.,
\(
 \mathbb{P}(Y = 1 \mid f(X) = 1, X \in G) =  \mathbb{P}(Y = 1 \mid f(X) = 1, X \in G^c).
\)
We define the PPV disparity for group $G$ as
\(
\epsilon(G) := \mathbb{P}(Y = 1 \mid f(X) = 1, X \in G) - \mathbb{P}(Y = 1 \mid f(X) = 1, X \in G^c),
\)
which represents the difference in PPV between the two groups.
Intuitively, the PPV measures how informative a positive prediction is. For example, if COMPAS classifies a defendant as high-risk ($f(X) = 1$), the PPV corresponds to the probability that they actually recidivate ($Y = 1$). 
We employ the proposed auditing procedure to quantify such disparity and derive tolerance-based bounds that provide rigorous statistical guarantees for fairness auditing.

Following prior work, we use the COMPAS dataset and binarize the risk scores by defining $f(X) = 1$ if the RPI score is greater than or equal to 5, and $f(X) = 0$ otherwise. We restrict our analysis to the subset ($n = 2174$) of the holdout data with $f(X) = 1$ and $X_{\mathrm{race}} \in \{\text{African-American}, \text{Caucasian}\}$. 
We instantiate our audit by taking the loss to be $M(f(X),Y)=\mathbbm{1}_{\{f(X)=Y\}}=Y$, the audited group $G$ to consist of African-American defendants with $f(X)=1$, while the reference quantity $\theta_P$ denotes the PPV among Caucasian defendants:
\(
M(f(X), Y) = Y
\), 
\(
G = \{\{(X, Y) \mid X_{\mathrm{race}} = \text{African-American}, f(X) = 1\}\},
\)
and
\(
\theta_P = \mathbb{E}[Y = 1 \mid X_{\mathrm{race}} = \text{Caucasian}].
\)
The parameter $\theta_P$ is estimated by its empirical conditional expectation, denoted as
\(
\hat{\theta} = \hat{\mathbb{E}}_n [Y = 1 \mid X_{\text{race}} = \text{Caucasian}].
\)
In this application, all empirical-likelihood procedures use the plug-in estimator $\hat{\theta}$ in place of $\theta_P$. 

First, to evaluate the claim in \citet{DieterichEtAl2016} that African-American defendants incur no PPV deficit relative to Caucasian defendants, we construct empirical-likelihood inversion regions for $\epsilon(G)$ via testing
\[
H_0(G): 0 \le \epsilon(G) \le  \epsilon_0 \quad \text{versus} \quad H_1(G): \epsilon(G) > \epsilon_0. 
\]
Applying the proposed tolerance-based auditing framework, we construct 90\% and 95\% empirical-likelihood inversion regions for $\epsilon(G)$ in the overall African-American group. Table~\ref{tab:aa-all-lb} reports the corresponding lower bounds under EL, CEL, SEL, and ASEL at these nominal levels.

\begin{table}[!htbp]
\centering
\caption{Lower bounds for the African-American group at nominal levels 90\% and 95\%.}
\label{tab:aa-all-lb}
\begin{tabular}{lrrrr}
\toprule
Nominal level & EL & CEL & SEL & ASEL \\
\midrule
90\% & 0.0212 & 0.0251 & 0.0209 & 0.0235 \\
95\% & 0.0180 & 0.0212 & 0.0175 & 0.0235 \\
\bottomrule
\end{tabular}
\end{table}
\noindent

Table~\ref{tab:aa-all-lb} reports one-sided lower bounds for the PPV disparity $\epsilon(G)$ obtained by EL, CEL, SEL, and ASEL.
At the 90\% nominal level, CEL yields a formal lower bound for violation certification of 2.51\%. 
SEL and ASEL yield screening lower bounds of 2.09\% and 2.35\%, respectively.
These bounds lead to distinct auditing conclusions depending on the prespecified tolerance.
For $\epsilon_0=0.01$, at both the 90\% and 95\% nominal levels, CEL formally certifies a tolerance violation, while SEL and ASEL also issue screening alerts.
For $\epsilon_0=0.03$, at either nominal level, CEL does not certify a tolerance violation
 and neither SEL nor ASEL issues a screening alert, since their lower bounds also fall below the
3-percentage-point threshold.
These results also contextualize earlier findings on COMPAS.
The positive CEL lower bounds at both nominal levels support the narrower claim of \citet{DieterichEtAl2016} that the PPV among African-American defendants is no lower than that among Caucasian defendants.
Relative to the bootstrap-based audit of \citet{CherianCandes2024}, which reports a certified 90\% lower bound of 1.87\%, CEL yields a larger lower bound of 2.51\% at the same nominal level, providing a stronger numerical lower-bound certificate for this dataset. 
Thus, the COMPAS application illustrates how the proposed framework translates the same observed predictive disparity into distinct auditing conclusions under different prespecified tolerances.

% Figure 5: African-American disparity, 2 rows $\times$ 4 columns (90% / 95%)
\begin{figure}[!htbp]
\centering
\begin{subfigure}[b]{0.24\textwidth}
\centering
\includegraphics[width=\linewidth]{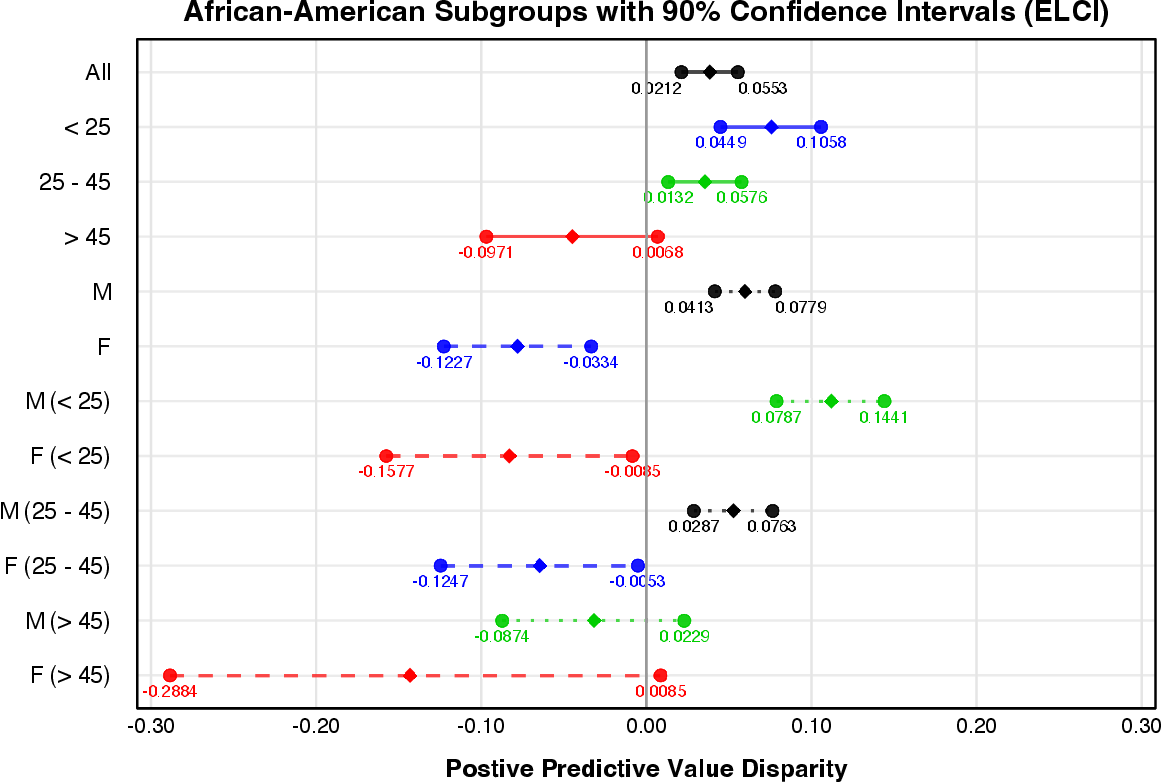}
\subcaption{EL}
\label{fig-PPV-African-American-0.9-a}
\end{subfigure}\hfill
\begin{subfigure}[b]{0.24\textwidth}
\centering
\includegraphics[width=\linewidth]{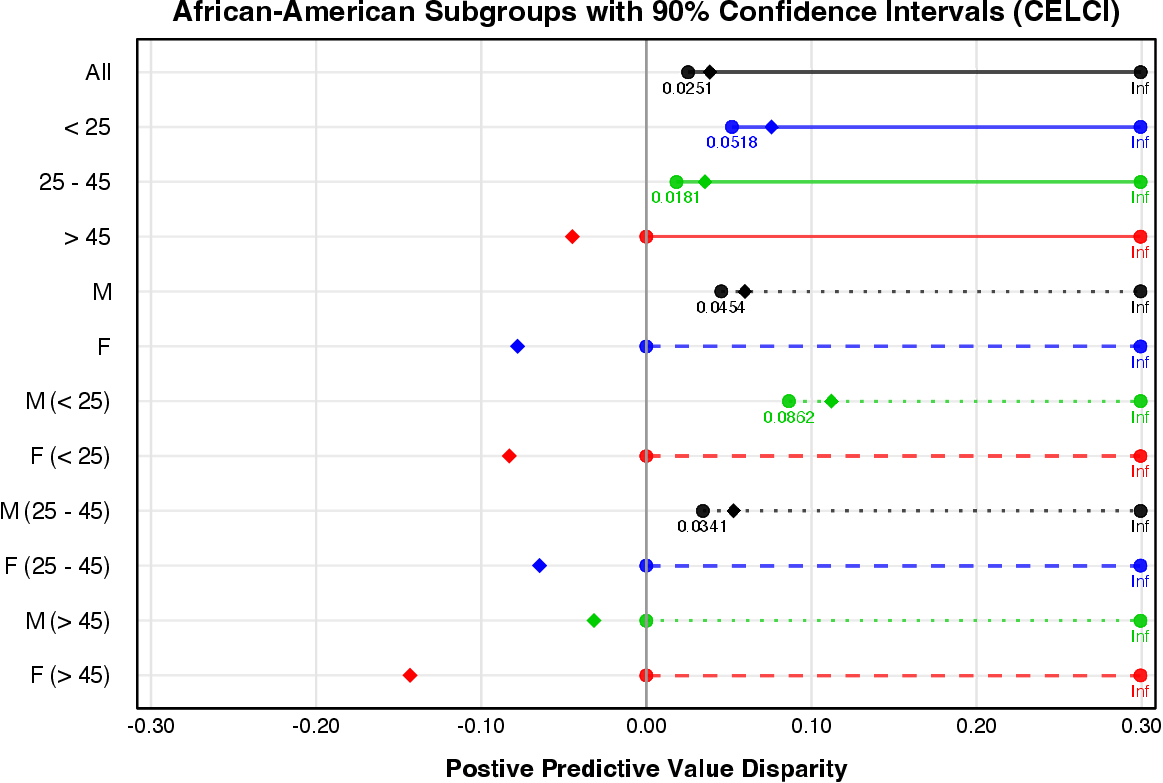}
\subcaption{CEL}
\label{fig-PPV-African-American-0.9-b}
\end{subfigure}\hfill
\begin{subfigure}[b]{0.24\textwidth}
\centering
\includegraphics[width=\linewidth]{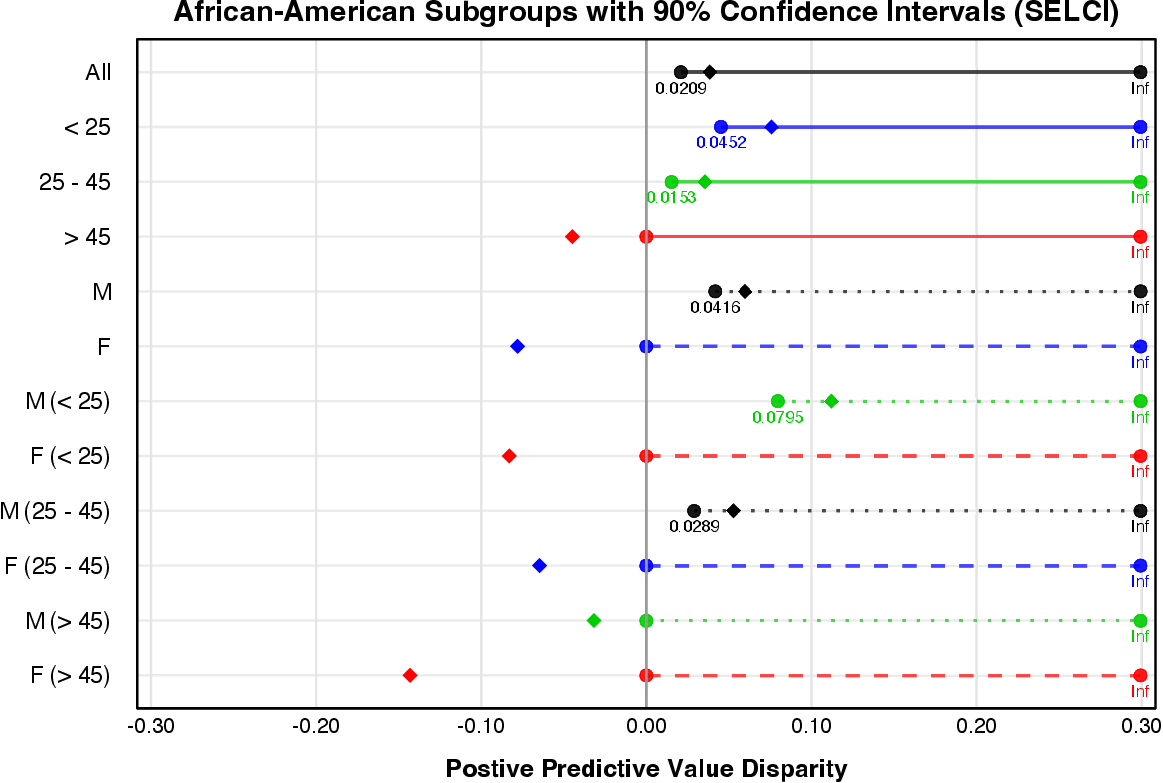}
\subcaption{SEL}
\label{fig-PPV-African-American-0.9-c}
\end{subfigure}\hfill
\begin{subfigure}[b]{0.24\textwidth}
\centering
\includegraphics[width=\linewidth]{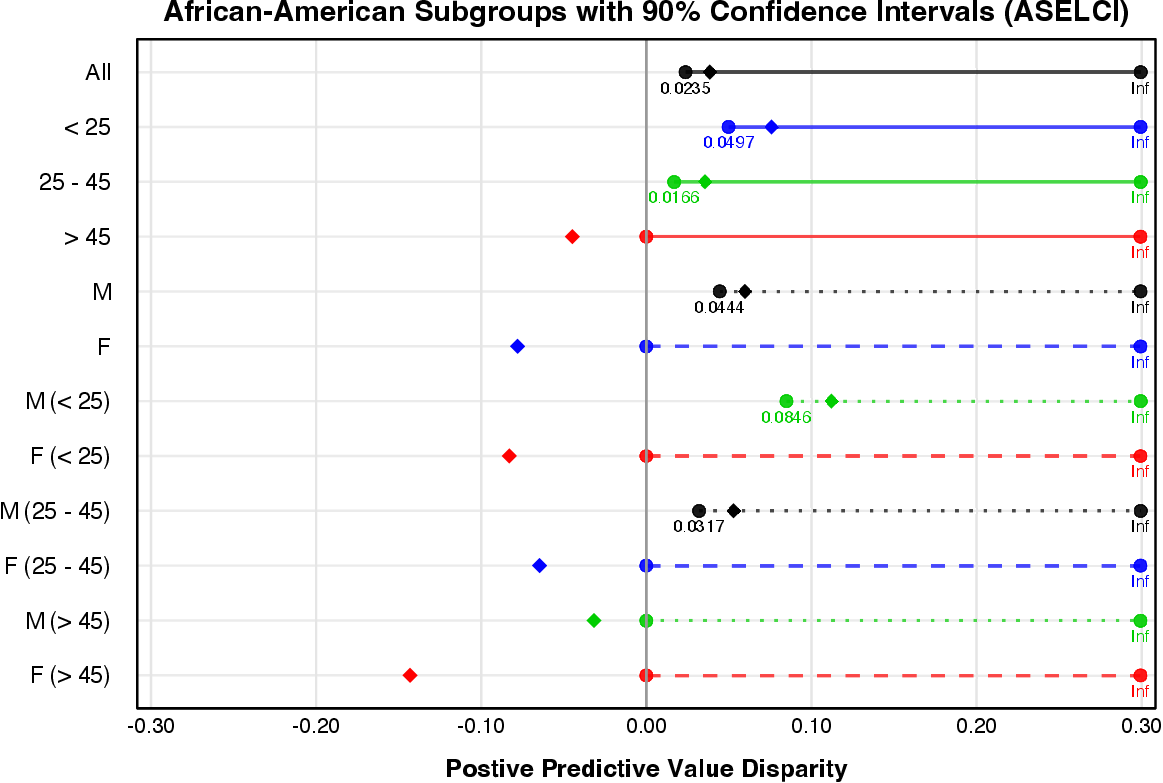}
\subcaption{ASEL}
\label{fig-PPV-African-American-0.9-d}
\end{subfigure}\\[0.6em]
\begin{subfigure}[b]{0.24\textwidth}
\centering
\includegraphics[width=\linewidth]{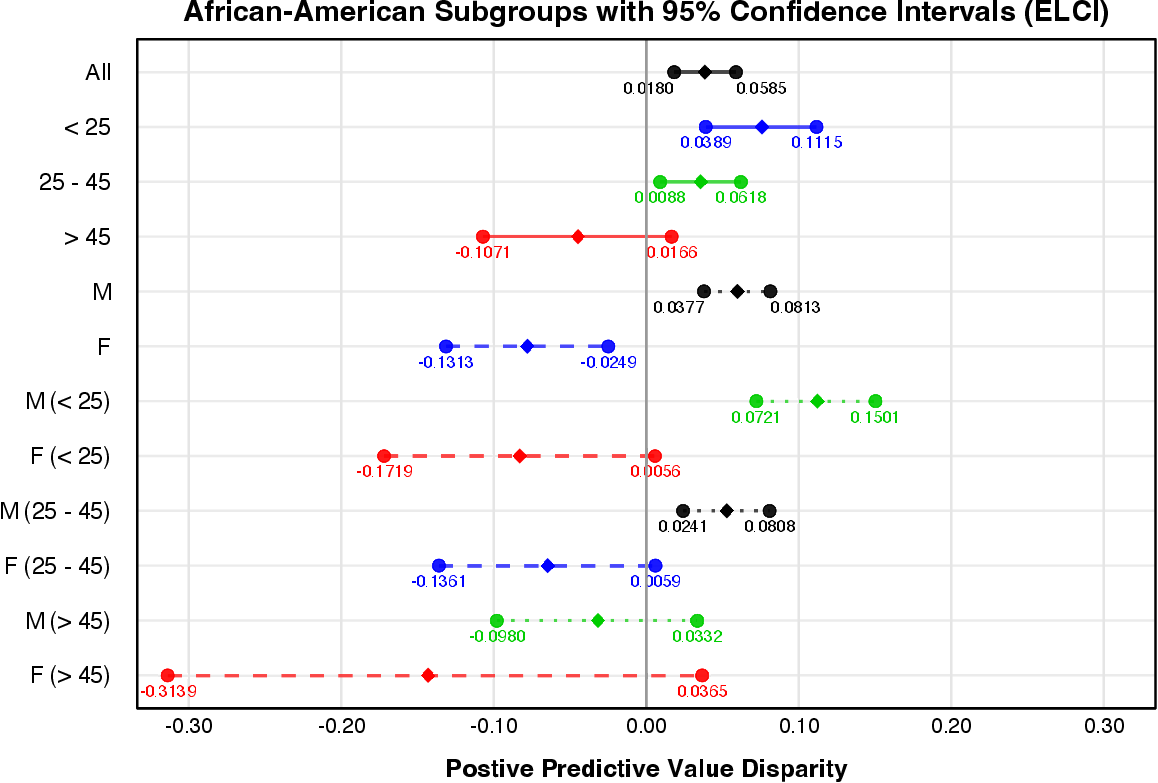}
\subcaption{EL}
\label{fig-PPV-African-American-0.95-a}
\end{subfigure}\hfill
\begin{subfigure}[b]{0.24\textwidth}
\centering
\includegraphics[width=\linewidth]{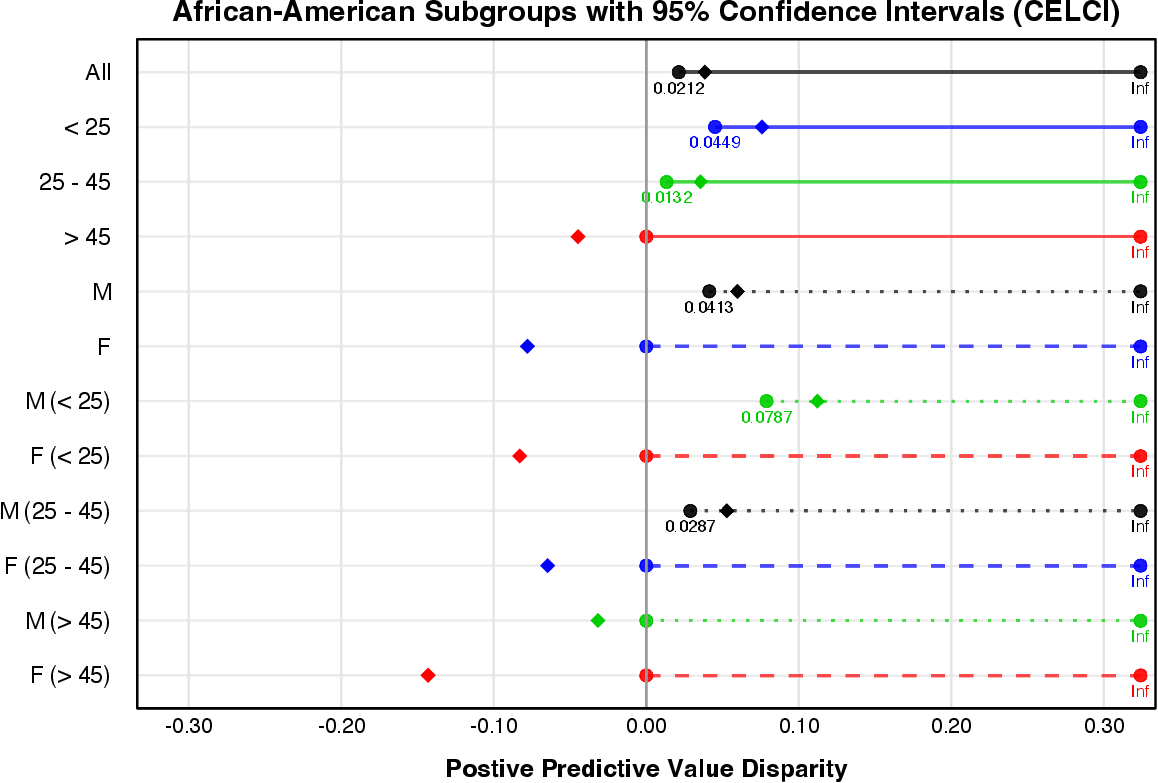}
\subcaption{CEL}
\label{fig-PPV-African-American-0.95-b}
\end{subfigure}\hfill
\begin{subfigure}[b]{0.24\textwidth}
\centering
\includegraphics[width=\linewidth]{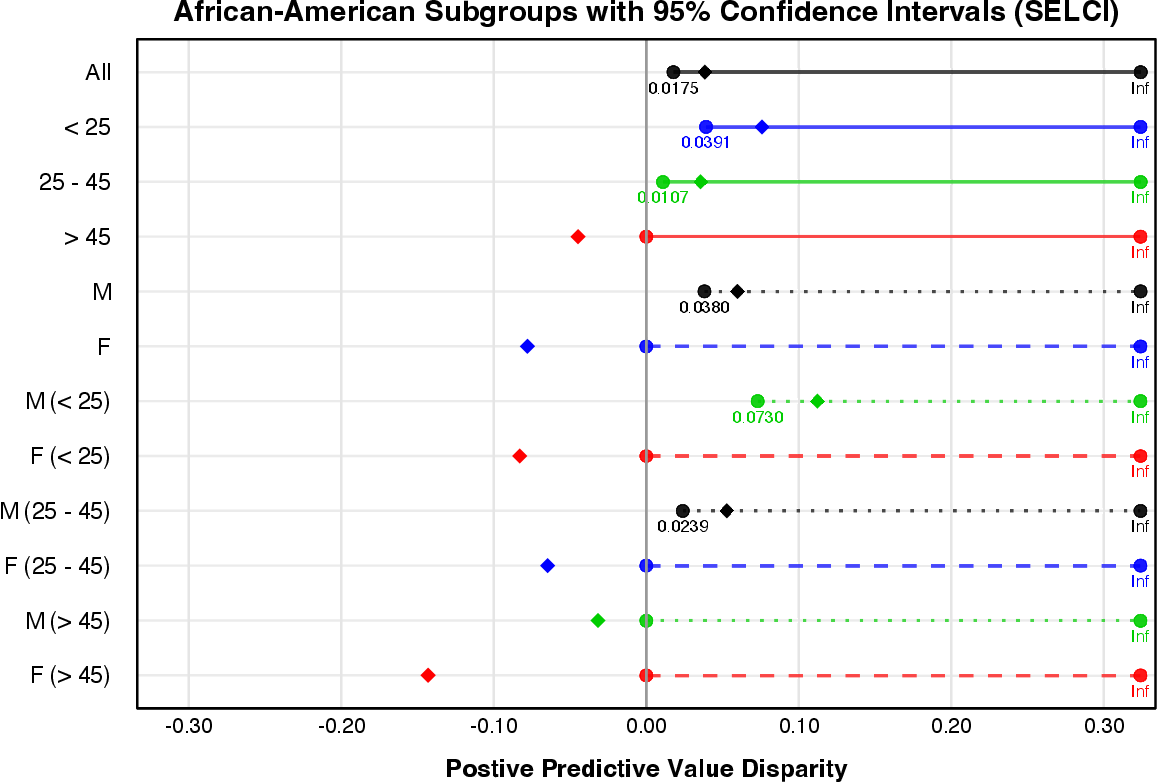}
\subcaption{SEL}
\label{fig-PPV-African-American-0.95-c}
\end{subfigure}\hfill
\begin{subfigure}[b]{0.24\textwidth}
\centering
\includegraphics[width=\linewidth]{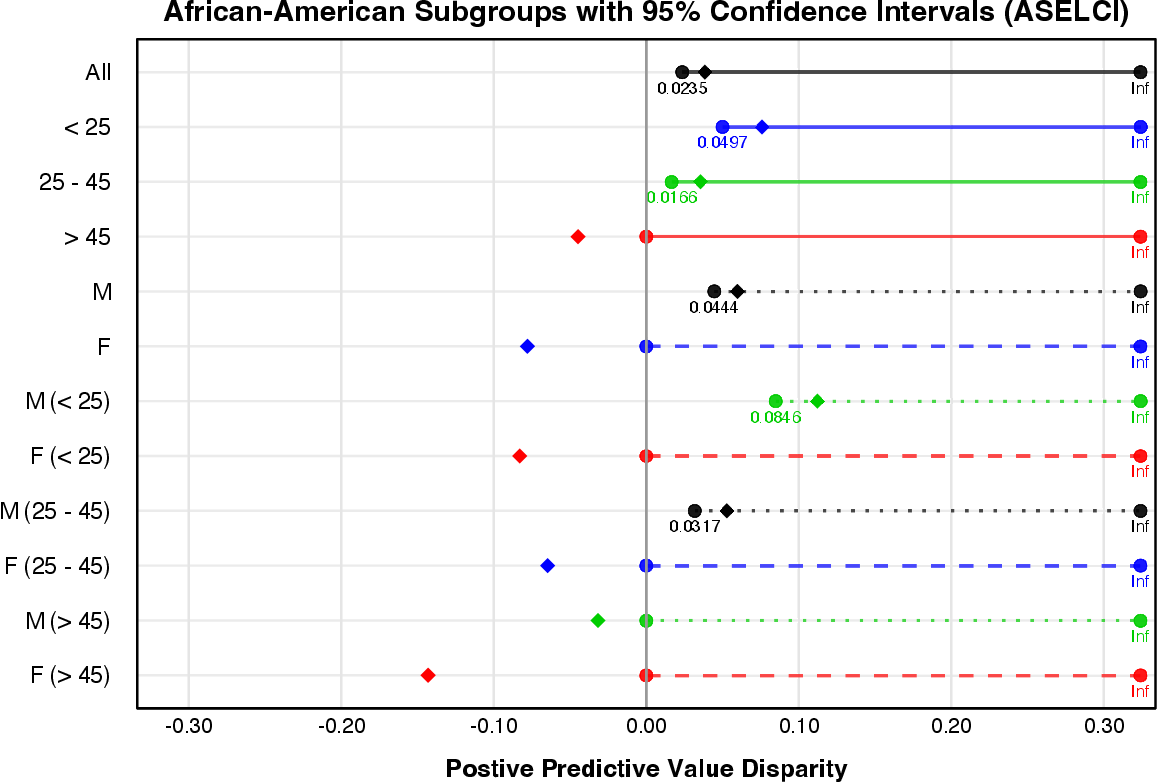}
\subcaption{ASEL}
\label{fig-PPV-African-American-0.95-d}
\end{subfigure}
\caption{
Pointwise inversion regions for the COMPAS PPV disparity of each African-American subgroup relative to the entire Caucasian reference group. 
Top row: 90\% nominal level and bottom row: 95\% nominal level.
}
\label{fig-PPV_African-American_Disparity}
\end{figure}

Second, again using the PPV for Caucasian defendants as the reference, we construct pointwise 90\% and 95\% inversion regions for the PPV disparity in each intersectional subgroup of African-American defendants defined by sex and age. 
Results are displayed in Figure~\ref{fig-PPV_African-American_Disparity}. 
Relative to the 95\% regions, those at the 90\% nominal level are narrower and therefore more sensitive for detecting nonzero disparities; the 95\% regions are wider and correspondingly more conservative.
At both nominal levels, the lower bounds from all four procedures are strictly positive for \textit{All}, $\!<25$, $25$--$45$, \textit{M}, \textit{M}($<25$), and \textit{M}($25$--$45$), supporting positive PPV disparities relative to the Caucasian reference. In particular, the signal is strongest for \textit{M}($<25$), whose lower bound at the 90\% nominal level reaches at least 7.95\%.
For the remaining subgroups, the one-sided lower bounds obtained from CEL, SEL, and ASEL do not uniformly support the existence of positive disparities, as several lower bounds are equal to zero.
Overall, Figure~\ref{fig-PPV_African-American_Disparity} examines Northpointe Inc.'s claims of PPV parity for African-American subpopulations and finds that younger male subgroups exhibit higher COMPAS PPV relative to the Caucasian reference.

% Figure 6: Sex--age disparity, 2 rows $\times$ 4 columns (90% / 95%)
\begin{figure}[!htbp]
\centering
\begin{subfigure}[b]{0.24\textwidth}
\centering
\includegraphics[width=\linewidth]{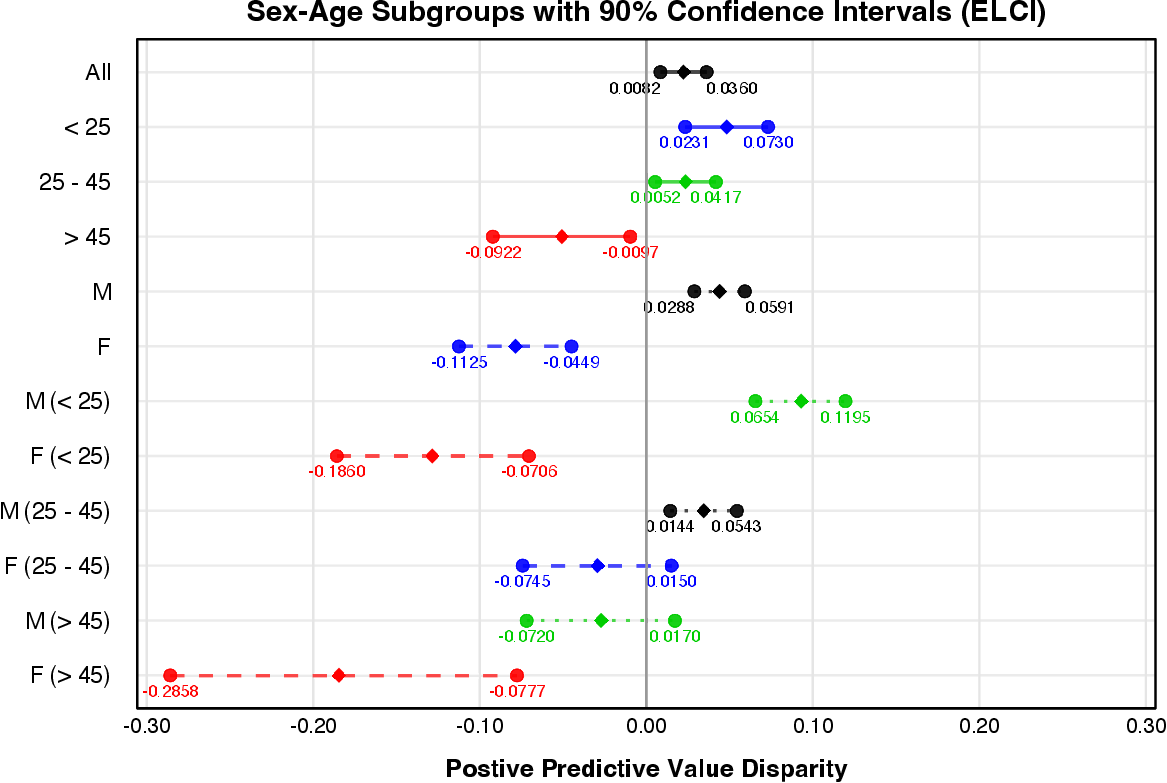}
\subcaption{EL}
\label{fig-PPV-Sex-Age-0.9-a}
\end{subfigure}\hfill
\begin{subfigure}[b]{0.24\textwidth}
\centering
\includegraphics[width=\linewidth]{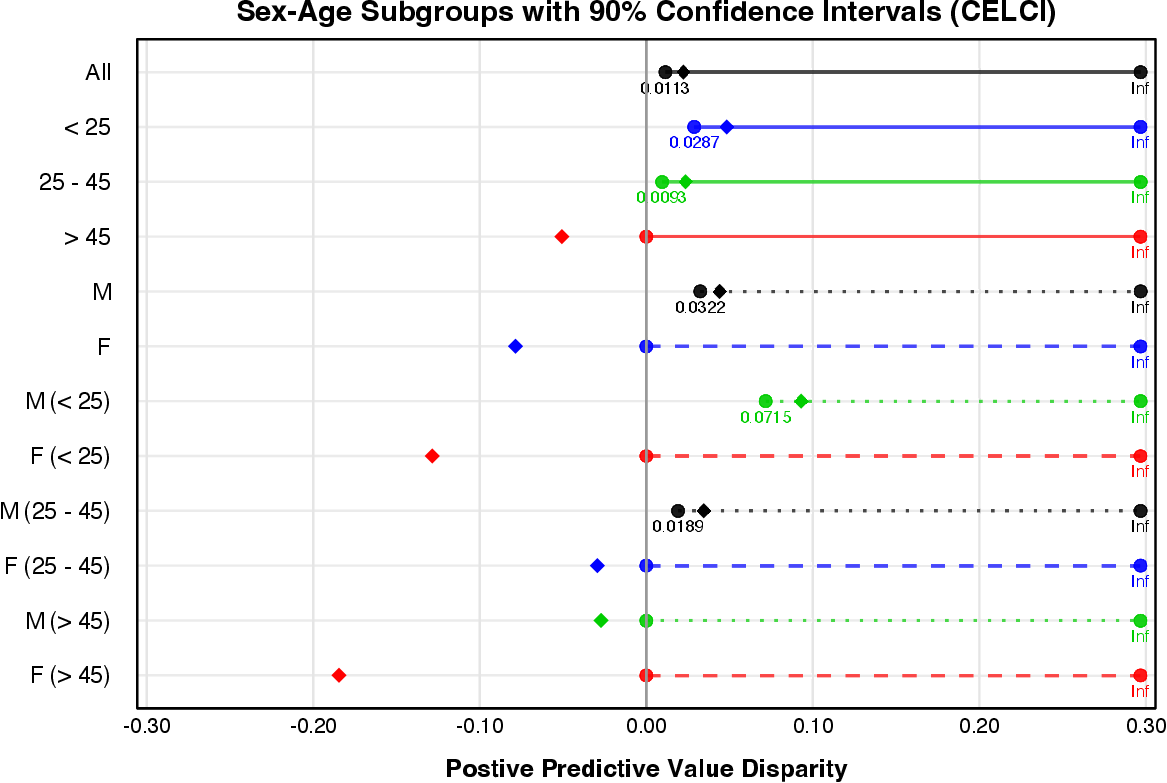}
\subcaption{CEL}
\label{fig-PPV-Sex-Age-0.9-b}
\end{subfigure}\hfill
\begin{subfigure}[b]{0.24\textwidth}
\centering
\includegraphics[width=\linewidth]{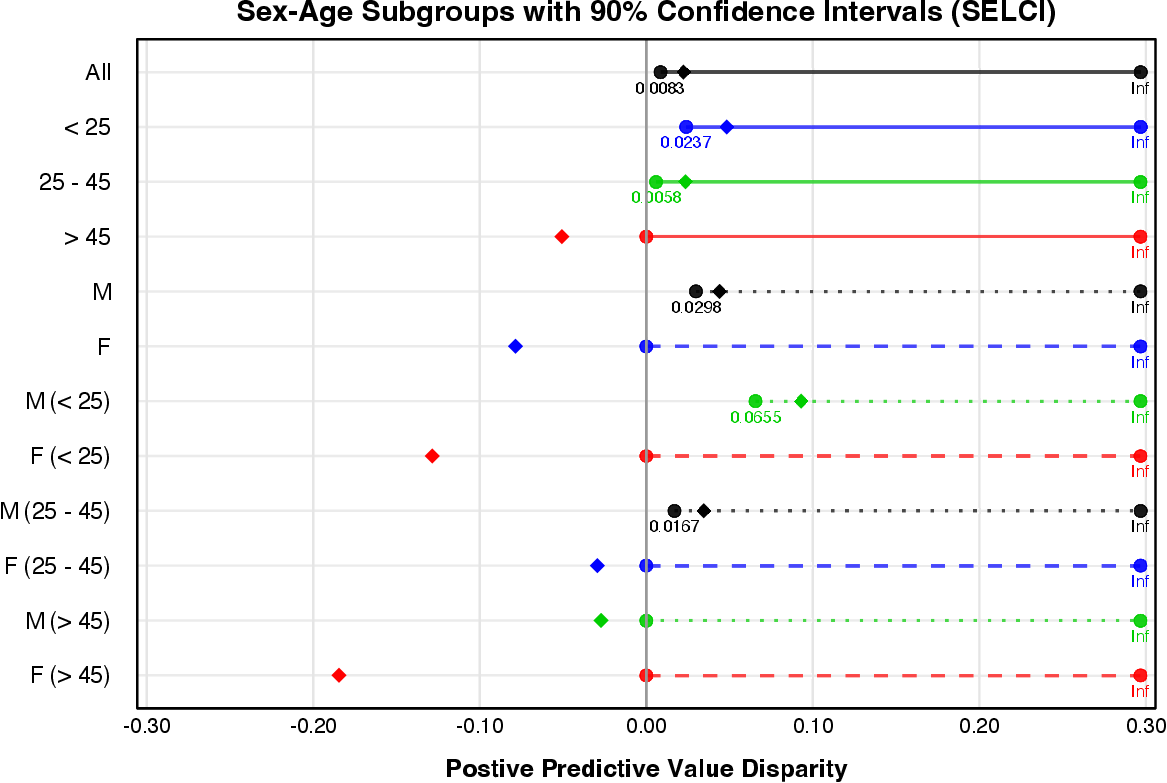}
\subcaption{SEL}
\label{fig-PPV-Sex-Age-0.9-c}
\end{subfigure}\hfill
\begin{subfigure}[b]{0.24\textwidth}
\centering
\includegraphics[width=\linewidth]{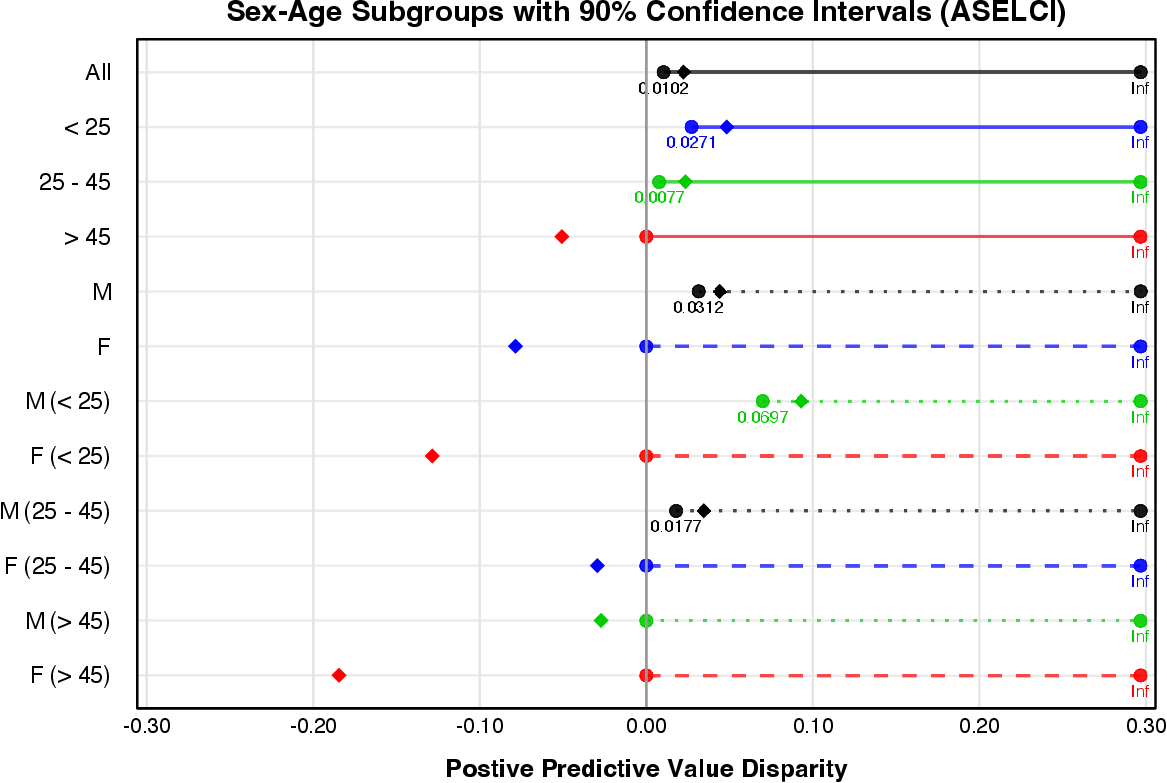}
\subcaption{ASEL}
\label{fig-PPV-Sex-Age-0.9-d}
\end{subfigure}\\[0.6em]
\begin{subfigure}[b]{0.24\textwidth}
\centering
\includegraphics[width=\linewidth]{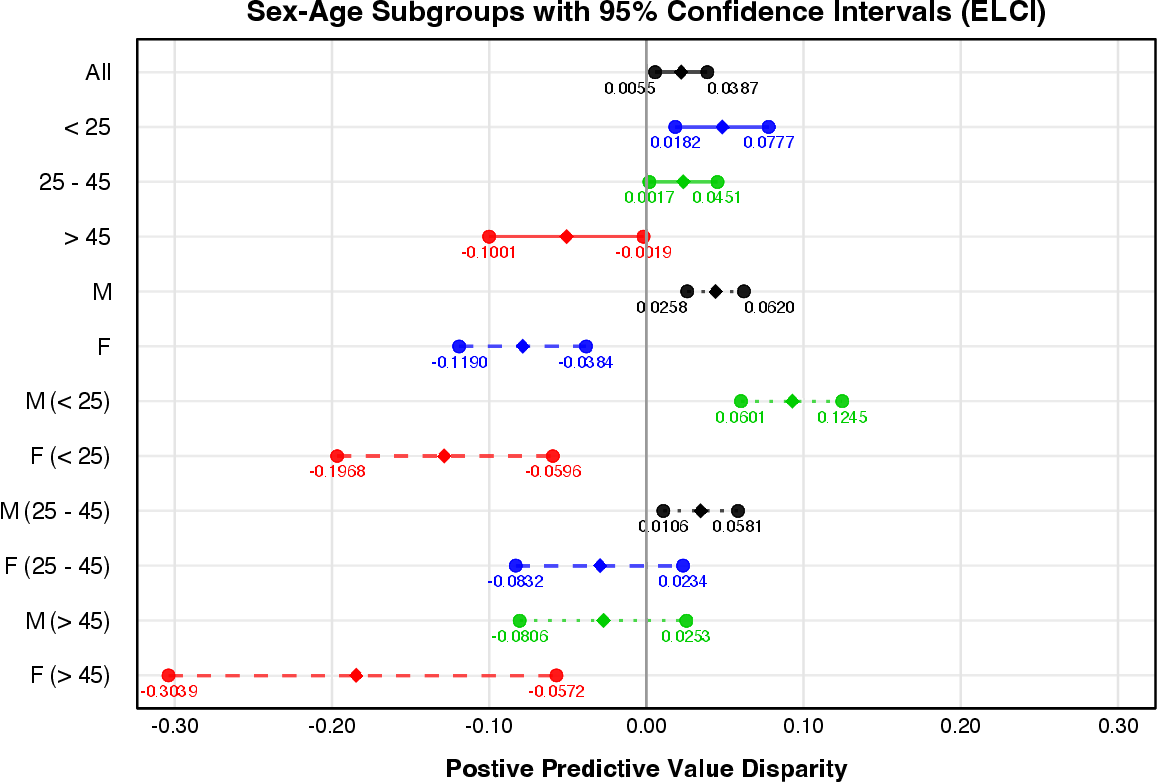}
\subcaption{EL}
\label{fig-PPV-Sex-Age-0.95-a}
\end{subfigure}\hfill
\begin{subfigure}[b]{0.24\textwidth}
\centering
\includegraphics[width=\linewidth]{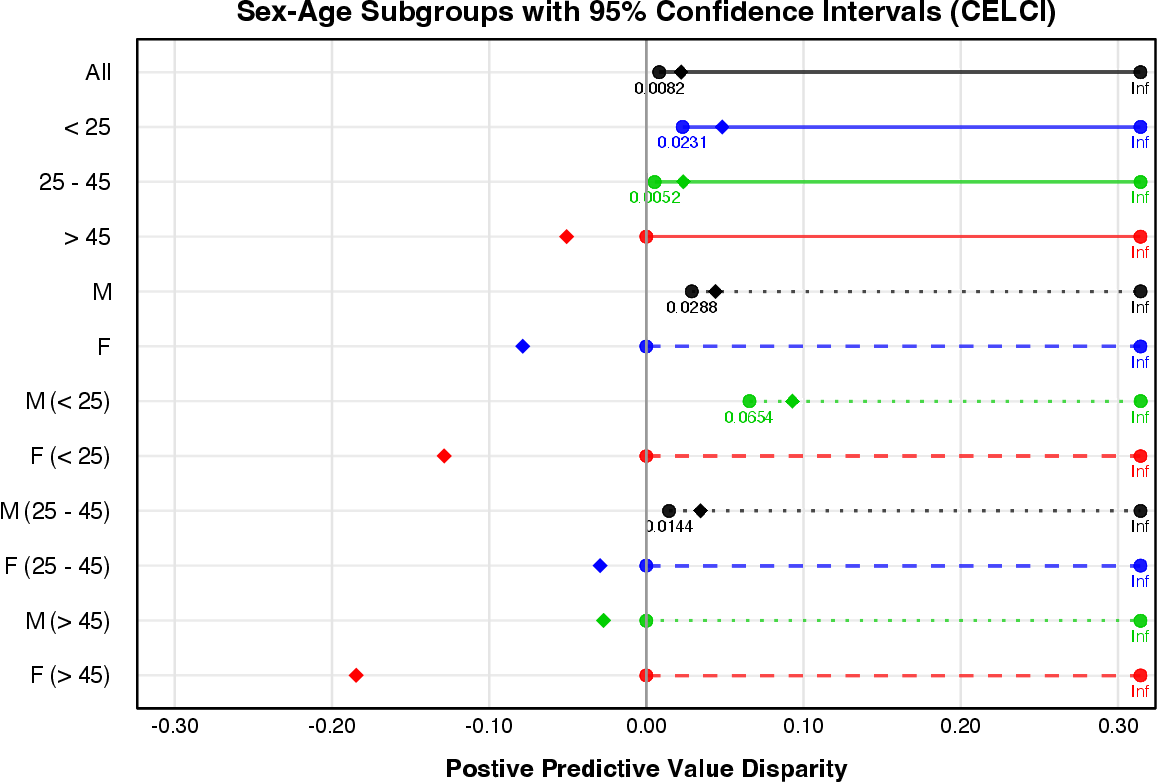}
\subcaption{CEL}
\label{fig-PPV-Sex-Age-0.95-b}
\end{subfigure}\hfill
\begin{subfigure}[b]{0.24\textwidth}
\centering
\includegraphics[width=\linewidth]{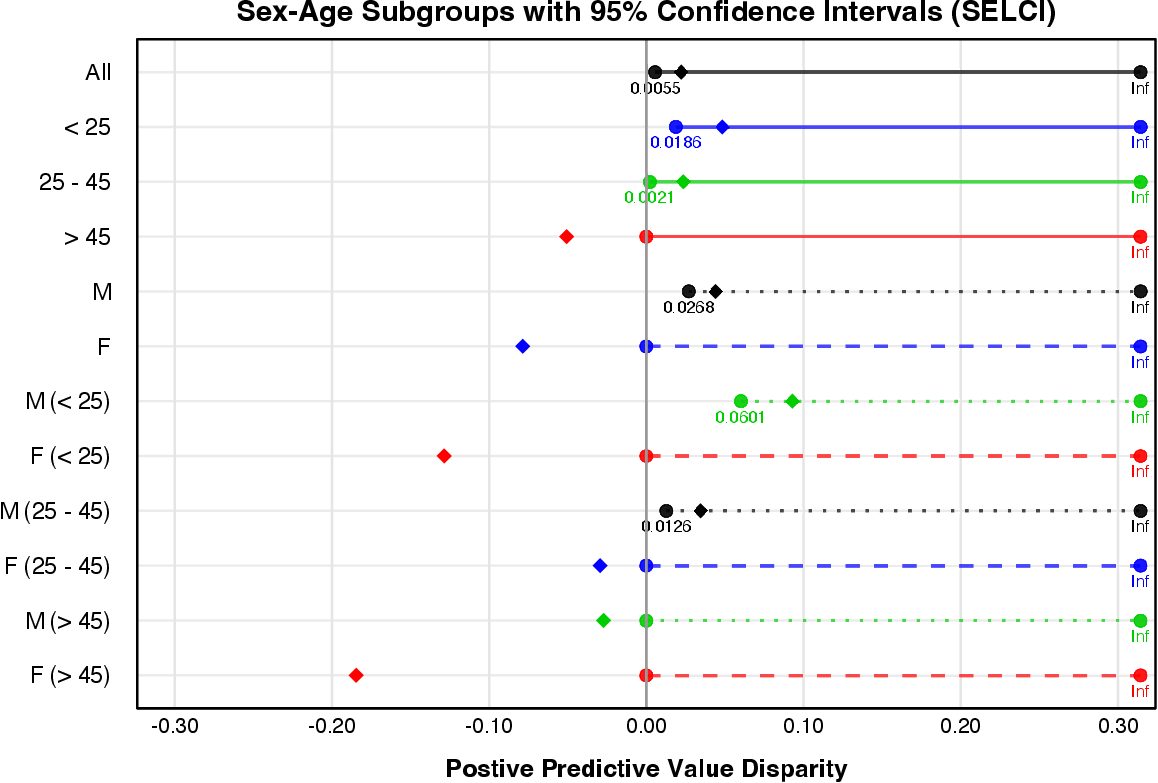}
\subcaption{SEL}
\label{fig-PPV-Sex-Age-0.95-c}
\end{subfigure}\hfill
\begin{subfigure}[b]{0.24\textwidth}
\centering
\includegraphics[width=\linewidth]{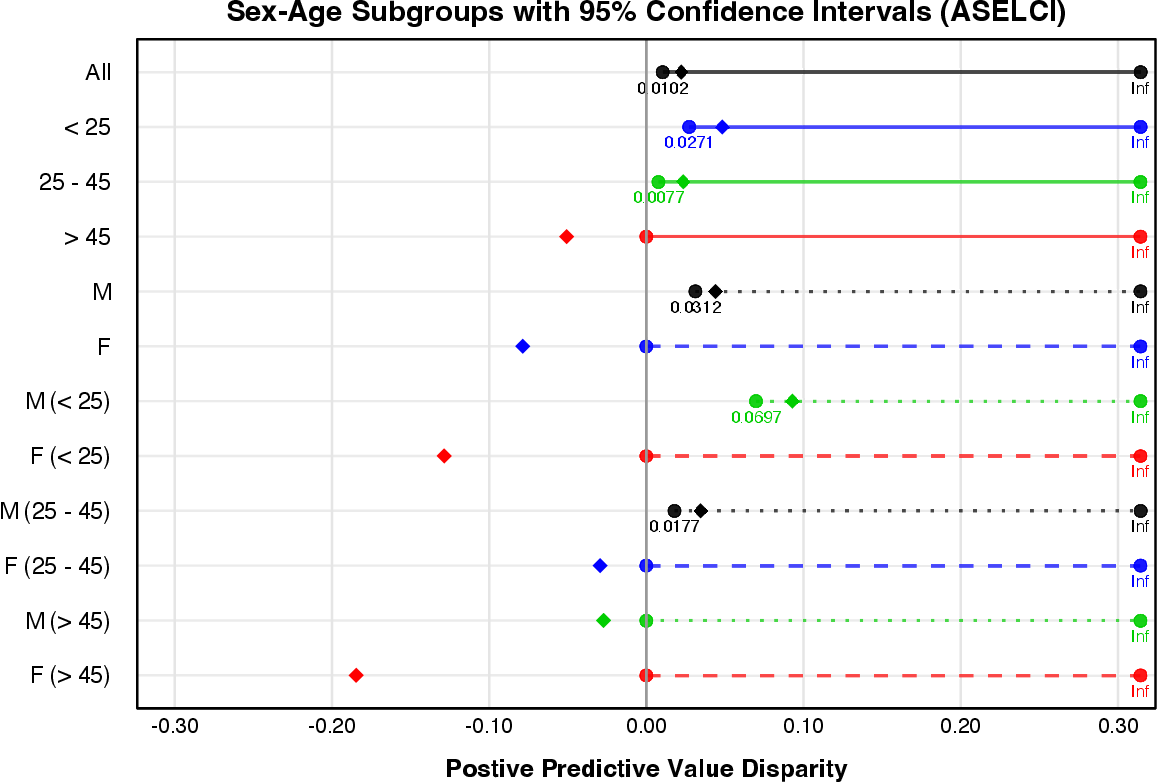}
\subcaption{ASEL}
\label{fig-PPV-Sex-Age-0.95-d}
\end{subfigure}
\caption{
Pointwise inversion regions for the COMPAS PPV disparity of each subgroup formed by intersections of age and sex relative to the entire Caucasian reference group.
Top row: 90\% nominal level and bottom row: 95\% nominal level.
}
\label{fig-PPV_Sex-Age_Disparity}
\end{figure}

Third, taking the PPV for Caucasian defendants as the reference, we construct pointwise 90\% and 95\% inversion regions for the PPV disparity in each subpopulation defined solely by the intersection of sex and age, based on $n=3317$ observations.
As shown in Figure~\ref{fig-PPV_Sex-Age_Disparity}, the largest positive disparities relative to this benchmark occur for \textit{All}, $\!<25$, $25$--$45$, \textit{M}, \textit{M}($<25$), and \textit{M}($25$--$45$). In particular, the signal is strongest for \textit{M}($<25$), whose lower bound at the 90\% nominal level reaches at least 6.55\%. 
In summary, conditional on a positive COMPAS prediction, the PPV for Caucasian defendants is lower than the overall PPV pooled across all racial categories and younger male subgroups exhibit higher COMPAS PPV relative to the Caucasian reference.

Finally, we use the same procedure to certify whether any demographic groups exhibit negative PPV disparities in Figure~\ref{fig-PPV-Race-Sex-Age-Disparity}. We conduct the hypothesis test:
\[
H_0(G_j): -\epsilon_0 \le \epsilon(G) \le 0 \quad \text{vs} \quad H_1(G_j): \epsilon(G) < -\epsilon_0, \quad j = 1, \dots, |\mathcal{G}|
\]
using the overall average PPV as the reference, i.e., $\hat{\theta} = \hat{\mathbb{E}}_n [Y = 1]$. 
This formulation is used to certify whether a group's PPV falls below the overall reference by more than the prespecified tolerance.
If \(U_G<-\epsilon_0\), the data provide sufficient evidence that the negative disparity exceeds the tolerance threshold, and the audited subpopulation is therefore flagged as exhibiting a PPV deficit.
If \(-\epsilon_0<U_G <0\), the data support a negative PPV disparity but do not provide sufficient evidence that its magnitude exceeds the specified tolerance. 
If \(U_G=0\), the data do not certify a negative disparity, and assessing a disparity in the opposite direction requires a separate violation-certification procedure.

Using negativity of the upper bound as evidence of a negative PPV disparity, Figure~\ref{fig-PPV-Race-Sex-Age-Disparity} shows negative disparities across many female intersectional subgroups, with the largest deficits concentrated among age--race intersections within Caucasian females.
The subgroup of Caucasian females under age~25 yields the strongest negative disparity signal: at the 90\% nominal level, all four procedures produce upper bounds below $-10\%$, while the formal CEL upper bound for violation certification is $-15.42\%$, certifying that its PPV is at least 15.42 percentage points lower than the overall average.
For the subgroup $>45$, \textit{Caucasian} ($n=27$), CEL reports an upper bound of $-12.06\%$, whereas SEL and ASEL return empty screening inversion regions, indicating rejection under the corresponding screening rule, thereby generating screening alerts for a negative PPV disparity relative to the overall average.
%so the data nevertheless support a PPV deficit relative to the overall average. 
In the ASEL panels, an upper bound equal to zero indicates non-rejection of $H_0$ for \textit{Caucasian} ($25$--$45$) and \textit{African-American}($>45$) at the 90\% nominal level, with the aggregate $25$--$45$ group added at the 95\% level.
Accordingly, ASEL does not provide screening evidence sufficient to identify these three subgroups as exhibiting a negative PPV disparity at the corresponding nominal level.
Overall, Figure~\ref{fig-PPV-Race-Sex-Age-Disparity} indicates that, conditional on a positive COMPAS prediction, PPV tends to be lower across most female intersection subgroups, with the strongest deficits among females under age~25 and Caucasian females over age~45.

% Figure 7: Female--Race--age disparity, 2 rows $\times$ 4 columns (90% / 95%)
\begin{figure}[!htbp]
\centering
\begin{subfigure}[b]{0.24\textwidth}
\centering
\includegraphics[width=\linewidth]{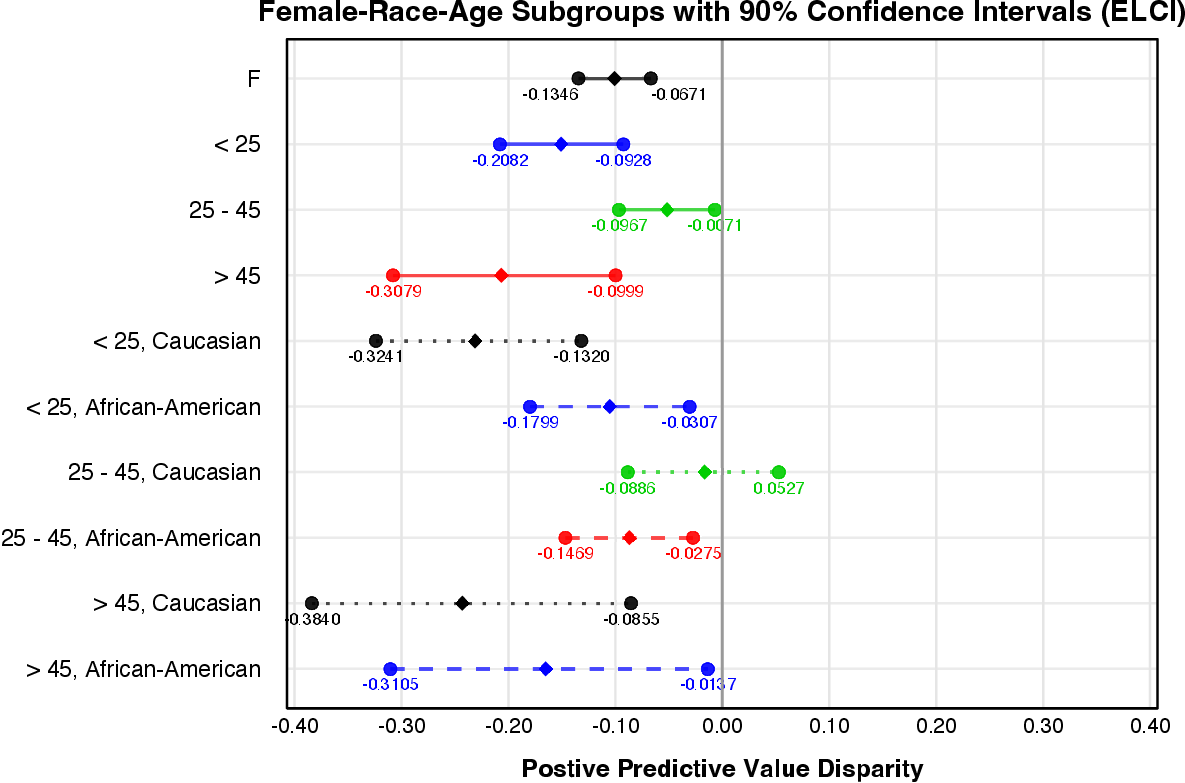}
\subcaption{EL}
\label{fig-PPV-Female-Race-Age-0.9-a}
\end{subfigure}\hfill
\begin{subfigure}[b]{0.24\textwidth}
\centering
\includegraphics[width=\linewidth]{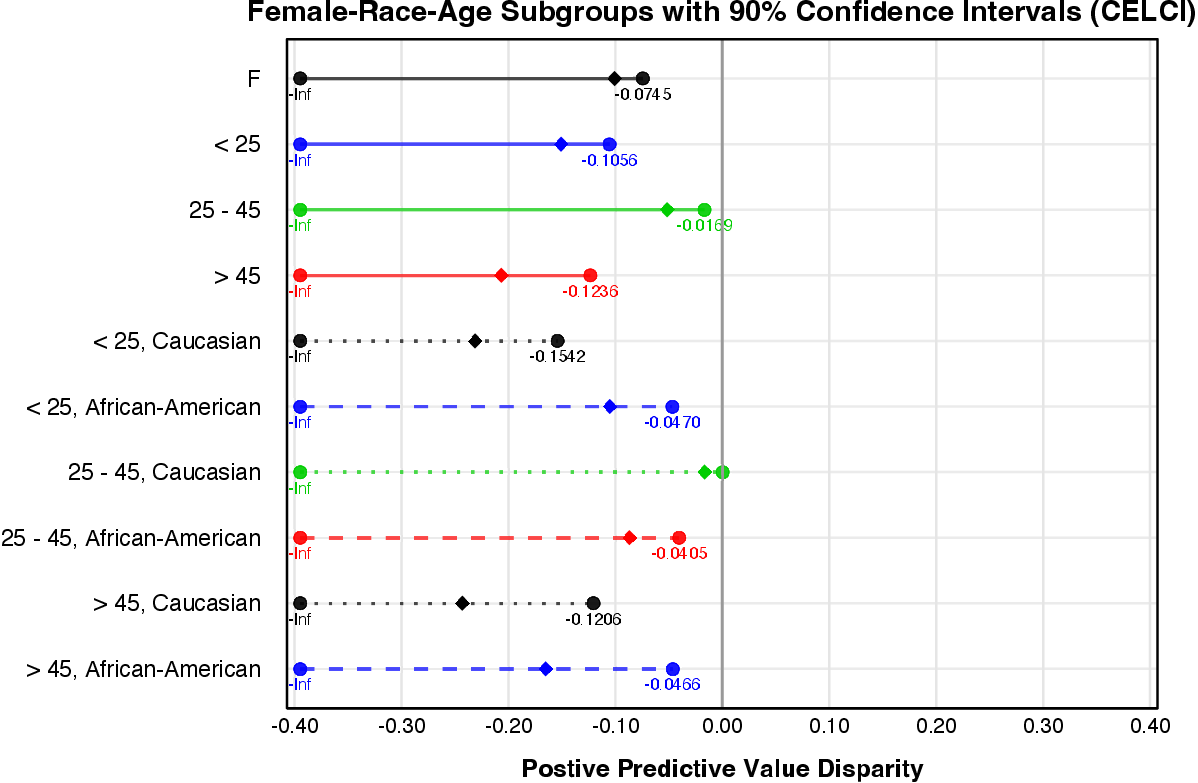}
\subcaption{CEL}
\label{fig-PPV-Female-Race-Age-0.9-b}
\end{subfigure}\hfill
\begin{subfigure}[b]{0.24\textwidth}
\centering
\includegraphics[width=\linewidth]{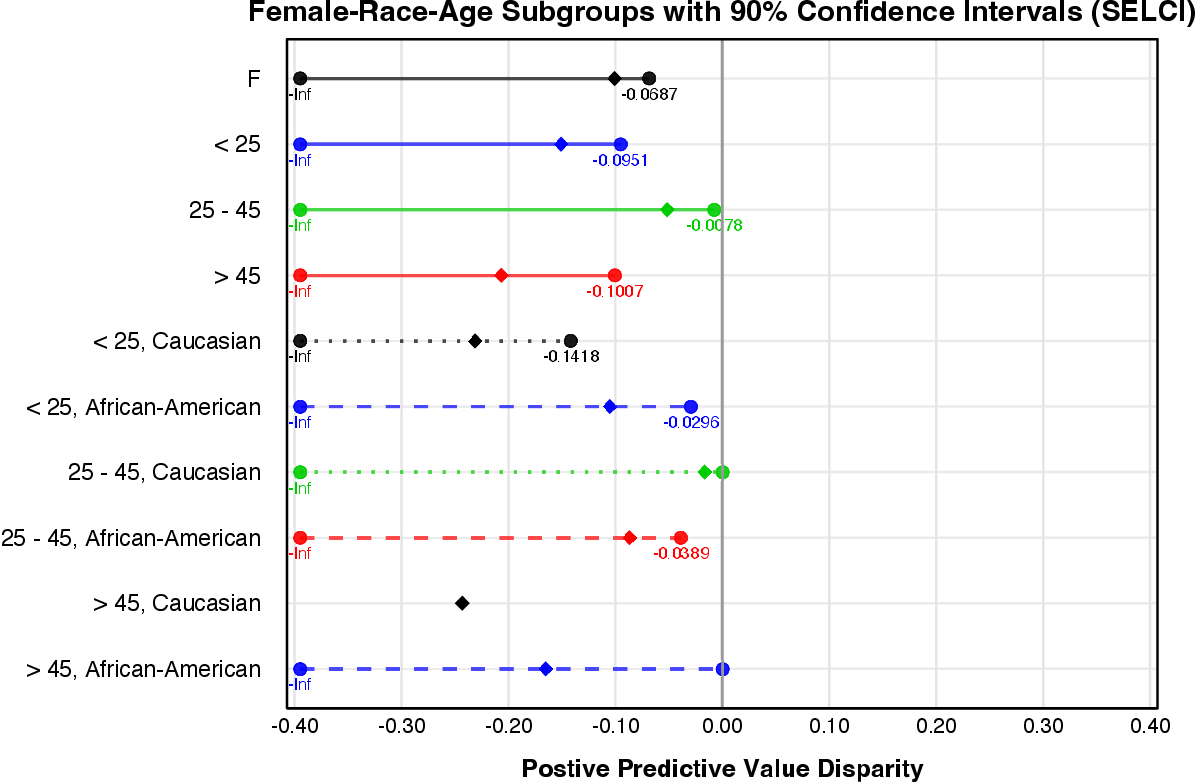}
\subcaption{SEL}
\label{fig-PPV-Female-Race-Age-0.9-c}
\end{subfigure}\hfill
\begin{subfigure}[b]{0.24\textwidth}
\centering
\includegraphics[width=\linewidth]{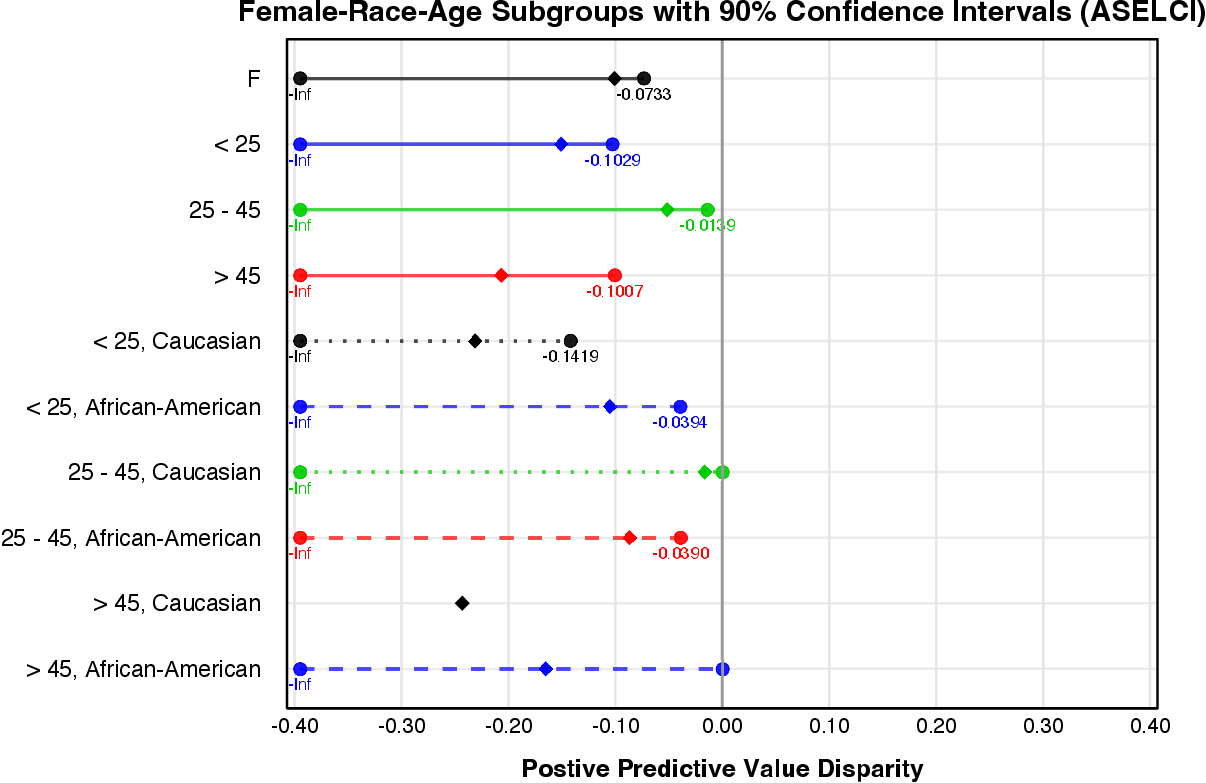}
\subcaption{ASEL}
\label{fig-PPV-Female-Race-Age-0.9-d}
\end{subfigure}\\[0.6em]
\begin{subfigure}[b]{0.24\textwidth}
\centering
\includegraphics[width=\linewidth]{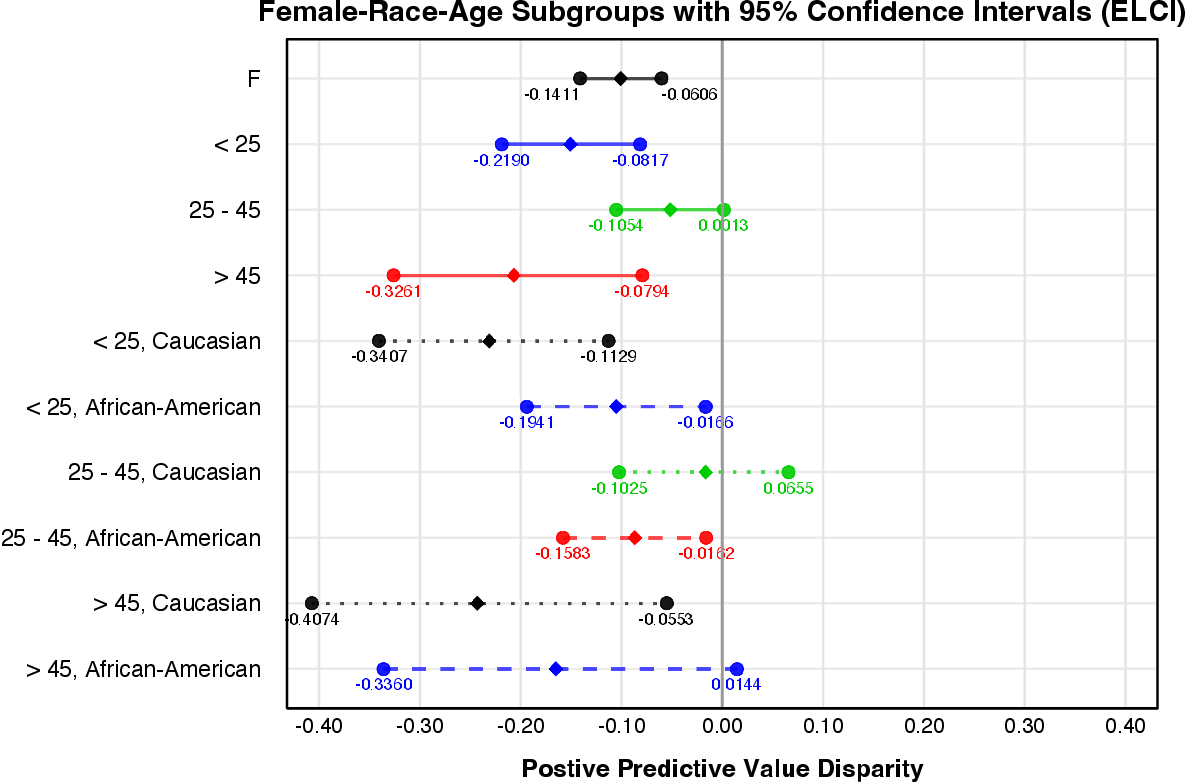}
\subcaption{EL}
\label{fig-PPV-Female-Race-Age-0.95-a}
\end{subfigure}\hfill
\begin{subfigure}[b]{0.24\textwidth}
\centering
\includegraphics[width=\linewidth]{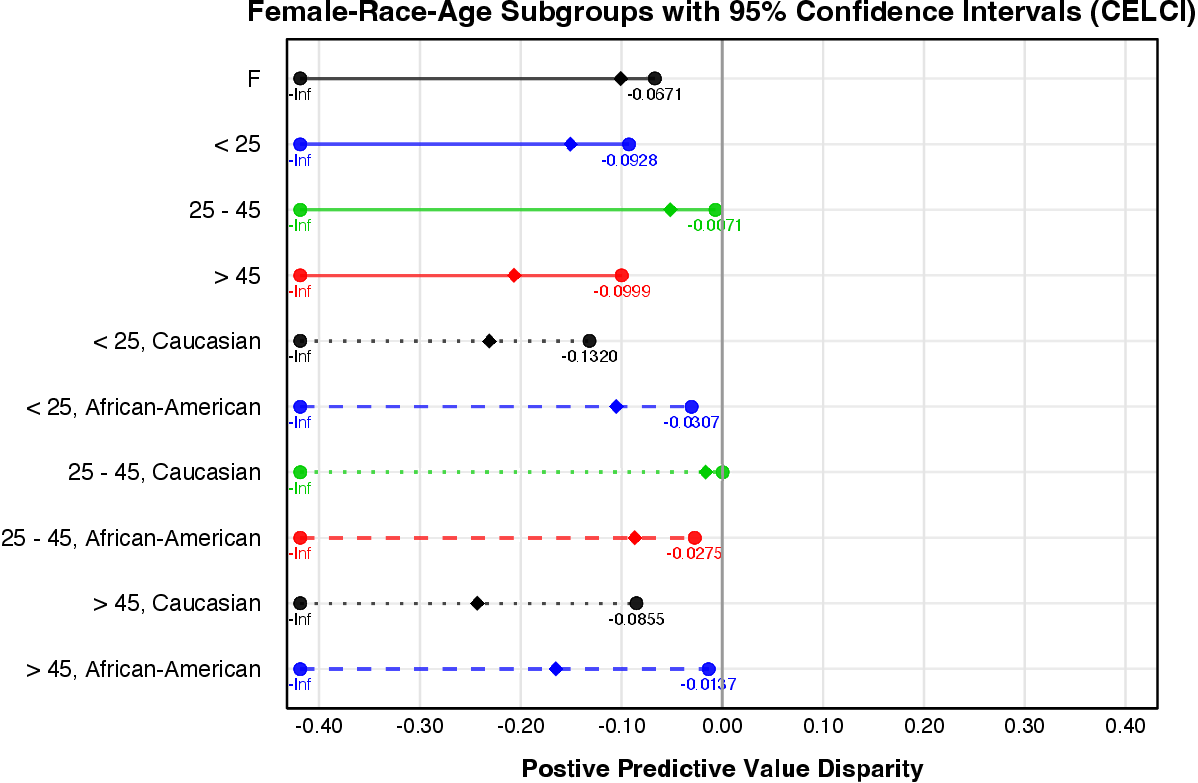}
\subcaption{CEL}
\label{fig-PPV-Female-Race-Age-0.95-b}
\end{subfigure}\hfill
\begin{subfigure}[b]{0.24\textwidth}
\centering
\includegraphics[width=\linewidth]{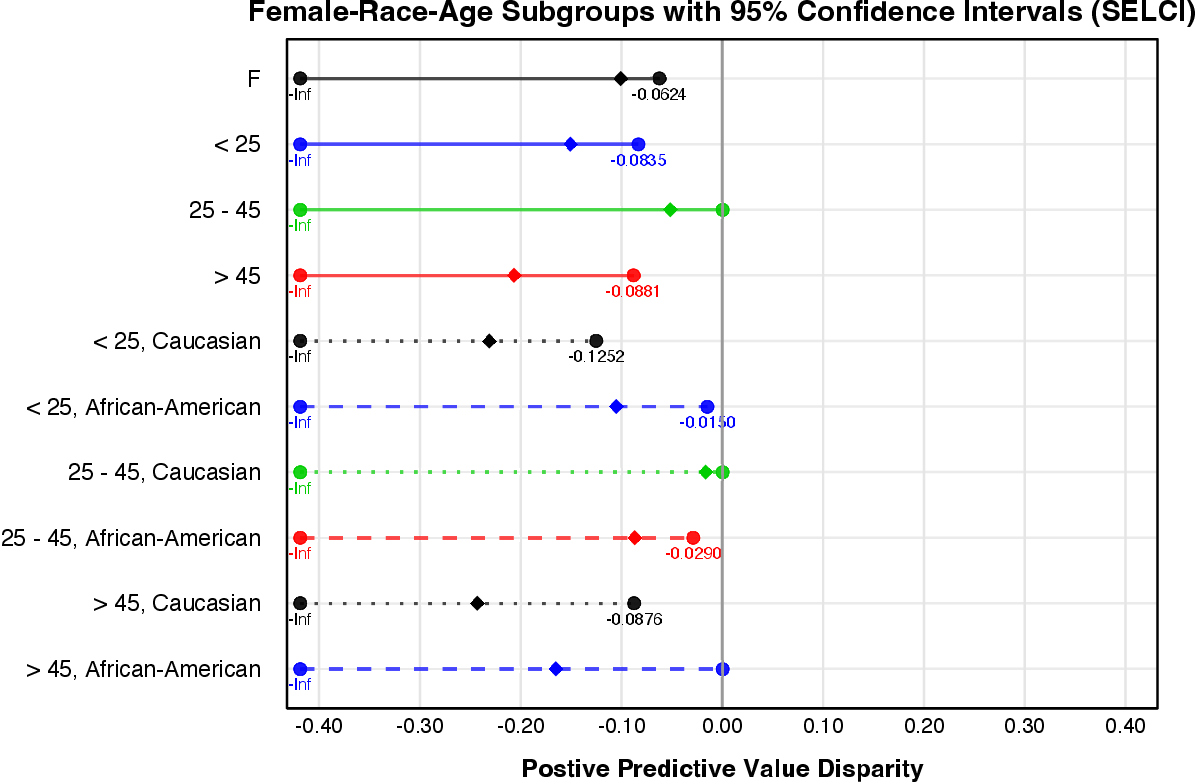}
\subcaption{SEL}
\label{fig-PPV-Female-Race-Age-0.95-c}
\end{subfigure}\hfill
\begin{subfigure}[b]{0.24\textwidth}
\centering
\includegraphics[width=\linewidth]{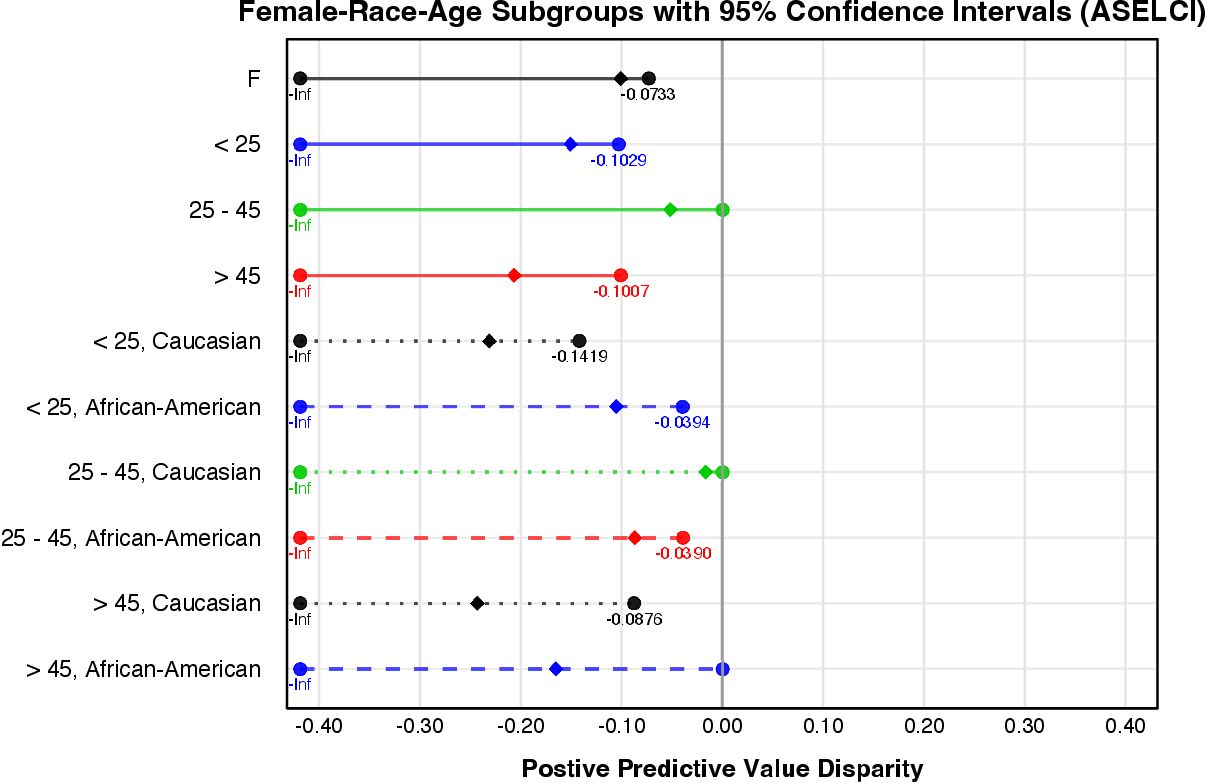}
\subcaption{ASEL}
\label{fig-PPV-Female-Race-Age-0.95-d}
\end{subfigure}
\caption{
Pointwise inversion regions for the COMPAS PPV disparity of each female subgroup defined by the intersection of race and age relative to the overall PPV among defendants receiving a positive COMPAS prediction.
Top row: 90\% nominal level; bottom row: 95\% nominal level. Six race--age intersection subgroups are omitted because the corresponding cell sizes fall below the minimum of $n=8$: $\!<25$, \textit{Hispanic}; $\!<25$, \textit{Other}; $25$--$45$, \textit{Hispanic}; $25$--$45$, \textit{Native American}; $>45$, \textit{Hispanic}; and $>45$, \textit{Native American}.
}
\label{fig-PPV-Race-Sex-Age-Disparity}
\end{figure}

In summary, the proposed tolerance-based auditing framework identifies three prominent patterns of predictive disparity in the COMPAS algorithm: 
higher PPV among African-American males under age~25 relative to the Caucasian reference; 
higher PPV among younger males pooled across racial categories relative to the same reference; 
and lower PPV among Caucasian females under age~25 and over age~45 relative to the overall average.

\section{Conclusion}\label{sec-conclusion}

This article develops a unified tolerance-based  fairness auditing framework that accommodates two complementary auditing objectives through three procedures: CEL, SEL, and ASEL.
CEL uses LFP calibration to control the risk of false violation declarations, thereby supporting formal violation certification.
For sensitivity-oriented auditing, we introduce an adaptive boundary-proxy principle based on data-adaptive evidence aggregation.
SEL and ASEL instantiate this principle through split empirical likelihood and exhibit the screening behavior characterized by the theory.
Simulations under normal, heavy-tailed, and skewed distributions illustrate the distinct finite-sample operating characteristics of violation certification and sensitivity screening, including their trade-offs between null rejection probabilities and detection sensitivity.
%The COMPAS analysis demonstrates how the proposed procedures and one-sided inversion bounds can be used to assess predictive-parity disparities across demographic and intersectional subgroups. The resulting one-sided CEL bound provides a quantitative violation-certification boundary against which application-specific tolerance levels can be directly assessed.
%In particular, the analysis identifies higher PPV among African-American defendants than among Caucasian defendants, with the disparity especially pronounced among African-American males under age~25.
The COMPAS application illustrates how the proposed framework interprets observed disparities in light of both the prespecified tolerance and the auditing objective, distinguishing formal violation certification from sensitivity-oriented screening.

There are two limitations that suggest directions for future research.
First, when the null and alternative hypotheses are interchanged, constructing a powerful procedure suitable for sensitivity screening remains technically challenging.
Second, sensitive attributes may be unavailable in practice because of privacy protections or difficulties in identifying reliable proxy variables.
Extending tolerance-based auditing to settings with missing sensitive attributes is therefore an important topic for future research.

\section{Conflicts of interest}
The authors have no conflicts of interest to declare.

\section{Funding}
Chuanlong Xie was supported by the National Natural Science Foundation of China (NSFC.12201048). Lixing Zhu was supported by the grants (NSFC12131006, NSFC12471276) from the National Natural Scientific Foundation of China and the grant (CI2023C063YLL) from the Scientific and Technological Innovation Project of China Academy of Chinese Medical Science.

\section{Supplementary Material}
The supplementary material contains an extension to lower-tail tolerance auditing and all proofs for the theoretical results presented in the main text of this paper.
\bibliographystyle{abbrvnat}
\bibliography{references}

@article{BenjaminiHochberg1995,
title={Controlling the false discovery rate: a practical and powerful approach to multiple testing},
author={Benjamini, Yoav and Hochberg, Yosef},
journal={Journal of the Royal statistical society: series B (Methodological)},
volume={57},
number={1},
pages={289--300},
year={1995},
publisher={Wiley Online Library}
}

@article{Owen1988,
  author = {Owen, Art B.},
  year = {1988},
  title = {Empirical likelihood ratio confidence intervals for a single functional},
  journal = {Biometrika},
  volume = {75},
  pages = {237--249}
}

@article{Owen1990,
  author = {Owen, Art B.},
  year = {1990},
  title = {Empirical likelihood ratio confidence regions},
  journal = {The Annals of Statistics},
  volume = {18},
  pages = {90--120}
}

@article{CaldersVerwer2010,
  author = {Calders, Toon and Verwer, Sicco},
  year = {2010},
  title = {Three naive {Bayes} approaches for discrimination-free classification},
  journal = {Data Mining and Knowledge Discovery},
  volume = {21},
  number = {2},
  pages = {277--292},
  doi = {10.1007/s10618-010-0190-x}
}

@inproceedings{DworkEtAl2012,
  author = {Dwork, Cynthia and Hardt, Moritz and Pitassi, Toniann and Reingold, Omer and Zemel, Richard},
  year = {2012},
  title = {Fairness through awareness},
  booktitle = {Proceedings of the 3rd Innovations in Theoretical Computer Science Conference},
  pages = {214--226},
  publisher = {Association for Computing Machinery},
  address = {New York, NY},
  doi = {10.1145/2090236.2090255}
}

@article{PessachShmueli2022,
  author = {Pessach, Dana and Shmueli, Erez},
  year = {2022},
  title = {A review on fairness in machine learning},
  journal = {ACM Computing Surveys},
  volume = {55},
  number = {3},
  pages = {1--44},
  doi = {10.1145/3494672}
}

@article{CatonHaas2024,
  author = {Caton, Simon and Haas, Christian},
  year = {2024},
  title = {Fairness in machine learning: A survey},
  journal = {ACM Computing Surveys},
  volume = {56},
  number = {7},
  articleno = {166},
  numpages = {38},
  doi = {10.1145/3616865}
}

@book{LehmannRomano2005,
  title = {Testing Statistical Hypotheses},
  author = {Lehmann, Erich L. and Romano, Joseph P.},
  year = {2005},
  edition = {3},
  publisher = {Springer},
  address = {New York}
}

@incollection{AngwinEtAl2016,
  author = {Angwin, J. and Larson, J. and Mattu, S. and Kirchner, L.},
  year = {2016},
  title = {Machine Bias},
  booktitle = {Ethics of Data and Analytics},
  pages = {254--264},
  publisher = {Auerbach Publications}
}

@article{CherianCandes2024,
  author = {Cherian, John J. and Cand{\`e}s, Emmanuel J.},
  year = {2024},
  title = {Statistical Inference for Fairness Auditing},
  journal = {Journal of Machine Learning Research},
  volume = {25},
  number = {149},
  pages = {1--49},
  url = {http://jmlr.org/papers/v25/23-0739.html}
}

@techreport{DieterichEtAl2016,
  author = {Dieterich, William and Mendoza, Christina and Brennan, Tim},
  year = {2016},
  title = {{COMPAS} Risk Scales: Demonstrating Accuracy Equity and Predictive Parity},
  institution = {Northpointe Inc.},
  type = {Technical report},
  number = {7.4},
  pages = {1}
}

@article{FloresEtAl2016,
  author = {Flores, Anthony W. and Bechtel, Kristin and Lowenkamp, Christopher T.},
  year = {2016},
  title = {False Positives, False Negatives, and False Analyses: A Rejoinder to ``Machine Bias''},
  journal = {Federal Probation},
  volume = {80},
  pages = {38--46}
}

@article{BrundageEtAl2020,
  author = {Brundage, Miles and Avin, Shahar and Wang, Jasmine and Belfield, Haydn and Krueger, Gretchen and Hadfield, Gillian and Khlaaf, Heidy and Yang, Jingying and Toner, Helen and Fong, Ruth and Maharaj, Tegan and Koh, Pang Wei and Hooker, Sara and Leung, Jade and Trask, Andrew and Bluemke, Emma and Lebensold, Jonathan and O'Keefe, Cullen and Koren, Mark and Ryffel, Theo and Rubinovitz, Jesse and Besiroglu, Tamay and Carugati, Francesco and Clark, Jack and Eckersley, Peter and de Haas, Sarah and Johnson, Maritza and Laurie, Ben and Ingerman, Alex and Krawczuk, Igor and Askell, Amanda and Cammarota, Rob and Lohn, Andrew and Krueger, David and Stix, Charlotte and Henderson, Peter and Graham, L{\'e}onard and Prunkl, Carina and Martin, Brian and Seger, Elizabeth and Zilberman, Nathan and {h{\'E}igeartaigh}, Se{\'a}n {\'O}},
  year = {2020},
  title = {Toward Trustworthy {AI} Development: Mechanisms for Supporting Verifiable Claims},
  journal = {arXiv preprint arXiv:2004.07213}
}

@article{LaineEtAl2024,
  author = {Laine, Joonas and Minkkinen, Matti and M{\"a}ntym{\"a}ki, Matti},
  year = {2024},
  title = {Ethics-based auditing to develop trustworthy artificial intelligence},
  journal = {Journal of Business Ethics}
}

@article{LacmanovicSkare2025,
  author = {Lacmanovi{\'c}, Sara and {\v{S}}kare, Marinko},
  year = {2025},
  title = {Algorithmic auditing and trustworthy artificial intelligence},
  journal = {Technology in Society}
}

@inproceedings{TramerEtAl2017,
  author = {Tram{\`e}r, Florian and Atlidakis, Vaggelis and Geambasu, Roxana and Hsu, Daniel and Hubaux, Jean-Pierre and Humbert, Mathias and Juels, Ari and Lin, Huang},
  year = {2017},
  title = {FairTest: Discovering unwarranted associations in data-driven applications},
  booktitle = {Proceedings of the 2nd IEEE European Symposium on Security and Privacy},
  pages = {401--416}
}

@article{XueEtAl2020,
  author = {Xue, Songkai and Ma, Wenshuo and Wang, Jingshu and Yang, Yun},
  year = {2020},
  title = {Auditing machine learning algorithms for fairness and discrimination},
  journal = {arXiv preprint arXiv:2009.06151}
}

@article{YanZhang2022,
  author = {Yan, Tom and Zhang, Chicheng},
  year = {2022},
  title = {Active fairness auditing},
  journal = {arXiv preprint arXiv:2206.13063}
}

@article{vonZahnEtAl2023,
  author = {von Zahn, Moritz and Feuerriegel, Stefan and K{\"u}hl, Niklas},
  year = {2023},
  title = {Towards a framework for certifying fairness in machine learning},
  journal = {arXiv preprint arXiv:2302.08509}
}

@inproceedings{DiCiccioEtAl2020,
  author = {DiCiccio, Cyrus J. and Vasudevan, Sriram and Basu, Kinjal and Kenthapadi, Krishnaram and Agarwal, Deepak},
  year = {2020},
  title = {Evaluating fairness using permutation tests},
  booktitle = {Proceedings of the 26th ACM SIGKDD International Conference on Knowledge Discovery and Data Mining},
  pages = {1467--1477}
}

@article{TangEtAl2026,
  author = {Tang, Jie and Xie, Chuanlong and Zeng, Xiang and Zhu, Lixing},
  year = {2026},
  title = {Empirical likelihood-based fairness auditing: distribution-free certification and flagging},
  journal = {arXiv preprint arXiv:2601.20269}
}

@inproceedings{GargEtAl2020,
  author = {Garg, P. and Villasenor, J. and Foggo, V.},
  title = {Fairness Metrics: A Comparative Analysis},
  booktitle = {IEEE International Conference on Big Data (Big Data)},
  pages = {3662--3666},
  year = {2020},
  publisher = {IEEE}
}

@inproceedings{VermaRubin2018a,
  author = {Verma, S. and Rubin, J.},
  title = {Fairness Definitions Explained},
  booktitle = {Proceedings of the International Workshop on Software Fairness},
  pages = {1--7},
  year = {2018}
}

@article{MehrabiEtAl2021,
  author = {Mehrabi, N. and Morstatter, F. and Saxena, N. and Lerman, K. and Galstyan, A.},
  title = {A Survey on Bias and Fairness in Machine Learning},
  journal = {ACM Computing Surveys (CSUR)},
  volume = {54},
  number = {6},
  pages = {1--35},
  year = {2021}
}

@article{DresselFarid2018,
  author = {Dressel, J. and Farid, H.},
  year = {2018},
  title = {The Accuracy, Fairness, and Limits of Predicting Recidivism},
  journal = {Science Advances},
  volume = {4},
  number = {1},
  pages = {eaao5580}
}

@inproceedings{FraserEtAl2022,
  author = {Fraser, K. C. and Kiritchenko, S. and Nejadgholi, I.},
  year = {2022},
  title = {Extracting Age-Related Stereotypes from Social Media Texts},
  booktitle = {Proceedings of the 2022 Annual Meeting of the Association for Computational Linguistics},
  pages = {3183--3194},
  publisher = {Association for Computational Linguistics}
}

@article{GargEtAl2023,
  author = {Garg, T. and Masud, S. and Suresh, T. and Chakraborty, T.},
  year = {2023},
  title = {Handling Bias in Toxic Speech Detection: A Survey},
  journal = {ACM Computing Surveys},
  volume = {55},
  number = {13s},
  pages = {1--32}
}

@inproceedings{HardtEtAl2016,
  author = {Hardt, M. and Price, E. and Srebro, N.},
  year = {2016},
  title = {Equality of Opportunity in Supervised Learning},
  booktitle = {Advances in Neural Information Processing Systems},
  volume = {29},
  pages = {3315--3323}
}

@inproceedings{JiangEtAl2023,
  author = {Jiang, X. and Dai, Y. and Wu, Y.},
  year = {2023},
  title = {Fair Selection through Kernel Density Estimation},
  booktitle = {Proceedings of the 2023 International Joint Conference on Neural Networks},
  pages = {1--8},
  publisher = {IEEE}
}

@article{LuoEtAl2024,
  author = {Luo, R. and Tang, T. and Xia, F. and Liu, J. and Xu, C. and Zhang, L. Y. and Xiang, W. and Zhang, C.},
  year = {2024},
  title = {Algorithmic Fairness: A Tolerance Perspective},
  journal = {arXiv preprint arXiv:2405.09543},
  doi = {10.48550/arXiv.2405.09543},
  url = {https://arxiv.org/abs/2405.09543}
}

@article{MehrabiEtAl2022,
  author = {Mehrabi, N. and Morstatter, F. and Saxena, N. and Lerman, K. and Galstyan, A.},
  year = {2022},
  title = {A Survey on Bias and Fairness in Machine Learning},
  journal = {arXiv preprint arXiv:1908.09635},
  doi = {10.48550/arXiv.1908.09635},
  url = {https://arxiv.org/abs/1908.09635}
}

@article{PoriaEtAl2020,
  author = {Poria, S. and Hazarika, D. and Majumder, N. and Mihalcea, R.},
  year = {2020},
  title = {Beneath the Tip of the Iceberg: Current Challenges and New Directions in Sentiment Analysis Research},
  journal = {IEEE Transactions on Affective Computing},
  pages = {1--1}
}

@article{SalimiEtAl2020,
  author = {Salimi, B. and Howe, B. and Suciu, D.},
  year = {2020},
  title = {Database Repair Meets Algorithmic Fairness},
  journal = {ACM SIGMOD Record},
  volume = {49},
  number = {1},
  pages = {34--41}
}

% When \showappendixtrue is selected, merge the supplementary material into
% this PDF starting on a fresh page.
% supp-body.tex is the single source, shared with the standalone
% "Supplementary Material.tex"; edit the content there, never here.
\ifshowappendix
\clearpage
\begin{appendices}
%%
%% Shared supplementary body.
%% \input by BOTH:
%%   - SEL-BRSSB.tex          (merged into the manuscript PDF, after \clearpage)
%%   - Supplementary Material.tex (standalone supplement PDF)
%% Both parents wrap it in \begin{appendices} ... \end{appendices}.
%% Do NOT put a preamble or \begin{document} here.
%%

% Restart supplementary page numbering at 1. Both parent documents enable
% hypertexnames=false in the preamble to keep internal page anchors unique.
\pagenumbering{arabic}

% Shared supplementary title, author, and affiliation block.
\begingroup
\noindent\rule{\textwidth}{1pt}\par
\vspace{3pt}%
{\Large\bfseries\raggedright
Supplementary Material: Tolerance-Based Fairness Auditing: Violation Certification and Sensitivity Screening\par}%
\vspace{8.8pt}%
{\bfseries\boldmath\fontsize{12}{15}\selectfont\raggedright
Jie Tang\textsuperscript{1},
Chuanlong Xie\ORCID{0000-0003-4292-8782}\textsuperscript{2}, and
Lixing Zhu\ORCID{0000-0003-1611-4267}\textsuperscript{3,$\ast$}\par}%
\vspace{3pt}%
{\fontsize{10}{13}\selectfont\rightskip10pt\raggedright
\textsuperscript{1,2,3}School of Statistics, Beijing Normal University,
Zhuhai 519087, Guangdong, China\par}%
\vspace{14pt}%
\endgroup
\thispagestyle{plain}

% Number supplementary equations and figures within each appendix section:
% A.1, A.2, ...; B.1, B.2, ... .
\counterwithin{equation}{section}
\counterwithin{figure}{section}

\section{Extension to lower-tail tolerance auditing}\label{app-extended-test}

Tolerance-based fairness auditing may also concern whether a negative disparity exceeds a prespecified tolerance in magnitude. 
We next consider one representative extension:
\bea \label{test-2} 
H_0:\ -\epsilon_0 \le \epsilon_G \le 0 \quad \text{versus} \quad H_1:\ \epsilon_G <-\epsilon_0,
\eea
where $\epsilon_0$ denotes a fixed positive real number.
Analogously to the discussion of test~\meqref{test-1}, the procedure constructs a one-sided upper confidence bound for the performance disparity and compares this bound with the prespecified lower tolerance threshold $-\epsilon_0$.
If the upper confidence bound is no smaller than $-\epsilon_0$, the audit does not provide sufficient evidence to certify a lower-tail tolerance violation.
If the upper confidence bound is below $-\epsilon_0$, the audited subpopulation is certified as exhibiting a lower-tail tolerance violation.
When the upper confidence bound equals zero, the lower-tail analysis does not certify a negative disparity; evidence in the opposite direction should be assessed separately if relevant to the auditing objective.

The following lower-tail extension illustrates how the proposed constructions can be adapted when the adverse direction is reversed.
The split-sample design has three reusable components: $T_1=\ell_1(0)$, which proxies estimation error and retains an asymptotic $\chi^2_1$ limit; a boundary component constructed from $\ell_2(\cdot)$; and a data-driven gate based on a standardized version of the auxiliary estimator $\hepsilon_2$.
For the lower-tail counterpart considered here, the construction is obtained by changing the boundary anchor in \(\ell_2(\cdot)\) and reversing the gate direction.
In the interior of $H_0$, the gate is asymptotically inactive and the statistic adapts to $T_1$; under $H_1$, the gate corresponding to the violated boundary is activated and the boundary component diverges.
Both SEL and ASEL continue to use the critical value $\chi_{1,\alpha_1}^2$.
We define the corresponding CEL, SEL, and ASEL statistics below.

For problem~\eqref{test-2}, the least favorable boundary is $-\epsilon_0$. The CEL statistic is therefore
\bea \label{Tc-2}
T_c= \left( -2 \log \frac{L(-\epsilon_0)}{L(\hepsilon_G)} \right) \mathbbm{1}(\hepsilon_G < -\epsilon_0) = \ell(-\epsilon_0) \mathbbm{1}(\hepsilon_G < -\epsilon_0).  \nn
\eea
For the split procedures, let
\(
T_1=\ell_1(0),\ 
T_2=\ell_2(-2\epsilon_0),\ 
T_3=\sqrt{m}\,(\hepsilon_2+2\epsilon_0).
\)
We define the SEL statistic by
\[
T_s=T_1+T_2\,\mathbbm{1}(T_3<-c_{\alpha_2}),
\]
where $c_{\alpha_2}=\Phi^{-1}(1-\alpha_2)\sigma_2$.
If $-\epsilon_0<\epsilon_G^*\le0$, then $\hepsilon_2+2\epsilon_0\pto 2(\epsilon_G^*+\epsilon_0)>0$, so $\mathbbm{1}(T_3<-c_{\alpha_2})\pto 0$ and $T_s=T_1+o_p(1)\tod\chi_1^2$.
At the boundary $\epsilon_G^*=-\epsilon_0$, the gate has a nondegenerate limiting activation probability, analogously to the upper-tail case.
Under $H_1$, when $\epsilon_G^*<-\epsilon_0$, the boundary proxy $T_3$ diverges to $-\infty$ at the $\sqrt{n}$ rate, so $\mathbbm{1}(T_3<-c_{\alpha_2})\pto 1$ and $T_s\pto\infty$.
Let $\widetilde{T}_2=\ell_2(0)$. The ASEL statistic is
\[
T_a=T_1+\widetilde{T}_2\,\mathbbm{1}(T_3<-c_{\alpha_2}).
\]
The interior-null behavior, boundary rejection behavior, local-alternative sensitivity, and fixed-alternative consistency of the CEL, SEL, and ASEL procedures for~\eqref{test-2} follow from the same arguments as Theorems~\mref{theo-2},~\mref{theo-4},~\mref{theo-5} and Corollaries~\mref{coro-2},~\mref{coro-3}, with the boundary point $\epsilon_0$ replaced by $-\epsilon_0$ and the gate direction reversed. We therefore omit these parallel statements.

We next examine the lower-tail one-sided test~\eqref{test-2} in Section~\ref{app-extended-test}.
For these experiments, we use the same values $\epsilon_0=\epsilon_1=\epsilon_2=0.25$, $\alpha=0.05$, $\alpha_2=0.08$, and $c_{\alpha_2}=\Phi^{-1}(1-\alpha_2)\,\hat{\sigma}_2$ as above, and we compare EL, CEL, SEL, ASEL, and the Student $t$ test under the normal and heavy-tailed designs.

Under the lower-tail tolerance null, rejection frequencies should again be interpreted according to the procedures' distinct operating criteria. CEL is evaluated as a violation-certification procedure, whereas SEL and ASEL are evaluated as sensitivity-oriented screening procedures and do not target uniform level-$\alpha$ control over the entire composite null.
Accordingly, boundary rejection frequencies above $\alpha$ for SEL and ASEL should be interpreted as screening behavior, rather than as evidence of calibrated violation certification.

Figure~\ref{fig-Power-test2} presents the simulation results for the lower-tail one-sided test~\eqref{test-2} considered in Section~\ref{app-extended-test}.
The normal and heavy-tailed panels display behavior similar to that observed for test~\meqref{test-1}.
As $\epsilon_G^*$ moves into the interior of $H_0$, the rejection frequencies of CEL and the $t$ test approach zero, whereas SEL and ASEL remain close to the nominal level.
Compared with CEL, ASEL exhibits higher rejection frequencies over much of the alternative region considered, while SEL shows smaller gains, mainly near the null boundary. These patterns indicate greater sensitivity to lower-tail violations in these simulation settings, with the improvement more pronounced for ASEL.
In the small-sample heavy-tailed setting, CEL exhibits slightly higher rejection frequencies than the Student $t$ test at some of the alternative parameter values considered.

% Figure 4: test~2 under normal and heavy-tailed models
\begin{figure}[H]
\centering
\begin{subfigure}[b]{0.25\textwidth}
\centering
\includegraphics[width=\linewidth]{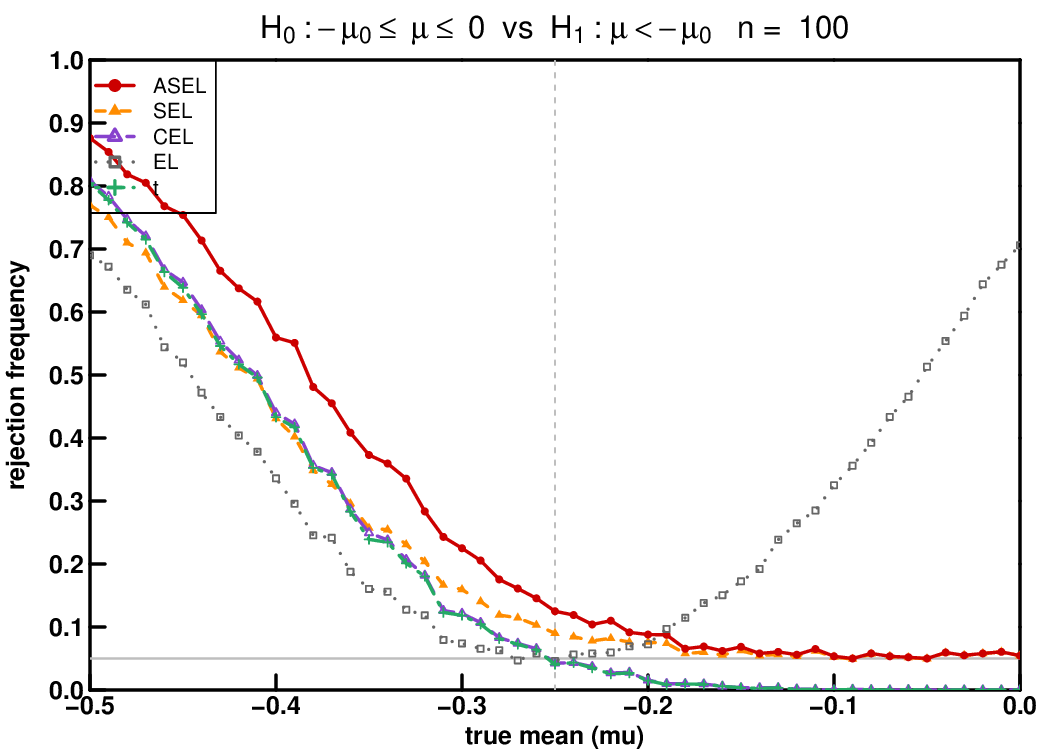}
\subcaption{}
\label{fig-Power-test2-a}
\end{subfigure}\hfill
\begin{subfigure}[b]{0.25\textwidth}
\centering
\includegraphics[width=\linewidth]{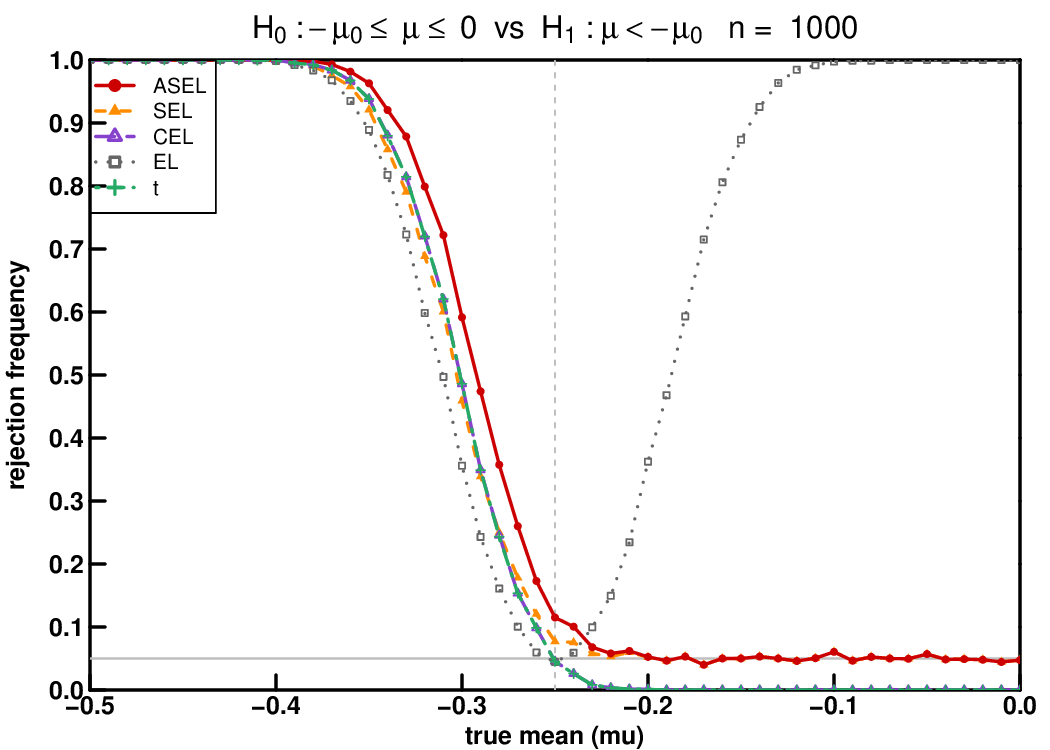}
\subcaption{}
\label{fig-Power-test2-b}
\end{subfigure}\hfill
\begin{subfigure}[b]{0.25\textwidth}
\centering
\includegraphics[width=\linewidth]{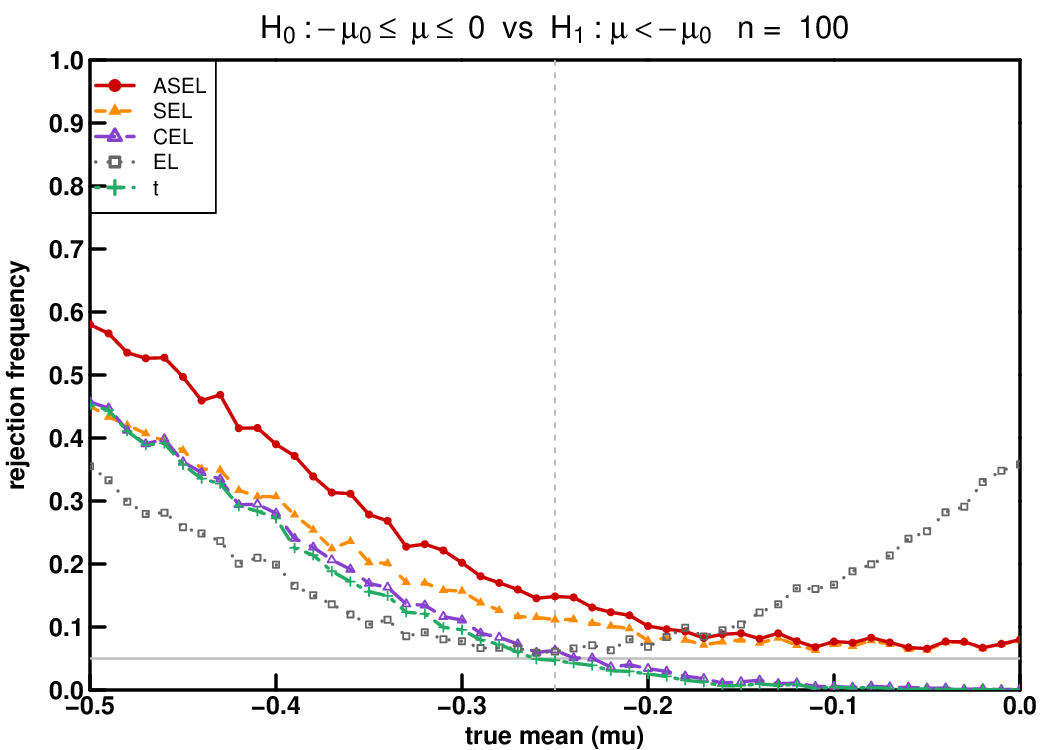}
\subcaption{}
\label{fig-Power-test2-c}
\end{subfigure}\hfill
\begin{subfigure}[b]{0.25\textwidth}
\centering
\includegraphics[width=\linewidth]{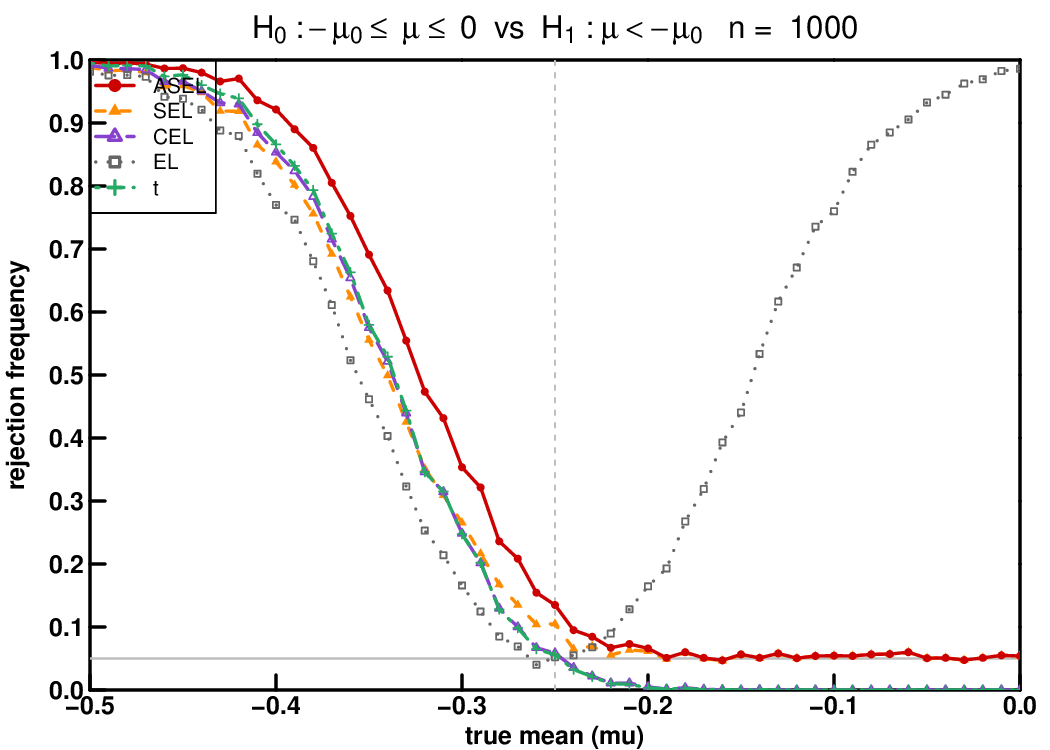}
\subcaption{}
\label{fig-Power-test2-d}
\end{subfigure}
\caption{Rejection-frequency behavior for lower-tail tolerance auditing under the symmetric normal model and the heavy-tailed Student~$t$ model: (a)--(b) correspond to $n=100,1000$ under the normal model, and (c)--(d) correspond to $n=100,1000$ under the heavy-tailed Student~$t$ model.
The horizontal gray line is the nominal level $\alpha=0.05$.
The vertical dashed line marks the null boundary.}
\label{fig-Power-test2}
\end{figure}

\section{Technical proofs}\label{app-proofs}
This appendix collects detailed proofs of the theorems and lemmas stated in the main text.

ASEL requires particular attention at the tolerance boundary. Under the assumptions of Theorem~\mref{theo-5}, when $\epsilon_G^*=\epsilon_0>0$, the component $\widetilde{T}_2=\ell_2(0)\pto\infty$, so the conditional rejection probability given gate activation tends to one. The gate opens with limiting probability $\alpha_2$. When it remains closed, the statistic equals $T_1$, whose asymptotic independence from the gate yields a conditional rejection probability tending to $\alpha_1$ at the threshold $\chi_{1,\alpha_1}^2$. Consequently,
\[
\lim_{n_G\to\infty}\mathbb{P}\{T_a>\chi_{1,\alpha_1}^2\}
=\alpha_1(1-\alpha_2)+\alpha_2.
\]
The same mechanism applies at the lower-tail boundary $\epsilon_G^*=-\epsilon_0$ after reversing the gate direction. This boundary rejection behavior is central to ASEL's sensitivity-oriented screening role. In particular, with $\alpha_1=\alpha$ and $\alpha_2>0$, the boundary rejection probability exceeds $\alpha$, so the interior $\chi_1^2$ limit does not imply uniform level-$\alpha$ control over the composite null.

\subsection{Related lemmas}\label{app-Related}

We first recall two empirical-likelihood results from \citet{TangEtAl2026} used repeatedly in the proofs. The first treats known $\theta_P$, while the second shows that replacing $\theta_P$ by a consistent estimator preserves the same Wilks limit.

\begin{lemm}[\citet{TangEtAl2026}, Theorem~1]\label{lemm-theoEL}
Suppose that conditions~\eqref{C1}--\eqref{C2} hold. 
Let $\epsilon_G^*$ denote the true performance disparity.
Then,
\[
\ell(\epsilon_G^*;\theta_P)\tod \chi^2_1 \quad \text{as} \quad n\to \infty,
\]
where $\chi^2_1$ is a chi-squared random variable with one degree of freedom.
\end{lemm}

\begin{lemm}[\citet{TangEtAl2026}, Theorem~2]\label{lemm-theoPIEL}
Suppose that conditions~\eqref{C1}--\eqref{C3} hold. Let $\epsilon_G^*$ denote the true performance disparity. Then,
\[
\ell(\epsilon_G^*;\htheta)\tod \chi^2_1 \quad \text{as} \quad n\to\infty,
\]
where $\chi^2_1$ is a chi-squared random variable with one degree of freedom.
\end{lemm}
Since $\theta_P$ can be replaced by the consistent estimator $\htheta$ when it is unknown, we omit $\theta_P$ and $\htheta$ from the notation in the sequel for notational simplicity.  

\begin{lemm}[Theorem 1 in \cite{BenjaminiHochberg1995}] \label{lemm-BH}
For independent test statistics and for any configuration of false null hypotheses, the above procedure controls the FFR at $\alpha$.
\end{lemm}

\subsection{Derivation of the constrained empirical likelihood statistic}\label{app-cel-derivation}

We first derive a convenient representation of the constrained empirical likelihood statistic $T_c$ introduced in~\meqref{eq-logL-ineq}.
The following concavity result characterizes the relevant constrained maxima.
\begin{lemm}[\citet{TangEtAl2026}, Lemma~1] \label{lemm_ell}
The empirical likelihood function \( L(\epsilon_G) \) is concave with respect to \( \epsilon_G \).
\end{lemm}
Lemma~\ref{lemm_ell} implies that the empirical likelihood $L(\epsilon_G)$ has a unique maximizer.
Without imposing a constraint on $\epsilon_G \in \mathbb{R}$, the maximum of $L(\epsilon_G)$ is attained at
\bea \label{hepsilonG}
\hepsilon_G = \frac{1}{n_G} \sum_{i = 1}^{n}M_i \mathbbm{1}_G - \theta_P,
\eea
with $\lambda=0$, $p_i(\hepsilon_G)=n^{-1}$ for $1 \le i \le n$, and $L(\hepsilon_G) = n^{-n}$.
By Lemma~\ref{lemm_ell}, the empirical log-likelihood ratio for imposing the equality constraint $\epsilon_G=\epsilon_0$ relative to the unconstrained profile likelihood admits the representation
\bea \label{eq-logL-H0}
\ell(\epsilon_0)=-2 \log { L(\epsilon_0) \over L(\hepsilon_G)} =-2 \log { \sup_{\epsilon_G=\epsilon_0}L(\epsilon) \over \sup_{\epsilon \in \mathbb{R}} L(\epsilon)}.
\eea
For the exact-equality formulation, this specializes to $\epsilon_0=0$.
For tolerance-based certification in~\meqref{test-1}, the null hypothesis restricts $\epsilon_G$ to $[0,\epsilon_0]$ and the parameter space to $\epsilon_G \ge 0$.
Accordingly, $T_c$ in~\meqref{eq-logL-ineq} replaces the unconstrained supremum in~\eqref{eq-logL-H0} with $\sup_{\epsilon_G \ge 0}L(\epsilon)$ and replaces the equality constraint in the numerator with the tolerance constraint $0\leq\epsilon_G\leq\epsilon_0$.

The application of Lemma~\ref{lemm_ell} yields
\[
\sup_{0 \le \epsilon_G \le \epsilon_0}L(\epsilon)=
\begin{cases}
L(0), & \hepsilon_G<0,\\[2pt]
L(\hepsilon_G), & 0\le \hepsilon_G\le\epsilon_0,\\[2pt]
L(\epsilon_0), & \hepsilon_G>\epsilon_0.
\end{cases}
\qquad
\sup_{\epsilon_G \ge 0}L(\epsilon)=
\begin{cases}
L(0), & \hepsilon_G<0,\\[2pt]
L(\hepsilon_G), & \hepsilon_G\ge 0,
\end{cases}
\]
It follows immediately that $T_c=0$ when $\hepsilon_G\le\epsilon_0$, and $T_c=-2\log(L(\epsilon_0)/L(\hepsilon_G))$ when $\hepsilon_G>\epsilon_0$.
Hence,
\bea \label{Tc}
T_c= \left( -2 \log \frac{L(\epsilon_0)}{L(\hepsilon_G)} \right) \mathbbm{1}(\hepsilon_G > \epsilon_0) = \ell(\epsilon_0) \mathbbm{1}(\hepsilon_G > \epsilon_0).
\eea

\subsection{Proofs of the Theorems and Corollaries in section~\mref{sec-methodology}}

% Theorem 1
{\bf Proof of Theorem~\mref{theo-1}. } 
From equation~\eqref{Tc}, we have $T_c = \ell(\epsilon_0) \mathbbm{1}(\hepsilon_G > \epsilon_0)$, where $\ell(\epsilon) = -2 \log \frac{ L(\epsilon)}{ L(\hepsilon_G)}$. 
Let $\bar{g}(\epsilon_G)=n^{-1}\sum_{i=1}^n g_i(\epsilon_G)$ and \( S_n = n^{-1} \sum_{i=1}^n (g_i (\epsilon_0))^2 \). 
When $\epsilon_G^* = \epsilon_0$, it follows from the proof of Lemma~\mref{lemm-theoEL} (\citet{TangEtAl2026}, Theorem~1) that  
\bea \label{T04}
\ell (\epsilon_0) = -2 \log \frac{L (\epsilon_0)}{L (\hepsilon_G)} = \frac{n \bar{g}^2(\epsilon_0)}{S_n} + o_p(1) = \frac{n \bar{g}^2(\epsilon_0)}{\sigma_G^2} + o_p(1),
\eea
where $ \sigma_G^2 = { \mathbb{E} \left[(M - \theta_P - \epsilon_0)^2 \mathbbm{1}_{G} \right] } = { \mathbb{P}(G) }\Var(M\ |\ G). $
Moreover, by the central limit theorem, \( \sqrt{n}\bar{g}(\epsilon_0) / \sigma_G \tod Z \), where \( Z \sim N(0,1) \). Therefore, for any \( t > 0 \),  
\bea \label{T24-1}
\begin{aligned}
 \mathbb{P}\left\{ T_c > t \right\} 
 &=  \mathbb{P}\left\{ \ell(\epsilon_0) > t,\; \hepsilon_G > \epsilon_0 \right\} 
 =  \mathbb{P}\left\{ \frac{n \bar{g}^2(\epsilon_0)}{\sigma_G^2} + o_p(1)  > t,\; \frac{\sqrt{n}\bar{g}(\epsilon_0)}{\sigma_G} > 0 \right\} \\
 & \to  \mathbb{P}\left\{ Z^2 > t,\; Z > 0 \right\} 
=  \mathbb{P}\left\{ Z  > \sqrt{t}  \right\} 
= \frac{1}{2} \mathbb{P}\left\{ \chi^2_1 > t \right\},
\end{aligned}\nn
\eea
since $  \mathbb{P}\left\{ \chi^2_1 > t \right\} =  \mathbb{P}\left\{ Z  > \sqrt{t}  \right\} +  \mathbb{P}\left\{ Z  < -\sqrt{t}  \right\} = 2  \mathbb{P}\left\{ Z  > \sqrt{t}  \right\} $.
Hence, we conclude that  
\[
T_c  \  \tod \  \frac{1}{2} \chi^2_0 + \frac{1}{2} \chi^2_1\quad \mathrm{for} \ \epsilon_G^* = \epsilon_0.
\]  
This completes the proof of Theorem~\mref{theo-1}.

% Theorem 2
{\bf Proof of Theorem~\mref{theo-2}. }
From Theorem~\mref{theo-1}, we have that 
$\mathbb{P}\left\{ T_c > z_{2\alpha}(1) \mid \epsilon_G^* = \epsilon_0 \right\} = \alpha$, as \( n \to \infty \).
Let \( \bar{h}(\epsilon_G) := n_G^{-1} \sum_{i=1}^nM_i \cdot \mathbbm{1}_G - \theta_P - \epsilon_G \). 
Since \( \sqrt{n}\bar{g}(\epsilon_G^*) / \sigma_G \stackrel{d}{\to} N(0, 1) \) 
and \( n_G/n \stackrel{p}{\to} \mathbb{P}(G) \), it follows from Slutsky's theorem that:
\bea \label{barh}
\sqrt{n}\bar{h}(\epsilon_G^*) = \frac{n}{n_G} \cdot  \sqrt{n}\bar{g}(\epsilon_G^*) \tod N\left(0,  \  \frac{\sigma_G^2}{\mathbb{P}^2(G)}\right).
\eea
Furthermore, for any fixed \( \epsilon_G^* < \epsilon_0 \), it follows from the above limit theorem that:  
\bea \label{T24-2}
\begin{aligned}
\mathbb{P}\left\{ T_c > z_{2\alpha}(1) \mid \epsilon_G^* \right\} 
&= \mathbb{P}\left\{ \ell(\epsilon_0) > z_{2\alpha}(1), \ \hepsilon_G > \epsilon_0 \mid \epsilon_G^* \right\} \\
&\le \mathbb{P}\left\{ \hepsilon_G > \epsilon_0 \mid \epsilon_G^* \right\} \\
&= \mathbb{P}\left\{ \frac{\sqrt{n} \mathbb{P}(G) \bar{h}(\epsilon_G^* )}{\sigma_G} > \frac{\sqrt{n}\mathbb{P}(G)  (\epsilon_0 - \epsilon_G^*)}{\sigma_G} \;\middle|\; \epsilon_G^* \right\} \\
&\to \mathbb{P}\left\{ Z > +\infty \right\} \\
&= 0,
\end{aligned}
\eea
where \( Z \sim N(0, 1) \).  
This completes the proof of Theorem~\mref{theo-2}.

% Corollary 2
{\bf Proof of Corollary \mref{coro-2}. }
Equation (\ref{T04}) remains valid when the true disparity is \( \epsilon_G^* = \epsilon_0 + \tau \sigma_G {n_G}^{-1/2} \), although in this case the true value lies within the alternative parameter space and deviates from \( \epsilon_0 \). From the central limit theorem, it can be deduced that  
\bea \label{tau}
\sqrt{n}\bar{g}(\epsilon_0) /\sigma_G 
= \sqrt{n}\bar{g}(\epsilon_G^*)/\sigma_G + \tau \sqrt{n_G/n}
\tod Z + \tau\sqrt{\mathbb{P}(G)}. 
\eea
Therefore, from (\ref{T04}), (\ref{barh}) and (\ref{tau}), it follows that  
\bea 
\begin{aligned}
\mathbb{P}\left\{ T_c > z_{2\alpha}(1) \mid \epsilon_G^* \right\} 
&= \mathbb{P}\left\{ \ell(\epsilon_0) > z_{2\alpha}(1),\; \hepsilon_G >  \epsilon_0  \mid \epsilon_G^* \right\} \\
&= \mathbb{P}\left\{ \frac{n \bar{g}^2(\epsilon_0)}{\sigma_G^2} + o_p(1) > z_{2\alpha}(1),\; \hepsilon_G > \epsilon_G^* - \tau \sigma_G {n_G}^{-1/2}   \right\} \\
&= \mathbb{P}\left\{ \frac{n \bar{g}^2(\epsilon_0)}{\sigma_G^2} + o_p(1) > z_{2\alpha}(1),\; \frac{\sqrt{n} \mathbb{P}(G) \bar{h}(\epsilon_G^* )}{\sigma_G} >  -\tau\sqrt{\mathbb{P}(G)} \right\} \\
&\to \mathbb{P}\left\{  (Z +  \tau\sqrt{\mathbb{P}(G)}  )^2 > z_{2\alpha}(1),\; Z + \tau\sqrt{\mathbb{P}(G)} >  0 \right\} \\
&= \mathbb{P}\left\{ Z > \sqrt{z_{2\alpha}(1)} - \tau\sqrt{\mathbb{P}(G)}   \right\} \\
&= \Phi\left( \tau\sqrt{\mathbb{P}(G)} - \sqrt{z_{2\alpha}(1)} \right).
\end{aligned} \nn 
\eea
This completes the proof of Corollary \mref{coro-2}.

% Add Theorem 
{\bf Proof of Theorem \ref{theoFFR}.} The proof is completed via Lemma \ref{lemm-BH}.

% Theorem 3
{\bf Proof of Theorem~\mref{theo-3}. }
For $j=1,2$, let
\[
\tilde g_i^j(\epsilon)=\omega_{ji}-\epsilon,\quad
\bar g_j(\epsilon)=m^{-1}\sum_{i=1}^m\tilde g_i^j(\epsilon),\quad
S^j_m(\epsilon)=m^{-1}\sum_{i=1}^m\{\tilde g_i^j(\epsilon)\}^2 .
\]
Since $\hepsilon_j=m^{-1}\sum_{i=1}^m \omega_{ji}$, we have $\bar g_j(\epsilon)=\hepsilon_j-\epsilon$.
By construction,
\[
\epsilon_1^*=\mathbb{E}(\omega_{1i})=0,\quad
\epsilon_2^*=\mathbb{E}(\omega_{2i})=2\epsilon_G^* .
\]
Under conditions~\meqref{C1}--\meqref{C2}, the paired observations
$(M_{a_i},M_{b_i})$, $i=1,\ldots,m$, are i.i.d.\ from the group-$G$ conditional distribution.
Hence, for each $j=1,2$, $\{\omega_{ji}\}_{i=1}^m$ are i.i.d.\ with mean $\epsilon_j^*$, variance
$0<\sigma_j^2=\Var(\omega_{ji})=2\sigma_G^2/\mathbb{P}(G)<\infty$, and a finite $2+\delta$ moment.
Applying the standard one-dimensional empirical likelihood theorem for the mean, equivalently Lemma~\mref{lemm-theoEL}, to the estimating equation
\[
\sum_{i=1}^m p_i(\omega_{ji}-\epsilon)=0,
\]
the scalar empirical-likelihood expansion gives that, uniformly for
$\epsilon-\epsilon_j^*=O(m^{-1/2})$,
\bea \label{eq-sel-local-expansion}
\ell_j(\epsilon)
=
\frac{m\{\bar g_j(\epsilon)\}^2}{S^j_m(\epsilon)}+o_p(1)
=
\frac{m(\hepsilon_j-\epsilon)^2}{S^j_m(\epsilon)}+o_p(1)
=
\frac{m(\hepsilon_j-\epsilon)^2}{\sigma_j^2}+o_p(1),
\eea
where the last equality follows from $S^j_m(\epsilon)=\sigma_j^2+o_p(1)$ uniformly in the same local neighborhood.
Taking $\epsilon=\epsilon_j^*$ yields
\[
\ell_j(\epsilon_j^*)
=\frac{m(\hepsilon_j-\epsilon_j^*)^2}{\sigma_j^2}+o_p(1)
\tod \chi_1^2,\qquad j=1,2.
\]
In particular,
\bea \label{eq-sel-T1}
\ell_1(\epsilon_1^*) = \frac{m (\hepsilon_1-\epsilon_1^*)^2}{\sigma_1^2} + o_p(1) \tod\chi^2_1,
\eea
and
\bea \label{eq-sel-T2}
\ell_2(\epsilon_2^*)= \frac{m (\hepsilon_2-\epsilon_2^*)^2}{\sigma_2^2} + o_p(1) \tod\chi^2_1.
\eea
This completes the proof of Theorem~\mref{theo-3}.

% Theorem 4
{\bf Proof of Theorem~\mref{theo-4}. }
Let
\[
Z_m:=\frac{\sqrt{m}(\hepsilon_2-2\epsilon_G^*)}{\sigma_2}.
\]
By the central limit theorem, $Z_m\tod N(0,1)$. From the definition of $T_3$, we obtain the standardized decomposition
\bea \label{eq-sel-T3-standardized}
\widetilde{T}_3
=\frac{T_3}{\sigma_2}
=\frac{\sqrt m(\hepsilon_2-2\epsilon_0)}{\sigma_2}
=Z_m+\frac{\sqrt m\{2(\epsilon_G^*-\epsilon_0)\}}{\sigma_2}.
\eea
If \(0\le\epsilon_G^*<\epsilon_0\), then $2(\epsilon_G^*-\epsilon_0)<0$ and hence
\[
T_3 = \sqrt m\{2(\epsilon_G^*-\epsilon_0)\}+O_p(1) \pto -\infty.
\]
Therefore, $\mathbb{P}(T_3>c_{\alpha_2})\to0$. On the event
$\{T_3\le c_{\alpha_2}\}$, we have $T_s=T_1$. Hence, for any continuity point $x$ of the limiting distribution of $T_1$,
\[
\left|\mathbb{P}(T_s>x)-\mathbb{P}(T_1>x)\right|
\le \mathbb{P}(T_3>c_{\alpha_2})\to0.
\]
By Theorem~\mref{theo-3}, $T_1\tod\chi_1^2$, so
\[
\mathbb{P}\{T_s>\chi_{1,\alpha_1}^2\}\to\mathbb{P}\{\chi_1^2>\chi_{1,\alpha_1}^2\}=\alpha_1.
\]

It remains to consider the boundary case $\epsilon_G^*=\epsilon_0$.  
By \eqref{eq-sel-T3-standardized}, we have
\[
\mathbbm{1}(T_3>c_{\alpha_2})=\mathbbm{1}(Z_m>z_{\alpha_2}) \tod \mathbbm{1}(Z>z_{\alpha_2}) ,
\]
where $z_{\alpha_2}=\Phi^{-1}(1-\alpha_2)$.
Since $2\epsilon_0=\epsilon_2^*$ at the boundary, the local quadratic expansion in \eqref{eq-sel-local-expansion} gives
\[
\ell_2(2\epsilon_0)
=\frac{m(\hepsilon_2-2\epsilon_0)^2}{S^2_m(2\epsilon_0)}+o_p(1).
\]
Moreover, $S^2_m(2\epsilon_0)=S^2_m(\epsilon_2^*)=\sigma_2^2+o_p(1)$.
Thus
\bea \label{eq-sel-T2-T3}
T_2=\ell_2(2\epsilon_0)=\frac{m(\hepsilon_2-2\epsilon_0)^2}{\sigma_2^2}+o_p(1)=Z_m^2+o_p(1)\tod Z^2.
\eea
The asymptotic independence used below follows from the joint central limit theorem. Indeed,
\[
\Cov(\omega_{1i},\omega_{2i})
=\Cov(M_{a_i}-M_{b_i},M_{a_i}+M_{b_i}-2\theta_P)
=\Var(M_{a_i}-\theta_P)-\Var(M_{b_i}-\theta_P)
=0.
\]
Consequently, by the multivariate central limit theorem,
\[
\sqrt m
\begin{pmatrix}
\hepsilon_1\\
\hepsilon_2-2\epsilon_G^*
\end{pmatrix}
\tod
N\left(
0,
\begin{pmatrix}
2\sigma_G^2/\mathbb{P}(G) & 0\\
0 & 2\sigma_G^2/\mathbb{P}(G)
\end{pmatrix}
\right).
\]
The limiting Gaussian components are therefore independent.
Together with the standard quadratic expansion of the empirical likelihood ratio, this joint central limit theorem gives the joint weak convergence of the empirical likelihood components in $T_s$.
In particular, $T_1=\ell_1(0)$ is asymptotically equivalent to the square of the standardized mean based on $\{\omega_{1i}\}_{i=1}^m$, whereas $Z_m$ is the standardized mean based on $\{\omega_{2i}\}_{i=1}^m$.
The continuous mapping theorem therefore gives
\[
(T_1,Z_m)\tod (U,Z),
\]
where $U\sim\chi_1^2$, $Z\sim N(0,1)$, and $U$ and $Z$ are independent. Hence
\[
T_s \tod U+Z^2\mathbbm{1}(Z>z_{\alpha_2}).
\]
Therefore,
\[
\begin{aligned}
\lim_{n_G\to\infty}\mathbb{P}\{T_s>\chi_{1,\alpha_1}^2\}
&= \mathbb{P}\left\{ U+Z^2\mathbbm{1}(Z>z_{\alpha_2})>\chi_{1,\alpha_1}^2 \right\}\\
&=\mathbb{P}\left(U>\chi_{1,\alpha_1}^2,\,Z\le z_{\alpha_2}\right)+\mathbb{P}\left(U+Z^2>\chi_{1,\alpha_1}^2,\,Z>z_{\alpha_2}\right)\\
&= \int_{-\infty}^{z_{\alpha_2}}
\mathbb{P}\{U>\chi_{1,\alpha_1}^2\}\phi(z)\,dz  + \int_{z_{\alpha_2}}^\infty
\mathbb{P}\{U>\chi_{1,\alpha_1}^2-z^2\}\phi(z)\,dz \\
&= \alpha_1\Phi(z_{\alpha_2})
+\int_{z_{\alpha_2}}^\infty
\bar F_{\chi_1^2}(\chi_{1,\alpha_1}^2-z^2)\phi(z)\,dz.
\end{aligned}
\]
This proves the first asserted boundary expression.

We next simplify this expression. Let $z_\alpha=\Phi^{-1}(1-\alpha)$. Since
$\chi_{1,\alpha_1}^2=z_{\alpha_1/2}^2$ and
$\bar F_{\chi_1^2}(x)=1$ for $x\le0$, the integral depends on the relative order of
$z_{\alpha_2}$ and $z_{\alpha_1/2}$. If $\alpha_1\ge 2\alpha_2$, then
$z_{\alpha_1/2}\le z_{\alpha_2}$, and thus
$\chi_{1,\alpha_1}^2-z^2\le0$ for all $z\ge z_{\alpha_2}$. Therefore,
\[
\lim_{n_G\to\infty}\mathbb{P}\{T_s>\chi_{1,\alpha_1}^2\}
=\alpha_1\Phi(z_{\alpha_2})+1-\Phi(z_{\alpha_2})
=\alpha_1+(1-\alpha_1)\alpha_2.
\]
If $\alpha_1<2\alpha_2$, then $z_{\alpha_1/2}>z_{\alpha_2}$, and we split the integral at
$z_{\alpha_1/2}$:
\[
\begin{aligned}
\lim_{n_G\to\infty}\mathbb{P}\{T_s>\chi_{1,\alpha_1}^2\}
&=\alpha_1\Phi(z_{\alpha_2})
+\int_{z_{\alpha_2}}^{z_{\alpha_1/2}}
\bar F_{\chi_1^2}(z_{\alpha_1/2}^2-z^2)\phi(z)\,dz
+1-\Phi(z_{\alpha_1/2}) \\
&=\alpha_1\left(\frac32-\alpha_2\right)
+\int_{z_{\alpha_2}}^{z_{\frac{\alpha_1}{2}}}
\bar F_{\chi_1^2}(z_{\frac{\alpha_1}{2}}^2-z^2)\phi(z)\,dz,
\end{aligned}
\]
where the last equality uses $\Phi(z_{\alpha_2})=1-\alpha_2$ and
$1-\Phi(z_{\alpha_1/2})=\alpha_1/2$. This completes the proof.

% Corollary 3
{\bf Proof of Corollary \mref{coro-3}. } 
Under the local parameterization $\epsilon_G^*=\epsilon_0+\tau\sigma_G/\sqrt{n_G}$, with $\tau>0$, 
since $n_G=2m$ and $\sigma_2=\sqrt{2\sigma_G^2/\mathbb{P}(G)}$, we have
\[
2\epsilon_G^*=2\epsilon_0+\frac{2\tau\sigma_G}{\sqrt{n_G}}
=2\epsilon_0+\frac{\tau\sqrt{\mathbb{P}(G)}\,\sigma_2}{\sqrt m}.
\]
It follows from \eqref{eq-sel-T3-standardized} that
\[
Z_m \tod Z,\qquad \widetilde{T}_3:=Z_m+\tau\sqrt{\mathbb{P}(G)} \tod V,
\]
where $Z\sim N(0,1)$ and $V=Z+\tau\sqrt{\mathbb{P}(G)}\sim N(\tau\sqrt{\mathbb{P}(G)},1)$. In particular,
\[
\mathbbm{1}(T_3>c_{\alpha_2})
=\mathbbm{1}(\widetilde{T}_3>z_{\alpha_2}) \tod \mathbbm{1}(V>z_{\alpha_2}),
\]
because $\mathbb{P}(V=z_{\alpha_2})=0$.

By \eqref{eq-sel-T1}, $T_1=\ell_1(0)\tod U$, where $U\sim\chi_1^2$.
Since $2\epsilon_0-\epsilon_2^*=O(m^{-1/2})$, the local expansion in \eqref{eq-sel-local-expansion} gives
\[
T_2=\ell_2(2\epsilon_0)=\frac{m(\hepsilon_2-2\epsilon_0)^2}{\sigma_2^2}+o_p(1)
=\widetilde{T}_3^{\,2}+o_p(1)\tod V^2.
\]
The joint central limit theorem argument used in the proof of Theorem~\mref{theo-4} also implies that $T_1$ and $Z_m$ are asymptotically independent, and hence $T_1$ and $V$ are asymptotically independent. Therefore,
\[
T_s=T_1+T_2\mathbbm{1}(T_3>c_{\alpha_2})
\tod U+V^2\mathbbm{1}(V>z_{\alpha_2}),
\]
where $U$ and $V$ are independent. Consequently, conditioning on $V=v$ yields
\[
\begin{aligned}
\lim_{n_G\to\infty}\mathbb{P}\{T_s>\chi_{1,\alpha_1}^2\mid \epsilon_G^*\}
&= \mathbb{P}\left\{ U+V^2\mathbbm{1}(V>z_{\alpha_2})>\chi_{1,\alpha_1}^2 \right\}\\
&=\mathbb{P}\left(U>\chi_{1,\alpha_1}^2,\,V\le z_{\alpha_2}\right)+\mathbb{P}\left(U+V^2>\chi_{1,\alpha_1}^2,\,V>z_{\alpha_2}\right)\\
&= \alpha_1\Phi(z_{\alpha_2}-\tau\sqrt{\mathbb{P}(G)}) + \int_{z_{\alpha_2}}^\infty
\mathbb{P}\{U>\chi_{1,\alpha_1}^2-v^2\}\phi(v-\tau\sqrt{\mathbb{P}(G)})\,dv \\
&= \alpha_1\Phi(z_{\alpha_2}-\tau\sqrt{\mathbb{P}(G)})
+\int_{z_{\alpha_2}}^\infty \bar F_{\chi_1^2}(\chi_{1,\alpha_1}^2-v^2)\phi(v-\tau\sqrt{\mathbb{P}(G)})\,dv.
\end{aligned}
\]
This proves the integral representation.

It remains to derive the two explicit forms. Let
\[
I(z_{\alpha_2})=
\int_{z_{\alpha_2}}^\infty
\bar F_{\chi_1^2}(\chi_{1,\alpha_1}^2-v^2)\phi(v-\tau\sqrt{\mathbb{P}(G)})\,dv.
\]
Recall that $\chi_{1,\alpha_1}^2=z_{\alpha_1/2}^2$ and $\bar F_{\chi_1^2}(x)=1$ for $x\le0$.
If $\alpha_1\ge 2\alpha_2$, then $z_{\alpha_1/2}\le z_{\alpha_2}$. Hence, for every $v\ge z_{\alpha_2}$,
\[
\bar F_{\chi_1^2}(\chi_{1,\alpha_1}^2-v^2)=1.
\]
Therefore, the second integral reduces to
\[
I(z_{\alpha_2})
=\int_{z_{\alpha_2}}^\infty \phi(v-\tau\sqrt{\mathbb{P}(G)})\,dv
=1-\Phi(z_{\alpha_2}-\tau\sqrt{\mathbb{P}(G)}).
\]
Thus,
\[
\lim_{n_G\to\infty}\mathbb{P}\{T_s>\chi_{1,\alpha_1}^2\mid \epsilon_G^*\}
=1-(1-\alpha_1)\Phi(z_{\alpha_2}-\tau\sqrt{\mathbb{P}(G)}).
\]
If $\alpha_1<2\alpha_2$, then $z_{\alpha_1/2}>z_{\alpha_2}$. Splitting the integral at
$z_{\alpha_1/2}$ gives
\[
\begin{aligned}
\lim_{n_G\to\infty}\mathbb{P}\{T_s>\chi_{1,\alpha_1}^2\mid \epsilon_G^*\}
={}&\alpha_1\Phi(z_{\alpha_2}-\tau\sqrt{\mathbb{P}(G)})
+\int_{z_{\alpha_2}}^{z_{\frac{\alpha_1}{2}}}
\bar F_{\chi_1^2}(\chi_{1,\alpha_1}^2-v^2)\phi(v-\tau\sqrt{\mathbb{P}(G)})\,dv \\
&+1-\Phi(z_{\frac{\alpha_1}{2}}-\tau\sqrt{\mathbb{P}(G)}),
\end{aligned}
\]
which is the second asserted form.

Finally, if $\epsilon_G^*>\epsilon_0$ is fixed, then $T_3\pto\infty$, and hence
$\mathbbm{1}(T_3>c_{\alpha_2})\pto1$. In addition, by the consistency of the empirical likelihood ratio under fixed alternatives, $T_2=\ell_2(2\epsilon_0)\pto\infty$ in probability, whereas $T_1=O_p(1)$. Therefore, $T_s\pto\infty$, and the rejection probability converges to one. This completes the proof.

% Theorem 5
{\bf Proof of Theorem~\mref{theo-5}. }
Recall that $\widetilde{T}_2=\ell_2(0)$ and
\[
T_a=T_1+\widetilde{T}_2\mathbbm{1}(T_3>c_{\alpha_2}).
\]
As in the proof of Theorem~\mref{theo-4}, let
\[
Z_m:=\frac{\sqrt m(\hepsilon_2-2\epsilon_G^*)}{\sigma_2},
\]
so that \eqref{eq-sel-T3-standardized} defines $\widetilde{T}_3$.

We consider the null region first.
If $0\le \epsilon_G^*<\epsilon_0$, the argument in the proof of Theorem~\mref{theo-4} gives
$T_3\pto-\infty$, and hence $\mathbb{P}(T_3>c_{\alpha_2})\to0$. Therefore, for any continuity point $x$ of the limiting distribution of $T_1$,
\[
\left|\mathbb{P}(T_a>x)-\mathbb{P}(T_1>x)\right|
\le \mathbb{P}(T_3>c_{\alpha_2})\to0.
\]
Since $T_1\tod\chi_1^2$, it follows that
\[
\mathbb{P}\{T_a>\chi_{1,\alpha_1}^2\}\to\alpha_1.
\]

We next consider the boundary case $\epsilon_G^*=\epsilon_0$. 
At the boundary, \eqref{eq-sel-T3-standardized} gives $\widetilde{T}_3=Z_m\tod Z$,
where $Z\sim N(0,1)$. Since $c_{\alpha_2}=z_{\alpha_2}\sigma_2$,
\[
\mathbbm{1}(T_3>c_{\alpha_2})
=\mathbbm{1}(\widetilde{T}_3>z_{\alpha_2})
=\mathbbm{1}(Z_m>z_{\alpha_2})\tod \mathbbm{1}(Z>z_{\alpha_2}).
\]
Since $\epsilon_0>0$, the point $0$ is separated from the mean
$\epsilon_2^*=2\epsilon_0$ of $\omega_{2i}$. Therefore, by consistency of the empirical likelihood ratio under fixed misspecification,
\[
\widetilde{T}_2=\ell_2(0)\pto\infty
\]
in probability.
Consequently, the limiting statistic may be represented as
\[
T_a \to T_a^\infty=
\begin{cases}
U, & Z\le z_{\alpha_2},\\
+\infty, & Z>z_{\alpha_2}.
\end{cases}
\]
Thus, on the event $\{Z>z_{\alpha_2}\}$, the limiting rejection probability is one, whereas on the event $\{Z\le z_{\alpha_2}\}$ the statistic reduces asymptotically to $T_1$. Using the asymptotic independence of $T_1$ and $Z_m$ established in the proof of Theorem~\mref{theo-4}, we obtain
\[
\begin{aligned}
\lim_{n_G\to\infty}\mathbb{P}\{T_a>\chi_{1,\alpha_1}^2\}
&= \mathbb{P}\left\{ T_a^\infty>\chi_{1,\alpha_1}^2 \right\}\\
&=\mathbb{P}\left(U>\chi_{1,\alpha_1}^2,\,Z\le z_{\alpha_2}\right)
+\mathbb{P}\left(Z>z_{\alpha_2}\right)\\
&=\mathbb{P}\left(U>\chi_{1,\alpha_1}^2\right)\mathbb{P}\left(Z\le z_{\alpha_2}\right)
+\mathbb{P}\left(Z>z_{\alpha_2}\right)\\
&=\alpha_1(1-\alpha_2)+\alpha_2
=\alpha_1+(1-\alpha_1)\alpha_2,
\end{aligned}
\]
where $U$ and $Z$ are independent.

Finally, consider the local alternative $\epsilon_G^*=\epsilon_0+\tau\sigma_G/\sqrt{n_G}$ with $\tau>0$. 
Because $n_G=2m$ and $\sigma_2=\sqrt{2\sigma_G^2/\mathbb{P}(G)}$,
\[
\frac{\sqrt m(2\epsilon_G^*-2\epsilon_0)}{\sigma_2}=\tau\sqrt{\mathbb{P}(G)}.
\]
It follows from \eqref{eq-sel-T3-standardized} that
\[
\widetilde{T}_3=Z_m+\tau\sqrt{\mathbb{P}(G)}\tod V,
\]
where $V \sim N(\tau\sqrt{\mathbb{P}(G)},1)$. Therefore,
\[
\mathbbm{1}(T_3>c_{\alpha_2})\tod \mathbbm{1}(V>z_{\alpha_2}),
\]
because $\mathbb{P}(V=z_{\alpha_2})=0$. Since $\epsilon_0>0$, the null value $0$ remains separated from the mean
\(\epsilon_2^*=2\epsilon_G^*\to 2\epsilon_0\). Hence, by consistency of the empirical likelihood ratio under this local sequence,
\[
\widetilde{T}_2=\ell_2(0)\pto\infty
\]
in probability. Thus ASEL rejects with probability tending to one on the event $\{V>z_{\alpha_2}\}$; on the complementary event, $T_a$ reduces asymptotically to $T_1$.
Since $V=Z+\tau\sqrt{\mathbb{P}(G)}$ is a measurable function of $Z_m$ and $T_1$ is asymptotically independent of $Z_m$ by the proof of Theorem~\mref{theo-4}, $U$ and $V$ are asymptotically independent.
Conditioning on $V$, we obtain
\[
\begin{aligned}
\lim_{n_G\to\infty}\mathbb{P}\{T_a>\chi_{1,\alpha_1}^2\mid \epsilon_G^*\}
&=\mathbb{P}(V>z_{\alpha_2}) +\mathbb{P}\{U>\chi_{1,\alpha_1}^2,\ V\le z_{\alpha_2}\} \\
&=1-\Phi(z_{\alpha_2}-\tau\sqrt{\mathbb{P}(G)})
+\alpha_1\Phi(z_{\alpha_2}-\tau\sqrt{\mathbb{P}(G)}) \\
&=1-(1-\alpha_1)\Phi(z_{\alpha_2}-\tau\sqrt{\mathbb{P}(G)}).
\end{aligned}
\]
This completes the proof.

\end{appendices}
\else
  % The supplementary material is compiled separately; do not include it here.
\fi

\end{document}